\documentclass{article}
\usepackage[eandd,final]{neurips_2026}

\usepackage[utf8]{inputenc}
\usepackage[T1]{fontenc}
\usepackage{hyperref}
\usepackage{url}
\usepackage{booktabs}
\usepackage{amsmath}
\usepackage{amssymb}
\usepackage{amsthm}
\usepackage{amsfonts}
\usepackage{mathtools}
\usepackage{nicefrac}
\usepackage{microtype}
\usepackage{xcolor}
\usepackage[capitalise,noabbrev]{cleveref}

\usepackage{colortbl}
\usepackage{graphicx}

\usepackage{multirow}
\usepackage{makecell}
\newcommand{\Xspace}{\mathcal{X}}
\newcommand{\Yspace}{\mathcal{Y}}
\newcommand{\Data}{\mathcal{D}}
\newcommand{\loss}{\ell}
\newcommand{\Reals}{\mathbb{R}}
\newcommand{\Rpos}{\mathbb{R}_{\geq 0}}

\newcommand{\learnedH}{H}
\newcommand{\selector}{C}

\newcommand{\ptrue}{p^{*}}
\newcommand{\predDist}[2]{p(#1 \mid #2)}
\newcommand{\idlikelihood}{p_{\text{id}}(x)}
\newcommand{\oodlikelihood}{p_{\text{ood}}(x)}

\newcommand{\flatacc}[2]{\smash[t]{#1{#2}}\vphantom{#2}}
\newcommand{\Tstar}{\smash[b]{T^{*}_{\ell}}\vphantom{T_{\ell}}}
\newcommand{\Astar}[1][\ell]{\smash[b]{A^{*}_{#1}}\vphantom{A_{\ell}}}
\newcommand{\Estar}[1][\ell]{\smash[b]{E^{*}_{#1}}\vphantom{E_{\ell}}}
\newcommand{\Tbar}{{\flatacc{\bar}{T}}_{\ell}}
\newcommand{\Abar}[1][\ell]{\smash[b]{{\flatacc{\bar}{A}}_{#1}}\vphantom{A_{\ell}}}
\newcommand{\Ebar}[1][\ell]{\smash[b]{{\flatacc{\bar}{E}}_{#1}}\vphantom{E_{\ell}}}
\newcommand{\smashb}[1]{\smash{#1}}
\newcommand{\Ahat}[1][\ell]{\smash[b]{{\flatacc{\hat}{A}}_{#1}}\vphantom{A_{\ell}}}
\newcommand{\Ehat}[1][\ell]{\smash[b]{{\flatacc{\hat}{E}}_{#1}}\vphantom{E_{\ell}}}
\newcommand{\That}{{\flatacc{\hat}{T}}_{\ell}}

\newcommand{\coverage}{\rho}
\newcommand{\Risk}{\mathrm{Ri}}
\newcommand{\Regret}{\mathrm{Re}}
\newcommand{\AuReC}{\mathrm{AuReC}}
\newcommand{\AuRC}{\mathrm{AuRC}}
\newcommand{\AuROC}{\mathrm{AUROC}}
\DeclareMathOperator*{\E}{\mathbb{E}}
\DeclareMathOperator*{\argmin}{arg\,min}
\DeclareMathOperator*{\argmax}{arg\,max}

\theoremstyle{plain}
\newtheorem{theorem}{Theorem}

\newcommand{\eqdef}{\triangleq}
\DeclarePairedDelimiter{\set}{\{}{\}}
\newcommand{\indicator}[1]{\mathbf{1}\!\set{#1}}
\newcommand{\learner}{\mathcal{A}}
\newcommand{\lab}{{\Data}}
\newcommand{\pool}{\mathcal{U}}
\newcommand{\score}{s}
\newcommand{\pid}{p_{\text{id}}}
\newcommand{\pood}{p_{\text{ood}}}
\newcommand{\NAUC}{\mathrm{NAUC}}

\title{Evaluating Epistemic Uncertainty:\\ Beyond OOD Detection and Active Learning}

\author{Jakub~Paplhám\\
  Czech Technical University in Prague\\
  \texttt{paplhjak@fel.cvut.cz} \\
   \AND
   Willem~Waegeman\\
   Ghent University \\
   \And
   Eyke~Hüllermeier \\
   LMU Munich, MCML, DFKI \\
   \And
   Vojtěch~Franc\\
  Czech Technical University in Prague\\
}

\begin{document}

\maketitle

{
{
\begin{abstract}
Current evaluation of epistemic uncertainty relies on tasks such as out-of-distribution detection, active learning and selective prediction. However, common epistemic uncertainty scores, such as mutual information or the variance of the predictions, estimate the regret, the excess error of a predictor over the Bayes-optimal predictor. Building on the epistemic reject-option framework, we evaluate epistemic uncertainty using its ability to identify regret, and compare this evaluation with the other tasks. Benchmarking standard methods on datasets with repeated human labels reveals that the tasks rank the methods differently, including pairwise rank inversions between methods that are top-ranked on one criterion and bottom-ranked on another. The disagreement is present already at the optimum, where each task calls for a different score. Formulating selective prediction as a constrained optimization over coverage, expected risk, and regret, we prove that the optimal score is a convex combination of the ground-truth aleatoric and epistemic uncertainties, so that selective prediction and the regret task are the two endpoints of one family of tasks. Out-of-distribution detection, active learning and the regret task thus disagree, theoretically as well as empirically, on which method is a good estimator of epistemic uncertainty, while selective prediction departs from the regret task where the true risk and the regret order the inputs differently. Researchers must therefore state the downstream decision that an epistemic uncertainty estimate is meant to support, instead of relying on a general notion of model ignorance.
\end{abstract}
}

{
\section{Introduction}
The primary goal of uncertainty quantification and disentanglement is to isolate epistemic uncertainty (model ignorance) from aleatoric uncertainty (irreducible data noise). Successfully isolating these components should allow practitioners to take better-informed actions, such as deciding whether to abstain from a risky prediction or to acquire new labels. However, the field often leaves the exact operational goal of epistemic uncertainty vaguely defined. Like \citet{hofman2026onesize}, we argue that to rigorously evaluate the uncertainty estimate, researchers must explicitly define the downstream action it is intended to inform, rather than relying on ambiguous notions of ignorance.

\begin{figure}[t]
    \centering
    \includegraphics[height=7cm,keepaspectratio]{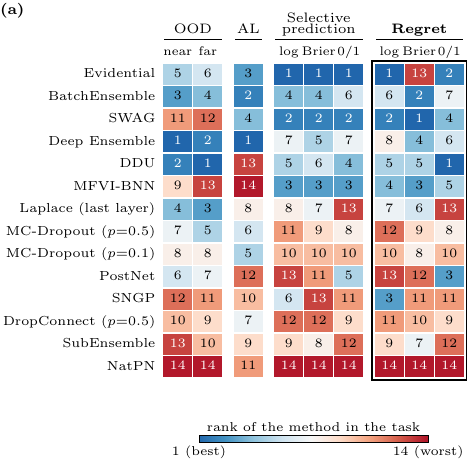}\hfill
    \includegraphics[height=7cm,keepaspectratio]{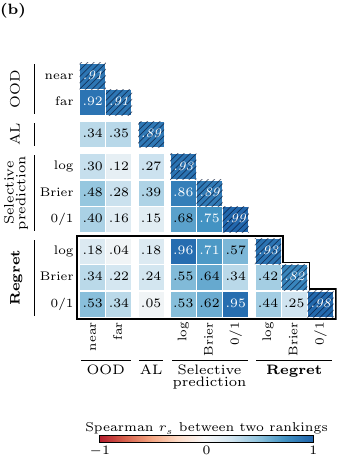}
    \caption{\textbf{The tasks used to evaluate epistemic uncertainty disagree on which method is best.} Uncertainty quantification methods on CIFAR-10H, each scored by its epistemic estimate. \textbf{(a)}~Rank of each method by each task's metric (\cref{tab:tasks}). \textbf{(b)}~Spearman correlation between the rankings of two tasks; on the hatched diagonal, of two seeds of the same task. The regret task ranks the methods largely unlike OOD detection and active learning, while each task's ranking is stable across seeds.}
    \label{fig:rank_grid}
\end{figure}

Many standard measures of epistemic uncertainty, such as mutual information or the variance of predictive means, estimate the reducible error, the regret: the excess error of the predictor over the Bayes-optimal predictor. For a Bayesian model the correspondence is exact, since the mutual information and the variance are the posterior expectations of the log-loss and the squared-loss regret (\cref{sec:setting}). Yet, because true regret is uncomputable in standard single-label empirical settings, the literature predominantly relies on surrogate evaluation tasks: out-of-distribution (OOD) detection, which separates in-distribution from out-of-distribution inputs \citep{hendrycks2017baseline}; active learning (AL), which selects the inputs to be labeled and added to the training set \citep{lewis1994sequential, houlsby2011bald, gal2017deep}; and selective prediction, which selects the inputs on which the model abstains to minimize its risk \citep{chow1970optimum, elyaniv2010foundations}. Benchmarks that evaluate a range of methods on more than one of these tasks notice that what works well for one task does not necessarily work well for another \citep{mucsanyi2024benchmarking, probly, hofman2026onesize}. These benchmarks are, however, incomplete: none of them evaluates the regret. This raises a question:
\begin{center}
\emph{Do the tasks agree on which methods estimate epistemic uncertainty well, and if not, why?}
\end{center}

We show that each of the tasks would be optimally solved by a score modeling a different notion of uncertainty. The individual solutions are known; we rewrite and analyze them in unified notation. Under standard formulations, OOD detection is optimally solved by the input likelihood ratio $\nicefrac{\oodlikelihood}{\idlikelihood}$, which is entirely independent of the predictive task loss. For a single label, the active learning objective is optimally solved by expected error reduction, which answers how a sample \emph{globally} affects the error of a newly trained model. In contrast, epistemic uncertainty in terms of the regret \emph{locally} describes the reducible error for a specific observation. A closed-form one-dimensional example (\cref{fig:sandbox}) illustrates the difference: the Bayes-optimal scores for OOD detection, active learning, and regret-minimization single out largely different regions of the input space.

Building on the epistemic reject-option framework of \citet{franc2025epistemic}, we evaluate uncertainty estimates directly by their ability to identify the regret, and compare this evaluation with the surrogate tasks. We unify standard selective classification (often used to evaluate total uncertainty) and the epistemic reject-option into a single constrained-optimization framework (\cref{sec:family}). We prove that the optimal score for any coverage, risk, and regret constraint is a convex combination of the true aleatoric ($\Astar[]$) and epistemic ($\Estar[]$) components; selective prediction and the regret task are the two endpoints of this framework. We benchmark fourteen{} uncertainty quantification methods implemented in the \texttt{probly} library \citep{probly} on datasets with repeated human labels, such as CIFAR-10H \citep{peterson2019human}, the DCIC datasets \citep{schmarje2022benchmark} and APPA-REAL \citep{agustsson2017appareal}, whose empirical annotator distribution serves as the benchmark reference for the true conditional distribution $\ptrue(y \mid x)$.

Empirically, the tasks rank the methods differently under the same setup (\cref{fig:rank_grid}): what one benchmark identifies as a good epistemic uncertainty score, another identifies as a poor one. Epistemic uncertainty in the sense of active learning or OOD detection is therefore a different quantity from epistemic uncertainty in the sense of the regret. Any estimate of epistemic uncertainty should then exactly specify the task it is meant to solve, rather than leave the notion vaguely defined.

}

{
\section{Related Work}

\paragraph{Reject-option prediction.}
\Citet{chow1970optimum} established that the optimum binary classifier with a reject option thresholds the conditional posterior. The task was later extended to formulations with a bounded coverage or a bounded risk \citep{pietraszek2005optimizing, elyaniv2010foundations, geifman2017selective}; all of these were shown to require the same strategy, a threshold on the conditional risk \citep{cortes2016learning, franc2023optimal}. Selective classification adapts this to deep learning \citep{geifman2017selective}, evaluating total predictive uncertainty through risk-coverage curves. Recent work unifies selective classification with OOD detection \citep{sircscod2023, franc2024scod} and surveys the broader landscape \citep{hendrickx2024machine}. \Citet{traub2024overcoming} catalog common evaluation flaws in these systems. Most relevant to our work, \citet{franc2025epistemic} extend the reject option to epistemic uncertainty, defining the predictor that abstains when the expected gap to the Bayes-optimal predictor is large. We build on their formulation: their regret task and selective classification are the two endpoints of the family of tasks in \cref{sec:family}, whose solutions weight the aleatoric and the epistemic component, as do the regression prediction intervals of \citet{azizi2026clear}.

\paragraph{Aleatoric and epistemic uncertainty.}
The distinction between aleatoric uncertainty, the irreducible randomness of the outcome, and epistemic uncertainty, the lack of knowledge that more data could remove, comes from risk analysis and reliability engineering \citep{hora1996aleatory, derkiureghian2009aleatory}. \citet{senge2014reliable} brought it to machine learning, and \citet{hullermeier2021aleatoric} made it widely known there, with a survey of its formalizations; in deep learning, \citet{kendall2017uncertainties} separated the two in Bayesian neural networks for computer vision. A broad body of work decomposes the predictive uncertainty of a model into the two parts: by entropy \citep{houlsby2011bald, depeweg2018decomposition, smith2018understanding}, by variance \citep{depeweg2018decomposition, sale2023variance}, by proper scoring rules and Bregman divergences \citep{gruber2023uncertainty, hofman2024quantifying, kotelevskii2025from}, and by distances between second-order distributions \citep{sale2024distance}. Both the measures and the methods that estimate them have been questioned \citep{wimmer2023quantifying, sale2023volume, bengs2022pitfalls}. The choice among these measures has recently been tied to the task: \citet{bickfordsmith2025rethinking} argue that uncertainty should be defined by the decision it informs, and \citet{hofman2026onesize} that different tasks call for different uncertainty measures. We reach a similar conclusion for the tasks used to evaluate epistemic uncertainty: their optimal scores differ, and so do their rankings of the methods (\cref{sec:outside,sec:experiments}).

\paragraph{The reference distribution.}
The regret is defined relative to the true conditional distribution, which single-label datasets do not provide. \citet{jimenez2026position} obtain a reference by retraining models on many datasets drawn from a known synthetic process, and show that second-order methods capture only part of the variance of the trained models and none of their bias. We use datasets with repeated human labels, of images \citep{peterson2019human, schmarje2022benchmark, agustsson2017appareal}, which provide an empirical annotator distribution for every input.

\paragraph{Benchmarks of uncertainty quantification.}
OpenOOD \citep{zhang2023openood} standardizes OOD detection at scale, with semantic and covariate shifts. \citet{mucsanyi2024benchmarking} evaluate nineteen methods on OOD detection, selective prediction, calibration and the estimation of label noise on one set of trained models and find that different estimators specialize in different tasks. They also report high rank correlation (Spearman correlation at least $0.78$) between learned aleatoric and epistemic components on standard benchmarks and interpret it as the two estimates capturing the same notion of uncertainty rather than two disentangled ones. The \texttt{probly} library \citep{probly} implements a broad range of methods and benchmarks them on selective prediction, OOD detection and active learning. Across these benchmarks the best method changes with the task. Our evaluation builds on \texttt{probly} and extends it with the regret evaluation, further tasks and datasets and a theoretical analysis.

}

{
\section{Preliminaries: Uncertainty and the Reject Option}
\label{sec:setting}

Let $\Xspace$ be the input space, $\Yspace$ the target space, and $\smashb{\Data \in (\Xspace \times \Yspace)^N}$ a training set. We assume that the data are generated by a parametric model $\predDist{x, y}{\theta} = p(x)\, \predDist{y}{x, \theta}$ with an unknown true parameter $\smashb{\theta^* \in \Theta}$, and write $\ptrue(y \mid x) = p(y \mid x, \smashb{\theta^*})$. Let $\loss \colon \smashb{\hat{\Yspace}} \times \Yspace \to \smashb{\Rpos}$ be a prediction loss that scores a prediction $\hat y \in \smashb{\hat{\Yspace}}$ against the true label $y$. In supervised learning, the goal is to predict with the smallest expected loss $\E_{\smashb{\scriptscriptstyle (x, y) \sim p(x, y \mid \theta^*)}}[\loss(\hat y(x), y)]$. Under a known $\theta$, the \emph{Bayes-optimal predictor} is $h(x, \theta)$. A model learned from $\Data$ outputs a predictive distribution $q(y \mid x)$ and predicts $\learnedH(x, \Data)$:
\begin{equation}
    h(x, \theta) \eqdef \argmin_{\hat{y}}\, \E_{y \sim \predDist{y}{x, \theta}}[\loss(\hat{y}, y)],
    \qquad
    \learnedH(x, \Data) \eqdef \argmin_{\hat{y}}\, \E_{y \sim q(y \mid x)}[\loss(\hat{y}, y)].
    \label{eq:predictors}
\end{equation}
For the \emph{Bayesian predictor}, $q$ is the posterior predictive $p(y \mid x, \Data) = \int p(y \mid x, \theta)\, p(\theta \mid \Data)\, d\theta$.

\paragraph{Realized and expected components.}
When evaluating a model empirically on a real-world task, we operate under a single, realized ground-truth $\theta^*$. At an input $x$, the predictor incurs the risk $\Tstar(x, \Data)$ and the Bayes-optimal predictor the Bayes risk $\Astar(x)$,
\begin{equation}
    \Tstar(x, \Data) = \E_{y \sim \ptrue(\cdot \mid x)}\big[\loss(\learnedH(x, \Data), y)\big],
    \qquad
    \Astar(x) = \E_{y \sim \ptrue(\cdot \mid x)}\big[\loss(h(x, \theta^*), y)\big],
    \label{eq:realized}
\end{equation}
and their difference $\Estar(x, \Data) = \Tstar(x, \Data) - \Astar(x)$ is the regret, the part of the risk caused by not knowing $\theta^*$ (closed forms in \cref{app:bayesian}). We omit $\Data$ when the training set is fixed. A learner that does not know $\theta^*$ can compute these components only in expectation over the posterior $p(\theta \mid \Data)$. For $x \in \Xspace$ and $\Data \in (\Xspace \times \Yspace)^N$, we define
\begin{equation}
\label{def:components}
\begin{aligned}
    \Tbar(x, \Data) &\eqdef {\E}_{(\theta,y)}\big[\loss(\learnedH(x, \Data), y)\big]
        && \text{(total)} \\
    \Abar(x, \Data) &\eqdef {\E}_{(\theta,y)}\big[\loss(h(x, \theta), y)\big]
        && \text{(aleatoric)} \\
    \Ebar(x, \Data) &\eqdef {\E}_{(\theta,y)}[\loss(\learnedH(x, \Data), y) - \loss(h(x, \theta), y)]
        && \text{(regret / epistemic)}
\end{aligned}
\end{equation}
where the joint expectation runs over $(\theta, y) \sim \predDist{\theta, y}{x, \Data}$. By construction, $\Tbar = \Abar + \Ebar$.
For the Bayesian predictor, the expected regret $\Ebar$ instantiates several standard measures of epistemic uncertainty, the mutual information between $y$ and $\theta$ under log loss and the posterior variance of the predictive mean under squared loss \citep{depeweg2018decomposition, franc2025epistemic}. These measures are therefore estimates of the realized regret $\Estar$, its expectation over the posterior. We evaluate an estimate by how well it identifies the realized regret. Throughout, by \emph{``method''} we refer to any uncertainty quantification method that provides $q$ and estimates $\Ahat(x)$, $\Ehat(x)$, $\That(x)$.

\paragraph{The reject option.}
A model with a reject option may abstain from predicting on an input $x$. A selector $\selector \colon \Xspace \to [0, 1]$ outputs the probability of accepting $x$. Three quantities over the true marginal distribution of the inputs, $\ptrue(x)$, describe the overall performance of the selector:
\begin{align}
    &\textbf{Coverage} & \coverage(\selector) &= {\E}_{x \sim \ptrue}\big[\selector(x)\big], \label{eq:coverage}\\
    &\textbf{Risk} & \Risk(\selector) &= {\E}_{x \sim \ptrue}\big[\selector(x)\, \Tstar(x)\big] / \coverage(\selector), \label{eq:risk}\\
    &\textbf{Regret} & \Regret(\selector) &= {\E}_{x \sim \ptrue}\big[\selector(x)\, \Estar(x)\big] / \coverage(\selector). \label{eq:regret}
\end{align}These quantities capture the fundamental trade-offs of model deployment: \textbf{Coverage} measures the proportion of data the model accepts, \textbf{Risk} measures the expected total loss incurred on those accepted inputs, and \textbf{Regret} isolates the expected reducible error caused strictly by the model imperfections.

}

{
\section{The Tasks and Their Optimal Scores}
\label{sec:tasks}
\label{sec:outside}

In this section, we formalize the tasks used to evaluate epistemic uncertainty and state the optimal scores for solving them (\cref{tab:tasks,app:tasks_full}). For each task, every input $x$ is assigned an uncertainty score $\score(x) \in \Reals$ (such as a model's epistemic estimate). The decision rule then acts on inputs with high uncertainty scores by rejecting them, flagging them as out-of-distribution, or selecting them for labeling. We begin with selective prediction and the regret task, which build on \cref{sec:setting}.

\subsection{Selective prediction and the regret task}
\label{sec:family_endpoints}
Selective prediction minimizes the selective risk \eqref{eq:risk} and is solved by a threshold on the conditional risk \citep{chow1970optimum}, which for a given predictor is the total uncertainty $\Tstar = \Astar + \Estar$; the regret task minimizes the selective regret \eqref{eq:regret} and is solved by a threshold on $\Estar$ \citep{franc2025epistemic},
\begin{equation}
    \selector^*_{\Risk}(x) = \indicator{\Tstar(x) \leq \tau}, \qquad \selector^*_{\Regret}(x) = \indicator{\Estar(x) \leq \tau}.
\end{equation}

In these two tasks, an uncertainty score defines one selector for every coverage $\kappa \in (0, 1]$: the selector $\selector^{\score}_{\kappa}$ accepts the fraction $\kappa$ of the inputs with the lowest uncertainty score $\score$. Any uncertainty quantification method can therefore be turned into a selector and evaluated on both tasks. The standard metrics are the areas under the risk--coverage and the regret--coverage curve \citep{geifman2019bias, franc2023optimal, franc2025epistemic},
\begin{equation}
    \AuRC(\score) = \smallint_0^1 \Risk(\selector^{\score}_{\kappa})\, d\kappa, \qquad \AuReC(\score) = \smallint_0^1 \Regret(\selector^{\score}_{\kappa})\, d\kappa.
    \label{eq:aurc}
\end{equation}
{Many uncertainty quantification methods also define the predictor, so the predictions, and with them the risk and the regret, differ between methods. The areas grow with the risk or the regret of the predictor as well as with the quality of the ordering. To evaluate the uncertainty estimate rather than the predictor, we use the normalized areas \citep{geifman2019bias, cattelan2024how}, $n\text{-}\AuRC(\score) \eqdef (\AuRC(\score) - \AuRC_{\text{oracle}}) / (\AuRC_{\text{random}} - \AuRC_{\text{oracle}})$, where $\AuRC_{\text{oracle}}$ is the area of the optimal score $\Tstar$ and $\AuRC_{\text{random}}$ that of a random ordering. The normalized area is $0$ for the optimal score and $1$ for a random ordering; $n$-$\AuReC$ is defined in the same way, with the optimal score $\Estar$.
}

\subsection{Risk and regret jointly}
\label{sec:family}

Selective prediction and the regret task are not independent tasks but the two endpoints of one family. The joint task controls the selective risk and the selective regret at the same time. Its goal is to find the selector with the largest coverage whose risk and regret stay within the budgets $r > 0$ and $\gamma > 0$,
\begin{equation}
    (\mathrm{P}_{\coverage}):\quad
        {\max}_{\selector}\; \coverage(\selector)
        \;\;\text{s.t.}\;\; \Risk(\selector) \leq r,\; \Regret(\selector) \leq \gamma.
        \label{eq:Pc}
\end{equation}
Exchanging the objective with a constraint gives two further formulations, $(\mathrm{P}_{\Risk})$ and $(\mathrm{P}_{\Regret})$, which minimize the risk or the regret at a required coverage $\kappa$ (\cref{app:proofs}). Since the three formulations share the same optimal selectors, we analyze only \eqref{eq:Pc}.

\begin{theorem}[Unified selector]
\label{thm:unified}
Assume that the expected risk and regret of the predictor, $\E_{x \sim \ptrue}[\Tstar(x)]$ and $\E_{x \sim \ptrue}[\Estar(x)]$, are finite, and that for every $\lambda \in [\nicefrac{1}{2}, 1]$ the score $(1-\lambda)\Astar + \lambda \Estar$ has a continuous distribution under $x \sim \ptrue$. For any feasible instance of \eqref{eq:Pc}, $(\mathrm{P}_{\Risk})$ or $(\mathrm{P}_{\Regret})$, there exist $\lambda \in [\nicefrac{1}{2}, 1]$ and $\tau \in \Reals \cup \{+\infty\}$ such that every optimal selector is
\begin{equation}
    \selector^{*}(x) \;=\; \indicator{(1-\lambda)\,\Astar(x) + \lambda\,\Estar(x) \;\leq\; \tau}.
\end{equation}
The endpoints recover the standard formulations: $\lambda = \nicefrac{1}{2}$ thresholds total uncertainty $\Tstar = \Astar + \Estar$ (selective prediction \citep{chow1970optimum, geifman2017selective}), while $\lambda = 1$ thresholds epistemic uncertainty $\Estar$ alone (epistemic reject-option \citep{franc2025epistemic}).
\end{theorem}

\begin{proof}
The proof applies the Karush--Kuhn--Tucker conditions \citep{karush39minima, Kuhn1951Nonlinear} to the Lagrangian of each problem; the case-by-case algebra is given in \cref{app:proofs}.
\end{proof}

\paragraph{The Pareto gap.}
\label{sec:family_measure}
The optimal selectors of \cref{thm:unified} trace the Pareto-optimal surface $\mathcal{S}^{*}$, the coverage, risk and regret that they attain as $\lambda$ sweeps $[\nicefrac{1}{2}, 1]$ and $\tau$ sweeps $\Reals$,
{\begin{equation}
    \mathcal{S}^{*} \;=\; \big\{ (\coverage, \Risk, \Regret)(C) \;\big|\; C = \indicator{(1-\lambda) \Astar + \lambda \Estar \leq \tau},\; \lambda \in[\nicefrac{1}{2}, 1],\, \tau \in \Reals \big\}.
    \label{eq:surface}
\end{equation}}No selector improves one of the three quantities without worsening another. A method with aleatoric and epistemic estimates $(\Ahat, \Ehat)$ attains the points of its selectors $\indicator{w_1 \Ahat + w_2 \Ehat \leq \tau}$ for real weights $w_1$, $w_2$ and thresholds $\tau$. We define the \emph{Pareto gap} as the distance from $\mathcal{S}^{*}$ to these points, measured by the modified inverted generational distance \citep{ishibuchi2015modified}: it averages, over the points of $\mathcal{S}^{*}$, the distance to the nearest point that the method attains, counting only the coordinates in which that point is worse. The risk and the regret are taken relative to their values at full coverage, and the gap is normalized so that $0$ means that the method attains $\mathcal{S}^{*}$ and $1$ corresponds to a random ordering. It measures the quality of the decomposition: the gap is zero when, for every $\lambda$, a combination of the two estimates orders the inputs as $(1-\lambda)\Astar + \lambda\Estar$ does. Previous work judges the decomposition by the rank correlation of the two estimates \citep{mucsanyi2024benchmarking}. The true components can, however, be strongly correlated themselves, for instance where the predictor errs mostly on ambiguous inputs; a high rank correlation between learned estimators may then accurately reflect the underlying data-generating process, rather than a failure of the uncertainty estimation.

}
{
\subsection{Out-of-distribution detection}
\label{sec:outside_ood}
Two further tasks are commonly used to evaluate epistemic uncertainty: out-of-distribution (OOD) detection and active learning. We show that, under their standard formulations, both tasks differ fundamentally from identifying the reducible loss (the regret).

\emph{Given.} An in-distribution density $\idlikelihood$ from which training inputs are drawn, and an out-of-distribution density $\oodlikelihood$. Neither is known, but unlabeled samples of both are available.

\emph{Goal.} Separate the two distributions. Whether the out-of-distribution inputs come from a new class, a data corruption, or a different domain, the benchmarked task is binary classification: a detector flags an input as out-of-distribution when its score exceeds a threshold. It makes two kinds of errors, flagging an in-distribution input (a false positive) and missing an out-of-distribution one (a false negative), and a higher threshold trades fewer errors of the first kind for more of the second.

\emph{Optimal score.} Whether a formulation bounds the rate of one error or weighs the two, it is solved by thresholding the likelihood ratio $\score_{\mathrm{OOD}}^*(x) = \nicefrac{\oodlikelihood}{\idlikelihood}$; only the threshold varies. By the Neyman--Pearson lemma \citep{neyman1933problem}, the likelihood ratio attains the largest true-positive rate, the fraction of out-of-distribution inputs flagged, at every false-positive rate, the fraction of in-distribution inputs flagged, so no other score trades the two errors better.

\emph{Metric.} The standard metric is the area under the receiver operating characteristic (ROC) curve, AUROC \citep{hendrycks2017baseline, zhang2023openood}. The ROC curve of a score plots the true-positive rate against the false-positive rate over all thresholds, and its area equals the probability that the score ranks an OOD input above an in-distribution one. Benchmarks also report the false-positive rate at a true-positive rate of $95\%$ \citep{zhang2023openood}. Both metrics depend on the score only through its ROC curve and are therefore optimized by the likelihood ratio.

OOD detection is therefore not the task of identifying the reducible loss. The likelihood ratio depends exclusively on marginal input densities, independently of the predictive loss $\loss$ and the conditional distribution $\ptrue(y \mid x)$. In contrast, the uncertainty estimates of a supervised classifier do not depend on the input densities alone: its predictive uncertainty measures label noise, and its epistemic uncertainty vanishes as training data grows, whereas the density ratio does not \citep{li2025position}. The optimal score changes when the formulation couples detection with the classifier's own task, as in selective classification with out-of-distribution data \citep{sircscod2023, franc2024scod}.

\subsection{Active learning}
\label{sec:outside_al}
\emph{Given.} A learning algorithm $\learner\colon (\Xspace \times \Yspace)^{*} \to \mathcal{P}(\Yspace)^{\Xspace}$ that maps a labeled set $\lab$, a finite set of input--label pairs, to a predictive distribution $q_\lab = \learner(\lab)$, with $\mathcal{P}(\Yspace)$ the distributions over $\Yspace$; an initial labeled set $\lab_0 = \{(x_i, y_i)\}_{i=1}^{n}$ drawn i.i.d.\ from $\ptrue(x, y)$; a pool $\pool = \{x_j\}_{j=1}^{m}$ of unlabeled inputs drawn i.i.d.\ from $\ptrue(x)$; and a budget of $B$ labels, drawn from $\ptrue(\cdot \mid x)$. The model trained on $\lab$ has the risk $R(\lab) = \E_{(x, y) \sim \ptrue}[\loss(\learnedH(x, \lab), y)]$, where $\learnedH(x, \lab)$ is the prediction of $q_\lab$ \eqref{eq:predictors}.

\emph{Goal.} Select the inputs whose labels improve the retrained model most. An acquisition rule scores the pool, requests labels for the highest-scoring inputs, and retrains the model. Two formulations of this goal are standard. The \emph{final-loss} formulation minimizes the expected risk of the retrained model,
\begin{equation}
    {\min}_{x_1, \dots, x_B \in \pool}\; {\E}_{y_1 \sim \ptrue(\cdot \mid x_1),\, \dots,\, y_B \sim \ptrue(\cdot \mid x_B)}\Big[R\big(\lab_0 \cup \{(x_b, y_b)\}_{b=1}^{B}\big)\Big].
\end{equation}
The \emph{information} formulation applies to a Bayesian learner, whose uncertainty about the parameters $\theta$ is the posterior $p(\theta \mid \lab)$. It acquires the input whose label is expected to reduce the entropy of this posterior most, ${\max_{x \in \pool} I(y; \theta \mid x, \lab)}$, the mutual information of label and parameters.

\emph{Optimal score.} The final-loss problem is intractable in general. Its expectation runs over every combination of the $B$ unknown labels, each of which requires retraining the model and evaluating its risk under the unknown $\ptrue$, and its minimum runs over all subsets of $B$ pool inputs. In practice, it is approximated greedily: one acquires the single input whose label is expected to reduce the risk most, retrains the model, and repeats, which in general does not solve the problem for $B$ labels. The one-label problem is solved by the input with the largest expected error reduction (EER) \citep{roy2001toward}, the expected risk decrease over the candidate's unknown label. As the Bayes risk does not change with retraining, EER is the expected decrease of the regret,
\begin{equation}
    \score_{\mathrm{EER}}(x) = \E_{y \sim \ptrue(\cdot \mid x)}\, \E_{x' \sim \ptrue}\!\big[\Estar(x', \lab) - \Estar(x', \lab \cup \{(x, y)\})\big],
\end{equation}
where $\Estar(x', \lab)$ is the regret \eqref{eq:realized} at $x'$ of the model trained on $\lab$. The score requires a retrained model for every candidate and possible label and is approximated \citep{roy2001toward}; it is exact only in simple settings, such as ridge regression under the squared loss. The information problem is solved by its objective, $\score_{\mathrm{BALD}}(x) = I(y; \theta \mid x, \lab)$, known as Bayesian active learning by disagreement (BALD) \citep{houlsby2011bald}; it is the epistemic term of the entropy decomposition (\cref{sec:setting}) \citep{depeweg2018decomposition}.

\emph{Metric.} Benchmarks average a performance measure of the retrained model over the acquisition rounds, which evaluates the final-loss formulation at every intermediate budget. In classification the measure is the test accuracy and the average is the normalized area under the accuracy--budget curve ($\NAUC$) \citep{settles2009active, gal2017deep}; in regression it is the test error.

Active learning is therefore not the task of identifying the reducible loss either. The regret appears in EER, but globally: EER measures how one label changes the regret of the retrained model over the whole test distribution, whereas the regret $\Estar(x)$ is the excess error of the current model at the single input $x$. Inputs with high $\Estar$ but small test density can yield near-zero EER, and vice versa. BALD, the optimal score of the information formulation, equals at each input the posterior expectation of the log-loss regret of the current model (\cref{app:bayesian}); under any other loss it differs from the regret.

\subsection{Relations between the optimal scores \& closed-form example}
\label{sec:outside_relations}
The optimal scores of the four tasks coincide only in special cases, so a method optimized under one task may fail under another. They are functions of different objects: OOD detection of the input densities, active learning of the error of a retrained model over the whole test distribution, selective prediction of the risk of the current model at an input, and the regret task of the excess of that risk over the Bayes risk. A closed-form example (\cref{fig:sandbox}) shows how differently they select inputs.
We construct a tractable 1-D regression sandbox, inspired by the \textit{Rossellini function} $y=\sin(\nicefrac{0.2}{(x+0.16)^3}) + \epsilon$, \citep{jimenez2026position, rossellini2024integrating}. The environment is designed such that the true function, the noise profile, and the data sampling distribution are all explicitly known, and every optimal score has a closed form (\cref{app:tasks_full}). By construction, the input space contains distinct regimes: regions with fast-oscillating target but low aleatoric noise (\textbf{A}), regions with high aleatoric noise but slow oscillations (\textbf{C}), and regions with artificially suppressed training data density (OOD) (\textbf{B}). \Cref{fig:sandbox} shows which inputs the optimal score of each objective rejects or selects. Minimizing the risk rejects the noisy region, minimizing the regret rejects where the predictor is wrong, detecting OOD rejects where the training data are sparse, and the two active-learning objectives select the inputs where a label would help most, which differ between them. Even where every quantity is known exactly, the optimal scores of the tasks single out largely different inputs. Whether these differences also appear in the rankings of the methods is the question of \cref{sec:experiments}.

\begin{figure}[t]
    \centering
    \includegraphics[height=6cm,keepaspectratio]{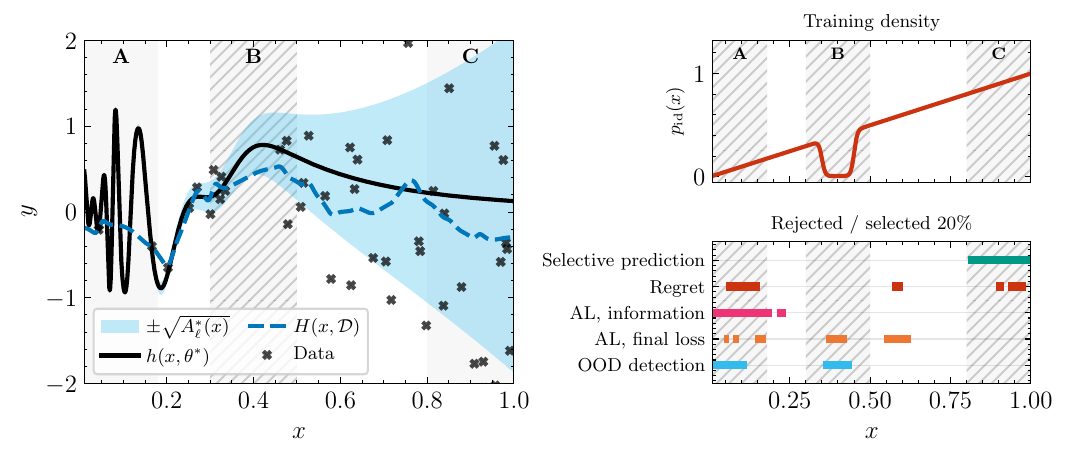}
    \caption{\textbf{Five objectives select five different subsets of the inputs.} A one-dimensional regression problem under the squared loss, in which the true conditional $\ptrue(y \mid x)$, a Gaussian with mean $h(x, \theta^*)$ and variance $\Astar(x)$, and the training density $\idlikelihood$ are known, so that every optimal score has a closed form for one predictor $\learnedH(x, \Data)$, the predictive mean of a ridge regression trained once on the sample shown. \textbf{Left:} the truth, the noise band, the training sample and the predictor; region \textbf{A} has a fast-oscillating truth and little noise, region \textbf{B} a well in the training density, and region \textbf{C} large noise. \textbf{Top right:} the training density. \textbf{Bottom right:} the $20\%$ of the uniform deployment distribution that the optimal score of each task rejects or selects for labeling.}
    \label{fig:sandbox}
\end{figure}

}

\begin{table}
\centering
{
\scriptsize
\setlength{\tabcolsep}{3pt}
\renewcommand{\arraystretch}{1.1}
\caption{\textbf{The tasks, their optimal rules and standard metrics} (definitions in \cref{app:tasks_full}). The optimal rules threshold different quantities, so the tasks define epistemic uncertainty differently.}
\label{tab:tasks}
\begin{tabular}{@{}
  >{\raggedright\arraybackslash}p{2.0cm}
  >{\raggedright\arraybackslash}p{2.4cm}
  >{\raggedright\arraybackslash}p{3.0cm}
  >{\raggedright\arraybackslash}p{4.25cm}
  >{\raggedright\arraybackslash}p{1.3cm}@{}}
\toprule
Task & Inputs & Formulations & Optimal rule & Metric \\
\midrule
Selective prediction (\cref{sec:family_endpoints})
& $q$, $\loss$, labeled test inputs
& Cost ($\varepsilon$); improvement ($r$); abstention ($\kappa$)
& $\indicator{\Tstar \leq \tau}$
& $\AuRC$, $n$-$\AuRC$ \\
\arrayrulecolor{black!25}\midrule\arrayrulecolor{black}
Regret \qquad(\cref{sec:family_endpoints})
& $q$, $\loss$, true test conditional
& Cost ($\delta$); improvement ($\gamma$); abstention ($\kappa$)
& $\indicator{\Estar \leq \tau}$
& $\AuReC$, $n$-$\AuReC$ \\
\arrayrulecolor{black!25}\midrule\arrayrulecolor{black}
Joint risk and regret (\cref{sec:family})
& $q$, $\loss$, true test conditional
& Improvement ($r,\gamma$); abstention ($\kappa,\gamma$) or ($\kappa,r$)
& $\indicator{(1-\lambda)\Astar+\lambda\Estar \leq \tau}$
& Pareto gap \\
\arrayrulecolor{black!25}\midrule\arrayrulecolor{black}
OOD detection (\cref{sec:outside_ood})
& Samples from $\pid$ and $\pood$
& Neyman--Pearson ($\alpha$); cost ($c_1,c_2$)
& $\indicator{\oodlikelihood/\idlikelihood > \tau}$
& AUROC \\
\arrayrulecolor{black!25}\midrule\arrayrulecolor{black}
Active learning (\cref{sec:outside_al})
& $\learner$, $\lab_0$, $\pool$, budget $B$
& Final loss (one label); information (one label)
& $\indicator{R(\lab)-\E_y R(\lab\cup\{(x,y)\})>\tau}$;
  $\indicator{I(y;\theta\mid x,\lab)>\tau}$
& $\NAUC$ \\
\bottomrule
\end{tabular}
}
\end{table}
{
\section{Experiments}
\label{sec:experiments}

To analyze whether the tasks agree in practice, we benchmark fourteen uncertainty quantification methods on the tasks of \cref{sec:tasks}, on datasets with repeated labels. The tasks are evaluated on identical networks; active learning is the exception, since it retrains on the labels it acquires.
}

{

\paragraph{Data with repeated labels.}
The regret is evaluated against the empirical annotator distribution. CIFAR-10H \citep{peterson2019human} provides about fifty labels for each of the $10{,}000$ CIFAR-10 test images. OOD detection on CIFAR-10 uses CIFAR-100 as the near-OOD set and the far-OOD sets of OpenOOD v1.5 \citep{zhang2023openood}. The eight real-world datasets of the DCIC suite \citep{schmarje2022benchmark} provide several labels per image in domains such as microscopy, plankton and tree species; the OOD sets of each are the test images of the other DCIC datasets. APPA-REAL \citep{agustsson2017appareal} provides several apparent-age votes per face image, which we treat as a distribution over integer ages. Its near-OOD set is faces detected in LAION-Face images \citep{laionface} that show drawings, dolls or animals, and its far-OOD sets are ImageNet and Places365 images. We evaluate the classification networks under the log, Brier and 0/1 loss and the APPA-REAL heads under the absolute loss. Where the training images have annotator distributions, on DCIC and APPA-REAL, we train both on them and on one label per image (\cref{app:dcic_appa_results}).
}

{
\paragraph{Methods.}
We use the implementations of the extensive \texttt{probly} library \citep{probly}, with its official CIFAR-10 checkpoints where available. The fourteen methods span the sampling-based, ensemble, deterministic single-pass, Dirichlet and post-hoc families, among them MC-Dropout \citep{gal2016dropout}, Deep Ensembles \citep{lakshminarayanan2017simple}, a last-layer Laplace approximation \citep{daxberger2021laplace}, SWAG \citep{maddox2019simple}, SNGP \citep{liu2020simple}, DDU \citep{mukhoti2023deep}, Evidential networks \citep{sensoy2018evidential} and PostNet \citep{charpentier2020posterior} (full list in \cref{app:implementation}). On CIFAR-10 the networks are ResNet-18 trained from scratch, on DCIC they are fine-tuned from ImageNet weights, and on APPA-REAL a head is trained on the features of a frozen ResNet-101 trained for age estimation on CSFD-1.6M \citep{csfd}; \cref{app:implementation} gives the protocols and the seeds.

}

\begin{table}[t]
\centering
\footnotesize
\setlength{\tabcolsep}{1.9pt}
\caption{\textbf{Summary of the results on every dataset.} \emph{Agreement of the method rankings}: Spearman correlation between the rankings of the methods by two seeds of the regret task (Two seeds, which shows whether the regret ranking is stable), by the regret task and OOD detection, active learning or selective prediction (Regret--OOD, Regret--AL, Regret--risk), and by OOD detection and active learning (OOD--AL); $1$: the same ranking, $0$: unrelated. Bold headers: the agreement of the regret task with OOD detection and active learning. \emph{True risk and regret}: the share of the true risk that is regret (small: the risk is mostly aleatoric), and the Spearman correlation of the true risk and the regret over the test inputs, medians over the methods. \emph{AL by the true regret}: relative change of the active-learning objective, the accuracy area on DCIC and the test risk on APPA-REAL, when the true regret of the current model selects the labels instead of random acquisition ($+$: better). Two numbers: near\,/\,far OOD sets, or networks trained on the annotator distributions\,/\,on one label per image; one number: annotator distributions, one label on CIFAR-10H. Log loss, absolute loss on APPA-REAL. Gray: the regret rankings of two seeds are not significantly correlated at the $5\%$ level; left out of the DCIC median. QualityMRI has no active learning: its pool is smaller than the budget.}
\label{tab:summary}
\begin{tabular}{@{}l ccccc @{\hspace{7pt}} cc @{\hspace{7pt}} c@{}}
\toprule
 & \multicolumn{5}{c}{Agreement of the method rankings} & \multicolumn{2}{c}{True risk and regret} & \makecell{AL by the\\true regret} \\
\cmidrule(lr){2-6}\cmidrule(lr){7-8}\cmidrule(l){9-9}
Dataset & Two seeds & \textbf{Regret--OOD} & \textbf{Regret--AL} & OOD--AL & Regret--risk & Regret share & Correlation & vs random, \% \\
 & & \scriptsize near\,/\,far &  & \scriptsize near\,/\,far & \scriptsize distr.\,/\,one & \scriptsize distr.\,/\,one & \scriptsize distr.\,/\,one & \scriptsize distr.\,/\,one \\
\midrule
CIFAR-10H & 0.93 & 0.18/0.04 & 0.18 & 0.34/0.35 & --\,/\,0.96 & --\,/\,0.69 & --\,/\,0.99 & --\,/\,-- \\
\addlinespace[2pt]
\multicolumn{9}{@{}l}{DCIC} \\
\quad Benthic & 0.71 & 0.27 & $-$0.06 & $-$0.30 & 0.87\,/\,0.88 & 0.62\,/\,0.71 & 0.87\,/\,0.96 & $-$1.2\,/\,$-$7.4 \\
\quad MiceBone & 0.90 & $-$0.05 & 0.12 & $-$0.21 & 0.98\,/\,0.87 & 0.37\,/\,0.55 & 0.76\,/\,0.95 & $-$0.2\,/\,$+$0.5 \\
\quad Pig & 0.56 & 0.21 & $-$0.26 & $-$0.26 & 0.81\,/\,0.64 & 0.34\,/\,0.44 & 0.50\,/\,0.79 & $-$2.8\,/\,$-$6.4 \\
\quad \textcolor{black!45}{Plankton} & \textcolor{black!45}{0.41} & \textcolor{black!45}{$-$0.01} & \textcolor{black!45}{$-$0.01} & \textcolor{black!45}{$-$0.17} & \textcolor{black!45}{0.86\,/\,0.98} & \textcolor{black!45}{0.54\,/\,0.67} & \textcolor{black!45}{0.97\,/\,0.99} & \textcolor{black!45}{0.0\,/\,$-$1.8} \\
\quad \textcolor{black!45}{QualityMRI} & \textcolor{black!45}{0.02} & \textcolor{black!45}{0.00} & \textcolor{black!45}{--} & \textcolor{black!45}{--} & \textcolor{black!45}{$-$0.27\,/\,0.90} & \textcolor{black!45}{0.05\,/\,0.48} & \textcolor{black!45}{0.12\,/\,0.93} & \textcolor{black!45}{--\,/\,--} \\
\quad Treeversity\#1 & 0.68 & 0.10 & $-$0.14 & 0.64 & 0.64\,/\,0.97 & 0.56\,/\,0.65 & 0.90\,/\,0.95 & $+$0.1\,/\,$-$3.9 \\
\quad Treeversity\#6 & 0.82 & 0.52 & 0.09 & 0.39 & 0.25\,/\,0.98 & 0.23\,/\,0.59 & 0.52\,/\,0.95 & $+$0.3\,/\,$-$4.0 \\
\quad \textcolor{black!45}{Turkey} & \textcolor{black!45}{0.50} & \textcolor{black!45}{0.30} & \textcolor{black!45}{0.46} & \textcolor{black!45}{0.02} & \textcolor{black!45}{0.86\,/\,0.96} & \textcolor{black!45}{0.52\,/\,0.65} & \textcolor{black!45}{0.86\,/\,0.96} & \textcolor{black!45}{$+$1.4\,/\,$-$1.0} \\
\cmidrule(l){2-9}
\quad \textit{median} & \textit{0.71} & \textit{0.21} & \textit{$-$0.06} & \textit{$-$0.21} & \textit{0.81\,/\,0.88} & \textit{0.37\,/\,0.59} & \textit{0.76\,/\,0.95} & \textit{$-$0.2\,/\,$-$4.0} \\
\addlinespace[2pt]
APPA-REAL & 0.92 & 0.64/0.47 & $-$0.08 & 0.21/0.31 & 0.92\,/\,0.90 & 0.35\,/\,0.37 & 0.69\,/\,0.72 & $+$2.8\,/\,$+$7.5 \\
\bottomrule
\end{tabular}
\end{table}

{

\Cref{tab:summary} summarizes the results on every dataset: how the tasks rank the methods, whether the true regret selects better labels for active learning than random acquisition, and the relation between the true risk and the regret. \Cref{app:cifar_results,app:dcic_appa_results} give the detailed results.

\paragraph{Summary.}
The tasks disagree on which method is best. On most datasets, the rankings by OOD detection, active learning and the regret task are weakly correlated or uncorrelated, whereas near- and far-OOD detection rank the methods alike. Even the true regret, the score a perfect estimate would give, selects worse labels for active learning than standard uncertainty scores. Selective prediction is the exception: it ranks the methods like the regret task where the true risk and the regret are strongly correlated, which holds on many datasets, but training on the annotator distributions weakens this correlation and can separate the two tasks. Collecting more datasets with repeated labels on which the two tasks decouple is a clear future direction for evaluating epistemic uncertainty.

\paragraph{Agreement of the tasks.}
The rankings by \emph{OOD detection} and by the \emph{regret task} are nearly uncorrelated on CIFAR-10H, with Spearman correlations of $0.18$ for the near and $0.04$ for the far OOD sets, weakly correlated on DCIC, with a median of $0.21$, and moderately correlated on APPA-REAL, at $0.64$ and $0.47$. The rankings by \emph{active learning} and by the \emph{regret task} correlate at $0.18$ on CIFAR-10H, at a median of $-0.06$ on DCIC and at $-0.08$ on APPA-REAL. Under the Brier and the 0/1 loss, the regret ranking on CIFAR-10H agrees more with OOD detection, from $0.22$ to $0.53$, and as little with active learning, at $0.24$ and $0.05$ (\cref{fig:rank_grid}). The rankings by \emph{OOD detection} and by \emph{active learning} correlate at $0.34$ and $0.35$ on CIFAR-10H, at a median of $-0.21$ on DCIC and at $0.21$ and $0.31$ on APPA-REAL. \emph{Near-} and \emph{far-OOD detection}, in contrast, rank the methods alike, at $0.92$ on CIFAR-10H and $0.74$ on APPA-REAL. The rankings are reproducible across seeds on most datasets; the three DCIC datasets on which they are unstable are gray in \cref{tab:summary} and left out of its medians.

\paragraph{Active learning by the true regret.}
Active learning is a common proxy task for evaluating epistemic uncertainty, which this paper interprets as the regret. This raises a question: \emph{Is the \textbf{true} regret a good acquisition score?} DCIC and APPA-REAL allow a direct test: they provide an annotator distribution for every image of the pool $\pool$, so the true regret $\Estar(x, \lab)$ of the current model can serve as the acquisition score. On both, standard scores computed by the same network, such as the entropy of its prediction, generally select better labels than the true regret, which on DCIC stays close to or below random acquisition (\cref{app:dcic_appa_results}). The optimal score of active learning, the expected error reduction, depends on how a label changes the retrained model rather than on the current regret alone.

\paragraph{Selective prediction and the regret task.}
Of the tasks, selective prediction ranks the methods most like the regret task, at $0.96$ on CIFAR-10H, $0.92$ on APPA-REAL and a median of $0.81$ on DCIC. This agreement is a property of the data and the predictors rather than of the tasks. The two tasks order the same predictions by the true risk $\Tstar = \Astar + \Estar$ and by its reducible part $\Estar$, and on most of our datasets $\Tstar$ orders the test inputs similarly to $\Estar$, so a score that selects well for one task selects well for the other. The two tasks separate as the predictor approaches the annotator distributions and its true risk becomes mostly aleatoric. Training on the annotator distributions lowers the median fraction of the true risk that is regret from $0.59$ to $0.37$ and its correlation with the true risk from $0.95$ to $0.76$. On Treeversity\#6 the rankings of the two tasks become unrelated, with a Spearman correlation of $0.25$, while the regret rankings of two seeds still correlate at $0.82$.

\paragraph{The Pareto gap.}
Between selective prediction and the regret task lie the tasks that trade the risk against the regret (\cref{sec:family}). By \cref{thm:unified}, each of them is solved by a threshold on a weighted sum of the true aleatoric and epistemic components, $(1-\lambda)\Astar + \lambda\Estar$, so a method that estimates the two components separately can in principle solve all of them by thresholding a weighted sum of $\Ahat$ and $\Ehat$. The Pareto gap measures how far it falls short of the optimal solutions. Previous work judges the separation by the rank correlation of $\Ahat$ and $\Ehat$ and takes a high correlation as a failure \citep{mucsanyi2024benchmarking}, but the true components are correlated themselves, at $0.87$ on CIFAR-10H and up to $0.62$ on DCIC. On most of our datasets the true risk and the regret order the inputs similarly, so a single true score is almost optimal for every trade-off, and the data cannot show whether separate estimates help. On Pig and Treeversity\#6, where the two clearly differ, separate estimates could help, but those of current methods help little: in the median over methods, the best weighted sum of a method's two estimates has only a $4\%$ and $9\%$ smaller gap than its best single estimate, and both remain far from the optimal solutions, with gaps of $0.40$ and $0.28$ ($0$: optimal, $1$: random; \cref{app:surface}). Current datasets can rarely evaluate how well methods disentangle aleatoric from epistemic uncertainty; this requires new high-quality datasets with repeated labels on which the true risk and the regret differ.

}

\section{Conclusion and Limitations}
\label{sec:conclusion}
We show that OOD detection, active learning, and regret minimization optimize distinct objectives: the input-density ratio, expected error reduction through retraining, and reducible risk, respectively. We unify selective classification and the epistemic reject-option in a constrained-optimization framework whose optimal score is a convex combination of the true aleatoric and epistemic components. We evaluate how well estimated components recover the optimal trade-offs between coverage, risk, and regret through the Pareto gap.
Empirically, the tasks rank methods differently, and even the true regret selects worse labels for active learning than standard uncertainty scores. Selective prediction ranks methods similarly to the regret task on most current datasets, where true risk and regret tend to order inputs similarly. For the same reason, these datasets can rarely reveal whether methods disentangle aleatoric from epistemic uncertainty; where risk and regret differ, current methods largely fail to capture this distinction, and combining their estimates offers little improvement over using either alone.
Epistemic uncertainty in the sense of active learning or OOD detection is therefore different from epistemic uncertainty in the sense of regret, and its evaluation should specify the task it is meant to support. Further evaluation of regret and uncertainty decomposition requires high-quality datasets with repeated labels on which true risk and regret differ. We see constructing such datasets as a key step toward advancing the evaluation of epistemic uncertainty.

\paragraph{Limitations.}
Calculating the true regret requires access to the exact data-generating distribution $\ptrue(y \mid x)$. Consequently, the regret task and the measures of the family are benchmarking tools for datasets with repeated labels and cannot be computed on single-label data. We compute them against the empirical annotator distribution, which serves as a benchmark reference. The regret is measured against the Bayes-optimal decision under this reference, an oracle that is not restricted to the model class, so it contains the approximation and optimization error together with the estimation error. We call this quantity the operational epistemic uncertainty of the predictor and do not separate its parts. The Pareto gap is an oracle evaluation that selects the thresholds, and for a pair the weights, on the test set. We use one architecture per dataset and three seeds, five on APPA-REAL.

\paragraph{Broader impact.}
This work constitutes research on the evaluation of machine learning reliability. We do not foresee any negative societal impacts.

}
\begin{ack}
This work was supported by the Czech Science Foundation under project No.~26-22444S. Jakub Paplhám was also supported by the Grant Agency of the Czech Technical University in Prague, grant No.~SGS26/074/OHK3/1T/13. Willem Waegeman received funding from the Flemish Government under the Flanders AI Research programme. Eyke Hüllermeier gratefully acknowledges funding by the German Research Foundation (534429653, 467367360). The authors declare no competing interests.
\end{ack}

\bibliographystyle{plainnat}
\bibliography{references}

\begin{thebibliography}{69}
\providecommand{\natexlab}[1]{#1}
\providecommand{\url}[1]{\texttt{#1}}
\expandafter\ifx\csname urlstyle\endcsname\relax
  \providecommand{\doi}[1]{doi: #1}\else
  \providecommand{\doi}{doi: \begingroup \urlstyle{rm}\Url}\fi

\bibitem[Agustsson et~al.(2017)Agustsson, Timofte, Escalera, Baro, Guyon, and
  Rothe]{agustsson2017appareal}
Eirikur Agustsson, Radu Timofte, Sergio Escalera, Xavier Baro, Isabelle Guyon,
  and Rasmus Rothe.
\newblock Apparent and real age estimation in still images with deep residual
  regressors on {APPA-REAL} database.
\newblock In \emph{12th IEEE International Conference on Automatic Face \&
  Gesture Recognition (FG 2017)}, pages 87--94. IEEE, 2017.
\newblock \doi{10.1109/FG.2017.20}.

\bibitem[Azizi et~al.(2026)Azizi, Bodik, Heiss, and Yu]{azizi2026clear}
Ilia Azizi, Juraj Bodik, Jakob Heiss, and Bin Yu.
\newblock {CLEAR}: Calibrated learning for epistemic and aleatoric risk.
\newblock In \emph{The Fourteenth International Conference on Learning
  Representations}, 2026.

\bibitem[Bai et~al.(2025)Bai, Cai, Chen, Chen, Chen, et~al.]{qwen3vl}
Shuai Bai, Yuxuan Cai, Ruizhe Chen, Keqin Chen, Xionghui Chen, et~al.
\newblock {Qwen3-VL} technical report.
\newblock \emph{arXiv preprint arXiv:2511.21631}, 2025.

\bibitem[Bengs et~al.(2022)Bengs, H{\"u}llermeier, and
  Waegeman]{bengs2022pitfalls}
Viktor Bengs, Eyke H{\"u}llermeier, and Willem Waegeman.
\newblock Pitfalls of epistemic uncertainty quantification through loss
  minimisation.
\newblock In \emph{Advances in Neural Information Processing Systems}, 2022.

\bibitem[Bickford~Smith et~al.(2025)Bickford~Smith, Kossen, Trollope, van~der
  Wilk, Foster, and Rainforth]{bickfordsmith2025rethinking}
Freddie Bickford~Smith, Jannik Kossen, Eleanor Trollope, Mark van~der Wilk,
  Adam Foster, and Tom Rainforth.
\newblock Rethinking aleatoric and epistemic uncertainty.
\newblock In \emph{International Conference on Machine Learning}, 2025.

\bibitem[Blundell et~al.(2015)Blundell, Cornebise, Kavukcuoglu, and
  Wierstra]{blundell2015weight}
Charles Blundell, Julien Cornebise, Koray Kavukcuoglu, and Daan Wierstra.
\newblock Weight uncertainty in neural networks.
\newblock In \emph{International Conference on Machine Learning}, 2015.

\bibitem[Cattelan and Silva(2024)]{cattelan2024how}
Lu{\'\i}s Felipe~Prates Cattelan and Danilo Silva.
\newblock How to fix a broken confidence estimator: Evaluating post-hoc methods
  for selective classification with deep neural networks.
\newblock In \emph{The 40th Conference on Uncertainty in Artificial
  Intelligence}, 2024.
\newblock URL \url{https://openreview.net/forum?id=IJBWLRCvYX}.

\bibitem[Charpentier et~al.(2020)Charpentier, Z{\"u}gner, and
  G{\"u}nnemann]{charpentier2020posterior}
Bertrand Charpentier, Daniel Z{\"u}gner, and Stephan G{\"u}nnemann.
\newblock Posterior network: Uncertainty estimation without {OOD} samples via
  density-based pseudo-counts.
\newblock In \emph{Advances in Neural Information Processing Systems}, 2020.

\bibitem[Charpentier et~al.(2022)Charpentier, Borchert, Z{\"u}gner, Geisler,
  and G{\"u}nnemann]{charpentier2022natural}
Bertrand Charpentier, Oliver Borchert, Daniel Z{\"u}gner, Simon Geisler, and
  Stephan G{\"u}nnemann.
\newblock Natural posterior network: Deep {B}ayesian predictive uncertainty for
  exponential family distributions.
\newblock In \emph{International Conference on Learning Representations}, 2022.

\bibitem[Chow(1970)]{chow1970optimum}
C~K Chow.
\newblock On optimum recognition error and reject tradeoff.
\newblock \emph{IEEE Transactions on Information Theory}, 16\penalty0
  (1):\penalty0 41--46, 1970.

\bibitem[Cortes et~al.(2016)Cortes, DeSalvo, and Mohri]{cortes2016learning}
Corinna Cortes, Giulia DeSalvo, and Mehryar Mohri.
\newblock Learning with rejection.
\newblock In \emph{International Conference on Algorithmic Learning Theory},
  2016.

\bibitem[Daxberger et~al.(2021)Daxberger, Kristiadi, Immer, Eschenhagen, Bauer,
  and Hennig]{daxberger2021laplace}
Erik Daxberger, Agustinus Kristiadi, Alexander Immer, Runa Eschenhagen,
  Matthias Bauer, and Philipp Hennig.
\newblock Laplace redux -- effortless {B}ayesian deep learning.
\newblock In \emph{Advances in Neural Information Processing Systems}, 2021.

\bibitem[Deng et~al.(2020)Deng, Guo, Ververas, Kotsia, and
  Zafeiriou]{deng2020retinaface}
Jiankang Deng, Jia Guo, Evangelos Ververas, Irene Kotsia, and Stefanos
  Zafeiriou.
\newblock {RetinaFace}: Single-shot multi-level face localisation in the wild.
\newblock In \emph{IEEE/CVF Conference on Computer Vision and Pattern
  Recognition}, 2020.

\bibitem[Depeweg et~al.(2018)Depeweg, Hern{\'a}ndez-Lobato, Doshi-Velez, and
  Udluft]{depeweg2018decomposition}
Stefan Depeweg, Jos{\'e}~Miguel Hern{\'a}ndez-Lobato, Finale Doshi-Velez, and
  Steffen Udluft.
\newblock Decomposition of uncertainty in {B}ayesian deep learning for
  efficient and risk-sensitive learning.
\newblock In \emph{International Conference on Machine Learning}, pages
  1184--1193, 2018.

\bibitem[Der~Kiureghian and Ditlevsen(2009)]{derkiureghian2009aleatory}
Armen Der~Kiureghian and Ove Ditlevsen.
\newblock Aleatory or epistemic? {D}oes it matter?
\newblock \emph{Structural Safety}, 31\penalty0 (2):\penalty0 105--112, 2009.

\bibitem[Dinkelbach(1967)]{dinkelbach1967on}
Werner Dinkelbach.
\newblock On nonlinear fractional programming.
\newblock \emph{Management Science}, 13\penalty0 (7):\penalty0 492--498, 1967.
\newblock ISSN 00251909, 15265501.
\newblock URL \url{http://www.jstor.org/stable/2627691}.

\bibitem[El-Yaniv and Wiener(2010)]{elyaniv2010foundations}
Ran El-Yaniv and Yair Wiener.
\newblock On the foundations of noise-free selective classification.
\newblock \emph{Journal of Machine Learning Research}, 11:\penalty0 1605--1641,
  2010.

\bibitem[Franc and Paplh{\'a}m(2025)]{franc2025epistemic}
Vojt{\v{e}}ch Franc and Jakub Paplh{\'a}m.
\newblock Epistemic reject option prediction.
\newblock \emph{arXiv preprint arXiv:2511.04855}, 2025.

\bibitem[Franc et~al.(2023)Franc, Pr{\r{u}}{\v{s}}a, and
  Vor{\'a}{\v{c}}ek]{franc2023optimal}
Vojt{\v{e}}ch Franc, Daniel Pr{\r{u}}{\v{s}}a, and V{\'a}clav
  Vor{\'a}{\v{c}}ek.
\newblock Optimal strategies for reject option classifiers.
\newblock \emph{Journal of Machine Learning Research}, 24\penalty0
  (11):\penalty0 1--49, 2023.

\bibitem[Franc et~al.(2024)Franc, Paplham, and Prusa]{franc2024scod}
Vojtech Franc, Jakub Paplham, and Daniel Prusa.
\newblock Scod: From heuristics to theory.
\newblock In \emph{Computer Vision – ECCV 2024: 18th European Conference,
  Milan, Italy, September 29 – October 4, 2024, Proceedings, Part LXXXIV},
  page 424–441, Berlin, Heidelberg, 2024. Springer-Verlag.
\newblock ISBN 978-3-031-72906-5.
\newblock \doi{10.1007/978-3-031-72907-2_25}.
\newblock URL \url{https://doi.org/10.1007/978-3-031-72907-2_25}.

\bibitem[Gal and Ghahramani(2016)]{gal2016dropout}
Yarin Gal and Zoubin Ghahramani.
\newblock Dropout as a bayesian approximation: Representing model uncertainty
  in deep learning.
\newblock In Maria~Florina Balcan and Kilian~Q. Weinberger, editors,
  \emph{Proceedings of The 33rd International Conference on Machine Learning},
  volume~48 of \emph{Proceedings of Machine Learning Research}, pages
  1050--1059, New York, New York, USA, 20--22 Jun 2016. PMLR.
\newblock URL \url{https://proceedings.mlr.press/v48/gal16.html}.

\bibitem[Gal et~al.(2017)Gal, Islam, and Ghahramani]{gal2017deep}
Yarin Gal, Riashat Islam, and Zoubin Ghahramani.
\newblock Deep {B}ayesian active learning with image data.
\newblock In \emph{International Conference on Machine Learning}, 2017.

\bibitem[Geifman and El-Yaniv(2017)]{geifman2017selective}
Yonatan Geifman and Ran El-Yaniv.
\newblock Selective classification for deep neural networks.
\newblock In \emph{Advances in Neural Information Processing Systems}, 2017.

\bibitem[Geifman et~al.(2019)Geifman, Uziel, and El-Yaniv]{geifman2019bias}
Yonatan Geifman, Guy Uziel, and Ran El-Yaniv.
\newblock Bias-reduced uncertainty estimation for deep neural classifiers.
\newblock In \emph{International Conference on Learning Representations}, 2019.

\bibitem[Gruber and Buettner(2023)]{gruber2023uncertainty}
Sebastian~G Gruber and Florian Buettner.
\newblock Uncertainty estimates of predictions via a general bias-variance
  decomposition.
\newblock In \emph{International Conference on Artificial Intelligence and
  Statistics}, 2023.

\bibitem[Hendrickx et~al.(2024)Hendrickx, Perini, Van~der Plas, Meert, and
  Davis]{hendrickx2024machine}
Kilian Hendrickx, Lorenzo Perini, Dries Van~der Plas, Wannes Meert, and Jesse
  Davis.
\newblock Machine learning with a reject option: a survey.
\newblock \emph{Mach. Learn.}, 113\penalty0 (5):\penalty0 3073–3110, March
  2024.
\newblock ISSN 0885-6125.
\newblock \doi{10.1007/s10994-024-06534-x}.
\newblock URL \url{https://doi.org/10.1007/s10994-024-06534-x}.

\bibitem[Hendrycks and Gimpel(2017)]{hendrycks2017baseline}
Dan Hendrycks and Kevin Gimpel.
\newblock A baseline for detecting misclassified and out-of-distribution
  examples in neural networks.
\newblock In \emph{International Conference on Learning Representations}, 2017.

\bibitem[Hofman et~al.(2024)Hofman, Sale, and
  H{\"u}llermeier]{hofman2024quantifying}
Paul Hofman, Yusuf Sale, and Eyke H{\"u}llermeier.
\newblock Quantifying aleatoric and epistemic uncertainty with proper scoring
  rules.
\newblock \emph{arXiv preprint arXiv:2404.12215}, 2024.

\bibitem[Hofman et~al.(2026{\natexlab{a}})Hofman, Damke, L{\"o}hr, Thies,
  Javanmardi, Margraf, Paplh{\'a}m, Sale, H{\"u}llermeier, and
  Muschalik]{probly}
Paul Hofman, Clemens Damke, Timo L{\"o}hr, Santo Thies, Alireza Javanmardi,
  Valentin Margraf, Jakub Paplh{\'a}m, Yusuf Sale, Eyke H{\"u}llermeier, and
  Maximilian Muschalik.
\newblock {\texttt{probly}}: Uncertainty-aware machine learning.
\newblock In \emph{Advances in Neural Information Processing Systems},
  2026{\natexlab{a}}.

\bibitem[Hofman et~al.(2026{\natexlab{b}})Hofman, Sale, and
  H{\"u}llermeier]{hofman2026onesize}
Paul Hofman, Yusuf Sale, and Eyke H{\"u}llermeier.
\newblock Uncertainty quantification for machine learning: One size does not
  fit all.
\newblock In \emph{AAAI Conference on Artificial Intelligence},
  2026{\natexlab{b}}.

\bibitem[Hora(1996)]{hora1996aleatory}
Stephen~C. Hora.
\newblock Aleatory and epistemic uncertainty in probability elicitation with an
  example from hazardous waste management.
\newblock \emph{Reliability Engineering \& System Safety}, 54\penalty0
  (2--3):\penalty0 217--223, 1996.

\bibitem[Houlsby et~al.(2011)Houlsby, Husz{\'a}r, Ghahramani, and
  Lengyel]{houlsby2011bald}
Neil Houlsby, Ferenc Husz{\'a}r, Zoubin Ghahramani, and M{\'a}t{\'e} Lengyel.
\newblock Bayesian active learning for classification and preference learning.
\newblock \emph{arXiv preprint arXiv:1112.5745}, 2011.

\bibitem[H{\"u}llermeier and Waegeman(2021)]{hullermeier2021aleatoric}
Eyke H{\"u}llermeier and Willem Waegeman.
\newblock Aleatoric and epistemic uncertainty in machine learning: An
  introduction to concepts and methods.
\newblock \emph{Machine Learning}, 110\penalty0 (3):\penalty0 457--506, 2021.

\bibitem[Ishibuchi et~al.(2015)Ishibuchi, Masuda, Tanigaki, and
  Nojima]{ishibuchi2015modified}
Hisao Ishibuchi, Hiroyuki Masuda, Yuki Tanigaki, and Yusuke Nojima.
\newblock Modified distance calculation in generational distance and inverted
  generational distance.
\newblock In Ant{\'o}nio Gaspar-Cunha, Carlos Henggeler~Antunes, and
  Carlos~Coello Coello, editors, \emph{Evolutionary Multi-Criterion
  Optimization}, pages 110--125, Cham, 2015. Springer International Publishing.
\newblock ISBN 978-3-319-15892-1.

\bibitem[Jim{\'e}nez et~al.(2026)Jim{\'e}nez, J{\"u}rgens, and
  Waegeman]{jimenez2026position}
Sebasti{\'a}n Jim{\'e}nez, Mira J{\"u}rgens, and Willem Waegeman.
\newblock Position: Epistemic uncertainty estimation methods are fundamentally
  incomplete.
\newblock In \emph{Proceedings of the 43rd International Conference on Machine
  Learning}, volume 306 of \emph{Proceedings of Machine Learning Research},
  pages 170691--170717. PMLR, 2026.
\newblock URL \url{https://proceedings.mlr.press/v306/jimenez26a.html}.

\bibitem[Karush(1939)]{karush39minima}
William Karush.
\newblock Minima of functions of several variables with inequalities as side
  conditions.
\newblock Master's thesis, Department of Mathematics, University of Chicago,
  Chicago, IL, USA, 1939.

\bibitem[Kendall and Gal(2017)]{kendall2017uncertainties}
Alex Kendall and Yarin Gal.
\newblock What uncertainties do we need in {B}ayesian deep learning for
  computer vision?
\newblock In \emph{Advances in Neural Information Processing Systems}, 2017.

\bibitem[Kirsch et~al.(2019)Kirsch, van Amersfoort, and
  Gal]{kirsch2019batchbald}
Andreas Kirsch, Joost van Amersfoort, and Yarin Gal.
\newblock {BatchBALD}: Efficient and diverse batch acquisition for deep
  {B}ayesian active learning.
\newblock In \emph{Advances in Neural Information Processing Systems},
  volume~32, 2019.

\bibitem[Kotelevskii et~al.(2025)Kotelevskii, Kondratyev, Tak{\'a}{\v{c}},
  Moulines, and Panov]{kotelevskii2025from}
Nikita Kotelevskii, Vladimir Kondratyev, Martin Tak{\'a}{\v{c}}, Eric Moulines,
  and Maxim Panov.
\newblock From risk to uncertainty: Generating predictive uncertainty measures
  via bayesian estimation.
\newblock In \emph{The Thirteenth International Conference on Learning
  Representations}, 2025.
\newblock URL \url{https://openreview.net/forum?id=cWfpt2t37q}.

\bibitem[Kuhn and Tucker(1951)]{Kuhn1951Nonlinear}
Harold~W. Kuhn and Albert~W. Tucker.
\newblock Nonlinear programming.
\newblock In Jerzy Neyman, editor, \emph{Proceedings of the Second Berkeley
  Symposium on Mathematical Statistics and Probability}, pages 481--492,
  Berkeley, Calif., 1951. University of California Press.
\newblock \doi{10.1525/9780520411586-036}.

\bibitem[Lakshminarayanan et~al.(2017)Lakshminarayanan, Pritzel, and
  Blundell]{lakshminarayanan2017simple}
Balaji Lakshminarayanan, Alexander Pritzel, and Charles Blundell.
\newblock Simple and scalable predictive uncertainty estimation using deep
  ensembles.
\newblock In \emph{Proceedings of the 31st International Conference on Neural
  Information Processing Systems}, NIPS'17, page 6405–6416, Red Hook, NY,
  USA, 2017. Curran Associates Inc.
\newblock ISBN 9781510860964.

\bibitem[Lewis and Gale(1994)]{lewis1994sequential}
David~D. Lewis and William~A. Gale.
\newblock A sequential algorithm for training text classifiers.
\newblock In \emph{Proceedings of the 17th Annual International ACM SIGIR
  Conference on Research and Development in Information Retrieval}, pages
  3--12, 1994.

\bibitem[Li et~al.(2025)Li, Lu, Kirichenko, Qiu, Rudner, Bruss, and
  Wilson]{li2025position}
Yucen~Lily Li, Daohan Lu, Polina Kirichenko, Shikai Qiu, Tim G.~J. Rudner,
  C.~Bayan Bruss, and Andrew~Gordon Wilson.
\newblock Position: Supervised classifiers answer the wrong questions for {OOD}
  detection.
\newblock In \emph{Forty-second International Conference on Machine Learning
  Position Paper Track}, 2025.

\bibitem[Liu et~al.(2020)Liu, Lin, Padhy, Tran, Bedrax-Weiss, and
  Lakshminarayanan]{liu2020simple}
Jeremiah~Zhe Liu, Zi~Lin, Shreyas Padhy, Dustin Tran, Tania Bedrax-Weiss, and
  Balaji Lakshminarayanan.
\newblock Simple and principled uncertainty estimation with deterministic deep
  learning via distance awareness.
\newblock In \emph{Advances in Neural Information Processing Systems}, 2020.

\bibitem[Maddox et~al.(2019)Maddox, Izmailov, Garipov, Vetrov, and
  Wilson]{maddox2019simple}
Wesley~J. Maddox, Pavel Izmailov, Timur Garipov, Dmitry~P. Vetrov, and
  Andrew~Gordon Wilson.
\newblock A simple baseline for {B}ayesian uncertainty in deep learning.
\newblock In \emph{Advances in Neural Information Processing Systems}, 2019.

\bibitem[Mucs{\'a}nyi et~al.(2024)Mucs{\'a}nyi, Kirchhof, and
  Oh]{mucsanyi2024benchmarking}
B{\'a}lint Mucs{\'a}nyi, Michael Kirchhof, and Seong~Joon Oh.
\newblock Benchmarking uncertainty disentanglement: Specialized uncertainties
  for specialized tasks.
\newblock In \emph{Advances in Neural Information Processing Systems Datasets
  and Benchmarks Track}, 2024.

\bibitem[Mukhoti et~al.(2023)Mukhoti, Kirsch, van Amersfoort, Torr, and
  Gal]{mukhoti2023deep}
Jishnu Mukhoti, Andreas Kirsch, Joost van Amersfoort, Philip~H.S. Torr, and
  Yarin Gal.
\newblock Deep deterministic uncertainty: A new simple baseline.
\newblock In \emph{2023 IEEE/CVF Conference on Computer Vision and Pattern
  Recognition (CVPR)}, pages 24384--24394, 2023.
\newblock \doi{10.1109/CVPR52729.2023.02336}.

\bibitem[Neyman and Pearson(1933)]{neyman1933problem}
Jerzy Neyman and Egon~S. Pearson.
\newblock On the problem of the most efficient tests of statistical hypotheses.
\newblock \emph{Philosophical Transactions of the Royal Society of London.
  Series A}, 231:\penalty0 289--337, 1933.

\bibitem[Paplham and Franc(2026)]{csfd}
Jakub Paplham and Vojtech Franc.
\newblock Photo dating by facial age aggregation.
\newblock In \emph{Proceedings of the IEEE/CVF Winter Conference on
  Applications of Computer Vision (WACV)}, 2026.

\bibitem[Peterson et~al.(2019)Peterson, Battleday, Griffiths, and
  Russakovsky]{peterson2019human}
Joshua~C Peterson, Ruairidh~M Battleday, Thomas~L Griffiths, and Olga
  Russakovsky.
\newblock Human uncertainty makes classification more robust.
\newblock In \emph{IEEE/CVF International Conference on Computer Vision}, 2019.

\bibitem[Pietraszek(2005)]{pietraszek2005optimizing}
Tadeusz Pietraszek.
\newblock Optimizing abstaining classifiers using {ROC} analysis.
\newblock In \emph{Proceedings of the 22nd International Conference on Machine
  Learning}, pages 665--672, 2005.

\bibitem[Rossellini et~al.(2024)Rossellini, Foygel~Barber, and
  Willett]{rossellini2024integrating}
Raphael Rossellini, Rina Foygel~Barber, and Rebecca Willett.
\newblock Integrating uncertainty awareness into conformalized quantile
  regression.
\newblock In Sanjoy Dasgupta, Stephan Mandt, and Yingzhen Li, editors,
  \emph{Proceedings of The 27th International Conference on Artificial
  Intelligence and Statistics}, volume 238 of \emph{Proceedings of Machine
  Learning Research}, pages 1540--1548. PMLR, 02--04 May 2024.
\newblock URL \url{https://proceedings.mlr.press/v238/rossellini24a.html}.

\bibitem[Roy and McCallum(2001)]{roy2001toward}
Nicholas Roy and Andrew McCallum.
\newblock Toward optimal active learning through sampling estimation of error
  reduction.
\newblock In \emph{International Conference on Machine Learning}, 2001.

\bibitem[Sale et~al.(2023{\natexlab{a}})Sale, Caprio, and
  H{\"u}llermeier]{sale2023volume}
Yusuf Sale, Michele Caprio, and Eyke H{\"u}llermeier.
\newblock Is the volume of a credal set a good measure for epistemic
  uncertainty?
\newblock In \emph{Proceedings of the Thirty-Ninth Conference on Uncertainty in
  Artificial Intelligence}, volume 216 of \emph{Proceedings of Machine Learning
  Research}, pages 1795--1804, 2023{\natexlab{a}}.

\bibitem[Sale et~al.(2023{\natexlab{b}})Sale, Hofman, Wimmer, H{\"u}llermeier,
  and Nagler]{sale2023variance}
Yusuf Sale, Paul Hofman, Lisa Wimmer, Eyke H{\"u}llermeier, and Thomas Nagler.
\newblock Second-order uncertainty quantification: Variance-based measures.
\newblock \emph{arXiv preprint arXiv:2401.00276}, 2023{\natexlab{b}}.

\bibitem[Sale et~al.(2024)Sale, Bengs, Caprio, and
  H{\"u}llermeier]{sale2024distance}
Yusuf Sale, Viktor Bengs, Michele Caprio, and Eyke H{\"u}llermeier.
\newblock Second-order uncertainty quantification: A distance-based approach.
\newblock In \emph{Proceedings of the 41st International Conference on Machine
  Learning}, volume 235 of \emph{Proceedings of Machine Learning Research},
  pages 43060--43076, 2024.

\bibitem[Schmarje et~al.(2022)Schmarje, Grossmann, Zelenka, Dippel, Kiko,
  Oszust, Pastell, Stracke, Valros, Volkmann, and Koch]{schmarje2022benchmark}
Lars Schmarje, Vasco Grossmann, Claudius Zelenka, Sabine Dippel, Rainer Kiko,
  Mariusz Oszust, Matti Pastell, Jenny Stracke, Anna Valros, Nina Volkmann, and
  Reinhard Koch.
\newblock {Is one annotation enough? A data-centric image classification
  benchmark for noisy and ambiguous label estimation}.
\newblock \emph{Advances in Neural Information Processing Systems},
  35:\penalty0 33215--33232, 2022.
\newblock URL
  \url{https://proceedings.neurips.cc/paper_files/paper/2022/file/d6c03035b8bc551f474f040fe8607cab-Paper-Datasets_and_Benchmarks.pdf}.

\bibitem[Senge et~al.(2014)Senge, B{\"o}sner, Dembczy{\'n}ski, Haasenritter,
  Hirsch, Donner-Banzhoff, and H{\"u}llermeier]{senge2014reliable}
Robin Senge, Stefan B{\"o}sner, Krzysztof Dembczy{\'n}ski, J{\"o}rg
  Haasenritter, Oliver Hirsch, Norbert Donner-Banzhoff, and Eyke
  H{\"u}llermeier.
\newblock Reliable classification: Learning classifiers that distinguish
  aleatoric and epistemic uncertainty.
\newblock \emph{Information Sciences}, 255:\penalty0 16--29, 2014.

\bibitem[Sensoy et~al.(2018)Sensoy, Kaplan, and Kandemir]{sensoy2018evidential}
Murat Sensoy, Lance Kaplan, and Melih Kandemir.
\newblock Evidential deep learning to quantify classification uncertainty.
\newblock In \emph{Proceedings of the 32nd International Conference on Neural
  Information Processing Systems}, NIPS'18, page 3183–3193, Red Hook, NY,
  USA, 2018. Curran Associates Inc.

\bibitem[Settles(2009)]{settles2009active}
Burr Settles.
\newblock Active learning literature survey.
\newblock Technical Report 1648, University of Wisconsin--Madison, Department
  of Computer Sciences, 2009.

\bibitem[Smith and Gal(2018)]{smith2018understanding}
Lewis Smith and Yarin Gal.
\newblock Understanding measures of uncertainty for adversarial example
  detection.
\newblock In \emph{Proceedings of the 34th Conference on Uncertainty in
  Artificial Intelligence}, volume R16 of \emph{Proceedings of Machine Learning
  Research}, pages 559--568, 2018.

\bibitem[Traub et~al.(2024)Traub, Bungert, L\"{u}th, Baumgartner, Maier-Hein,
  Maier-Hein, and J\"{a}ger]{traub2024overcoming}
Jeremias Traub, Till~J. Bungert, Carsten~T. L\"{u}th, Michael Baumgartner,
  Klaus~H. Maier-Hein, Lena Maier-Hein, and Paul~F. J\"{a}ger.
\newblock Overcoming common flaws in the evaluation of selective classification
  systems.
\newblock In A.~Globerson, L.~Mackey, D.~Belgrave, A.~Fan, U.~Paquet,
  J.~Tomczak, and C.~Zhang, editors, \emph{Advances in Neural Information
  Processing Systems}, volume~37, pages 2323--2347. Curran Associates, Inc.,
  2024.
\newblock \doi{10.52202/079017-0076}.
\newblock URL
  \url{https://proceedings.neurips.cc/paper_files/paper/2024/file/047c84ec50bd8ea29349b996fc64af4b-Paper-Conference.pdf}.

\bibitem[Valdenegro-Toro(2019)]{valdenegro2019deep}
Matias Valdenegro-Toro.
\newblock Deep sub-ensembles for fast uncertainty estimation in image
  classification.
\newblock \emph{arXiv preprint arXiv:1910.08168}, 2019.

\bibitem[Wan et~al.(2013)Wan, Zeiler, Zhang, LeCun, and
  Fergus]{wan2013regularization}
Li~Wan, Matthew Zeiler, Sixin Zhang, Yann LeCun, and Rob Fergus.
\newblock Regularization of neural networks using {DropConnect}.
\newblock In \emph{International Conference on Machine Learning}, 2013.

\bibitem[Wen et~al.(2020)Wen, Tran, and Ba]{wen2020batchensemble}
Yeming Wen, Dustin Tran, and Jimmy Ba.
\newblock {BatchEnsemble}: An alternative approach to efficient ensemble and
  lifelong learning.
\newblock In \emph{International Conference on Learning Representations}, 2020.

\bibitem[Wimmer et~al.(2023)Wimmer, Sale, Hofman, Bischl, and
  H{\"u}llermeier]{wimmer2023quantifying}
Lisa Wimmer, Yusuf Sale, Paul Hofman, Bernd Bischl, and Eyke H{\"u}llermeier.
\newblock Quantifying aleatoric and epistemic uncertainty in machine learning:
  Are conditional entropy and mutual information appropriate measures?
\newblock In \emph{Proceedings of the Thirty-Ninth Conference on Uncertainty in
  Artificial Intelligence (UAI)}, volume 216, pages 2282--2292, 2023.

\bibitem[Xia and Bouganis(2023)]{sircscod2023}
Guoxuan Xia and Christos-Savvas Bouganis.
\newblock Augmenting softmax information for selective classification with
  out-of-distribution data.
\newblock In Lei Wang, Juergen Gall, Tat-Jun Chin, Imari Sato, and Rama
  Chellappa, editors, \emph{Computer Vision -- ACCV 2022}, pages 664--680,
  Cham, 2023. Springer Nature Switzerland.
\newblock ISBN 978-3-031-26351-4.

\bibitem[Zhang et~al.(2024)Zhang, Yang, Wang, Wang, Lin, Zhang, Sun, Du, Li,
  Liu, Chen, and Li]{zhang2023openood}
Jingyang Zhang, Jingkang Yang, Pengyun Wang, Haoqi Wang, Yueqian Lin, Haoran
  Zhang, Yiyou Sun, Xuefeng Du, Yixuan Li, Ziwei Liu, Yiran Chen, and Hai Li.
\newblock {OpenOOD} v1.5: Enhanced benchmark for out-of-distribution detection.
\newblock \emph{Journal of Data-centric Machine Learning Research}, 2, 2024.
\newblock URL \url{https://openreview.net/forum?id=cnnTnJQigs}.

\bibitem[Zheng et~al.(2022)Zheng, Yang, Zhang, Bao, Chen, Huang, Yuan, Chen,
  Zeng, and Wen]{laionface}
Yinglin Zheng, Hao Yang, Ting Zhang, Jianmin Bao, Dongdong Chen, Yangyu Huang,
  Lu~Yuan, Dong Chen, Ming Zeng, and Fang Wen.
\newblock General facial representation learning in a visual-linguistic manner.
\newblock In \emph{IEEE/CVF Conference on Computer Vision and Pattern
  Recognition}, 2022.

\end{thebibliography}

\newpage
\appendix

\begin{center}
    \vspace*{2em}
    {\LARGE \textbf{Appendix: Evaluating Epistemic Uncertainty} \par}
    \vspace{1.5em}
\end{center}

{
\section{The Realized Components}
\label{app:bayesian}

{
This appendix gives the realized components of \cref{eq:realized} in closed form for the losses of our experiments and for the squared loss of \cref{fig:sandbox}. It then relates them to the expected components of \cref{def:components}, which gives the two identities stated in \cref{sec:setting}: for the Bayesian predictor, the mutual information and the posterior variance of the predictive mean are the expected log-loss and squared-loss regrets.

Fix an input $x$ and the training set $\Data$, and write $p = \ptrue(\cdot \mid x)$ for the true conditional and $q = q(\cdot \mid x)$ for the predictive distribution of the model. The model predicts $\learnedH(x, \Data)$, the prediction that is optimal under $q$, and the Bayes-optimal predictor predicts $h(x, \theta^*)$, the prediction that is optimal under $p$ (\cref{eq:predictors}). $\Tstar(x)$ and $\Astar(x)$ are the expected losses of these two predictions under $p$, and $\Estar(x) = \Tstar(x) - \Astar(x)$.

\paragraph{Classification.} With $K$ classes, $p$ and $q$ lie in the probability simplex, and $e_y$ denotes the one-hot vector of class $y$. Under the log loss $\loss(q, y) = -\log q(y)$ and the Brier loss $\loss(q, y) = \|q - e_y\|_2^2$, both proper scoring rules, the prediction is a distribution: $q$ for the model and $p$ for the Bayes-optimal predictor. Under the 0/1 loss $\loss(\hat y, y) = \indicator{\hat y \neq y}$ it is a class: $\hat y \in \argmax_y q(y)$ for the model and $y^* \in \argmax_y p(y)$ for the Bayes-optimal predictor. The expected loss of the model splits into that of the Bayes-optimal prediction and a remainder,
\begin{align*}
    \E_{y \sim p}\big[-\log q(y)\big] &= \mathbb{H}(p) + \mathrm{KL}(p \,\|\, q), \\
    \E_{y \sim p}\big[\|q - e_y\|_2^2\big] &= \big(1 - \|p\|_2^2\big) + \|q - p\|_2^2, \\
    \E_{y \sim p}\big[\indicator{\hat y \neq y}\big] &= \big(1 - p(y^*)\big) + \big(p(y^*) - p(\hat y)\big),\end{align*}
with the entropy $\mathbb{H}(p) = -\sum_y p(y) \log p(y)$ and the Kullback--Leibler divergence $\mathrm{KL}(p \,\|\, q) = \sum_y p(y) \log\big(p(y)/q(y)\big)$ \citep{depeweg2018decomposition, kotelevskii2025from}; the Brier line is from \citet{gruber2023uncertainty} and \citet{hofman2024quantifying}, the 0/1 line from \citet{franc2025epistemic}. The first term is the Bayes risk and the second the regret:
\begin{align*}
    \text{log loss:} && \Astar[\log] &= \mathbb{H}(p), & \Estar[\log] &= \mathrm{KL}(p \,\|\, q); \\
    \text{Brier:}    && \Astar[\text{Br}] &= 1 - \|p\|_2^2, & \Estar[\text{Br}] &= \|q - p\|_2^2; \\
    \text{0/1:}      && \Astar[0/1] &= 1 - p(y^*), & \Estar[0/1] &= p(y^*) - p(\hat y).\end{align*}

\paragraph{Regression.} For a real-valued target such as the apparent age of APPA-REAL, write $\mu_p$ and $\mu_q$ for the means of $p$ and $q$, $F_p(t) = \Pr_{y \sim p}(y \le t)$ for the cumulative distribution function of $p$, and $m_p$ and $m_q$ for the lower medians of $p$ and $q$. The expected squared loss is minimized by the mean and the expected absolute loss by a median, so the model predicts $\mu_q$ and $m_q$, and the Bayes-optimal predictor $\mu_p$ and $m_p$. The expected losses of the model split as
\begin{align*}
    \E_{y \sim p}\big[(\mu_q - y)^2\big] &= \E_{y \sim p}\big[(\mu_p - y)^2\big] + (\mu_q - \mu_p)^2, \\
    \E_{y \sim p}\big[\lvert m_q - y \rvert\big] &= \E_{y \sim p}\big[\lvert m_p - y \rvert\big] + \int_{m_p}^{m_q} \big(2F_p(t) - 1\big)\,\mathrm{d}t.\end{align*}
The first line is the bias--variance decomposition. The second holds because the expected absolute loss of a prediction $c$ is $\int_{-\infty}^{c} F_p(t)\, \mathrm{d}t + \int_{c}^{\infty} \big(1 - F_p(t)\big)\, \mathrm{d}t$, which changes by the integral of $2F_p - 1$ as $c$ moves from $m_p$ to $m_q$. The integral is non-negative, since $F_p(t) \ge \nicefrac{1}{2}$ above the median $m_p$ and $F_p(t) < \nicefrac{1}{2}$ below it. The first term is the Bayes risk, the variance of $p$ or its mean absolute deviation from the median, and the second the regret:
\begin{align*}
    \text{squared:}  && \Astar[\text{sq}] &= \E_{y \sim p}\big[(\mu_p - y)^2\big], & \Estar[\text{sq}] &= (\mu_q - \mu_p)^2; \\
    \text{absolute:} && \Astar[\text{abs}] &= \E_{y \sim p}\big[\lvert m_p - y \rvert\big], & \Estar[\text{abs}] &= \int_{m_p}^{m_q} \big(2F_p(t) - 1\big)\,\mathrm{d}t.\end{align*}
The squared regret depends on $p$ only through its mean: it is the squared distance of the prediction from the Bayes-optimal one. The absolute regret depends on $p$ only between the two medians. The integral above holds for any $p$, including the discrete annotator distributions of APPA-REAL. For a continuous $p$, it is twice the expected distance of the model's prediction from the labels that lie between the two medians,
\begin{equation*}
    \Estar[\text{abs}] = 2\, \E_{y \sim p}\big[\lvert m_q - y \rvert\, \indicator{y \text{ lies between } m_p \text{ and } m_q}\big],\end{equation*}
because then $F_p(m_p) = \nicefrac{1}{2}$, so that $\lvert 2F_p(t) - 1 \rvert$ is twice the probability between $m_p$ and $t$, and exchanging the order of integration gives the expectation. The same distance between the two medians can therefore carry a large or a small regret. This is why we evaluate APPA-REAL under the absolute loss.

\paragraph{Dependence on the loss.} The log and the Brier loss share the Bayes-optimal prediction, $p$ itself, and both regrets vanish exactly when $q = p$; for $q \neq p$ they are different divergences and can order two inputs differently. The 0/1 regret depends on $q$ only through the predicted class: it vanishes when $\hat y$ is a most probable class under $p$ and otherwise equals the gap $p(y^*) - p(\hat y)$. The losses therefore define different regret tasks, and a method can rank differently under each: on CIFAR-10H, the regret rankings under two losses correlate at $0.25$ to $0.44$ (\cref{fig:rank_grid}). Regrets under different losses also carry different units, so they can be compared in rank but not in size. For instance, consider a two-class problem in which the true conditional puts probability $0.9$ on cat and the model predicts cat with probability $0.55$: its log-loss regret is $0.29$ nats, its Brier regret $0.245$, and its 0/1 regret zero, since it still predicts the most probable class.

\paragraph{Relation to the expected components.} Putting a parameter value $\theta$ in place of $\theta^*$ in \cref{eq:realized}, in the true conditional and in $h(x, \theta^*)$, gives the components under $\theta$, $T_\theta(x, \Data)$, $A_\theta(x)$ and $E_\theta(x, \Data)$; the realized regret is $\Estar = E_{\theta^*}$. Since $p(\theta, y \mid x, \Data) = p(\theta \mid \Data)\, p(y \mid x, \theta)$, the expected components of \cref{def:components} are the posterior averages of these, $\Ebar(x, \Data) = \E_{\theta \sim p(\theta \mid \Data)}[E_\theta(x, \Data)]$, and likewise $\Tbar$ and $\Abar$. A Bayesian learner can compute only these averages. The expected regret $\Ebar$ is the optimal score of the regret task (\cref{sec:family_endpoints}) in expectation over the posterior \citep{franc2025epistemic}, and a Bayesian measure of epistemic uncertainty that equals $\Ebar$ estimates the realized regret by its expectation over the posterior. The realized components can also be defined with the true conditional alone; they then include the misspecification error, and this correspondence is no longer exact.

For the Bayesian predictor, $q$ is the posterior predictive $p(\cdot \mid x, \Data)$. Under the log loss, $E_\theta(x, \Data)$ is the Kullback--Leibler divergence of $p(\cdot \mid x, \theta)$ from $q$, and its posterior average is the mutual information of the label and the parameters \citep{depeweg2018decomposition},
\begin{equation*}
    \Ebar[\log](x, \Data) = \E_{\theta \sim p(\theta \mid \Data)}\Big[\mathrm{KL}\big(p(\cdot \mid x, \theta) \,\|\, p(\cdot \mid x, \Data)\big)\Big] = I(y; \theta \mid x, \Data).\end{equation*}
It is the BALD score (\cref{eq:bald_app}), so the same quantity both estimates the log-loss regret and solves the information formulation of active learning for one label (\cref{sec:outside_al}). Under the squared loss, write $\mu_\theta$ for the mean of $p(\cdot \mid x, \theta)$. The Bayesian predictor predicts $\mu_q = \E_{\theta \sim p(\theta \mid \Data)}[\mu_\theta]$, so $E_\theta(x, \Data) = (\mu_q - \mu_\theta)^2$ and $\Ebar[\text{sq}](x, \Data) = \E_{\theta \sim p(\theta \mid \Data)}\big[(\mu_q - \mu_\theta)^2\big]$, the posterior variance of the predictive mean.
}

\section{The Tasks in One Notation}
\label{app:tasks_full}

{
\Cref{tab:tasks_full} lists the formulations of every task of \cref{sec:tasks} in one notation. This appendix states OOD detection and active learning in full, with the optimal scores and metrics that \cref{sec:outside_ood,sec:outside_al} give in short, and describes the one-dimensional example of \cref{fig:sandbox}.
}

\begin{table}[htbp]
\centering
{
\tiny
\setlength{\tabcolsep}{2.2pt}
\renewcommand{\arraystretch}{1.2}
\caption{\textbf{The tasks in one notation.} Each row is a formulation with its parameter, definition, optimal rule and standard metric. Formulations in one merged cell share the optimal rule and differ only in its threshold, which in the cost-based formulations is the cost ($\tau = \varepsilon$, $\delta$ or $c_1/c_2$); the formulations of the joint task share the family of rules of \cref{thm:unified}. $R(\lab)$ and the scores of active learning are defined in \cref{app:task_al}. The rule $\selector$ accepts an input in the selective, regret and joint rows, flags it in the OOD rows and selects it for labeling in the active-learning rows. The selective rows follow \citet{franc2023optimal}. A dash marks a formulation without a standard metric; the joint task is evaluated by the Pareto gap (\cref{sec:family_measure}).}
\label{tab:tasks_full}
\resizebox{\textwidth}{!}{\begin{tabular}{cllll}
\toprule
Task & Formulation (parameter) & Definition & Optimal rule & Standard metric \\
\midrule
\multirow{3}{*}{\rotatebox{90}{\textbf{Selective}}}
 & cost-based (reject cost $\varepsilon$) & $\min_{\selector}\; \coverage\,\Risk + \varepsilon\,(1 - \coverage)$ & \multirow{3}{*}{$\selector^*_{\Risk} = \indicator{\Tstar \leq \tau}$} & -- \\
 & bounded-improvement (target risk $r$) & $\max_{\selector}\; \coverage \;\;\text{s.t.}\;\; \Risk \leq r$ & & -- \\
 & bounded-abstention (target coverage $\kappa$) & $\min_{\selector}\; \Risk \;\;\text{s.t.}\;\; \coverage \geq \kappa$ & & $\AuRC$, $n$-$\AuRC$ \\
\midrule
\multirow{3}{*}{\rotatebox{90}{\textbf{Regret}}}
 & cost-based (reject cost $\delta$) & $\min_{\selector}\; \coverage\,\Regret + \delta\,(1 - \coverage)$ & \multirow{3}{*}{$\selector^*_{\Regret} = \indicator{\Estar \leq \tau}$} & -- \\
 & bounded-improvement (target regret $\gamma$) & $\max_{\selector}\; \coverage \;\;\text{s.t.}\;\; \Regret \leq \gamma$ & & -- \\
 & bounded-abstention (target coverage $\kappa$) & $\min_{\selector}\; \Regret \;\;\text{s.t.}\;\; \coverage \geq \kappa$ & & $\AuReC$, $n$-$\AuReC$ \\
\midrule
\multirow{2}{*}{\rotatebox{90}{\textbf{OOD}}}
 & Neyman--Pearson (false-positive level $\alpha$) & $\max_{\selector}\; \E_{x \sim \pood}[\selector(x)] \;\;\text{s.t.}\;\; \E_{x \sim \pid}[\selector(x)] \leq \alpha$ & \multirow{2}{*}{$\selector_{\mathrm{OOD}} = \indicator{\oodlikelihood/\idlikelihood > \tau}$} & AUROC \\
 & cost-based (error costs $c_1$, $c_2$) & $\min_{\selector}\; c_1 \E_{x \sim \pid}[\selector(x)] + c_2 \E_{x \sim \pood}[1 - \selector(x)]$ & & -- \\
\midrule
\multirow{2}{*}{\rotatebox{90}{\textbf{AL}}}
 & final loss (one label) & $\min_{x \in \pool}\; \E_{y \sim \ptrue(\cdot \mid x)}\big[R(\lab \cup \{(x, y)\})\big]$ & $\selector_{\mathrm{AL}} = \indicator{\score_{\mathrm{EER}} > \tau}$ & $\NAUC$ (0/1 loss) \\
 & information (one label) & $\max_{x \in \pool}\; I(y; \theta \mid x, \lab)$ & $\selector_{\mathrm{AL}} = \indicator{\score_{\mathrm{BALD}} > \tau}$ & -- \\
\midrule
\multirow{3}{*}{\rotatebox{90}{\textbf{\begin{tabular}{@{}c@{}}Risk \&\\Regret\end{tabular}}}}
 & bounded-improvement (risk $r$, regret $\gamma$) & $\max_{\selector}\; \coverage \;\;\text{s.t.}\;\; \Risk \leq r,\; \Regret \leq \gamma$ & \multirow{3}{*}{$\indicator{(1-\lambda)\Astar + \lambda\Estar \leq \tau}$} & -- \\
 & bounded-abstention, regret budget ($\kappa$, $\gamma$) & $\min_{\selector}\; \Risk \;\;\text{s.t.}\;\; \coverage \geq \kappa,\; \Regret \leq \gamma$ & & -- \\
 & bounded-abstention, risk budget ($\kappa$, $r$) & $\min_{\selector}\; \Regret \;\;\text{s.t.}\;\; \coverage \geq \kappa,\; \Risk \leq r$ & & -- \\
\bottomrule
\end{tabular}}
}
\end{table}

\subsection{Out-of-distribution detection}

{
\paragraph{Given.} An in-distribution density $\idlikelihood$, from which the training inputs are drawn, and an out-of-distribution density $\oodlikelihood$; neither is known, but samples of both are available. A detector flags an input as out-of-distribution with probability $\selector(x) \in [0, 1]$, in practice by thresholding a score, $\selector_\tau(x) = \indicator{\score(x) > \tau}$. Its false- and true-positive rates are the fractions of in- and out-of-distribution inputs that it flags,
\begin{equation}
    \mathrm{FPR}(\selector) = \E_{x \sim \pid}\big[\selector(x)\big], \qquad
    \mathrm{TPR}(\selector) = \E_{x \sim \pood}\big[\selector(x)\big].
\end{equation}

\paragraph{Goal.} Separate the two distributions. Whether the out-of-distribution inputs come from a new class, a corruption or a different domain, the benchmarked task is a binary classification of inputs with two kinds of error, flagging an in-distribution input and missing an out-of-distribution one. A formulation either maximizes the true-positive rate at a false-positive level $\alpha$ or minimizes the error rates weighted by costs $c_1$ and $c_2$:
\begin{equation}
\begin{aligned}
    \textbf{Neyman--Pearson:}\quad & \max_{\selector}\; \mathrm{TPR}(\selector) \;\;\text{s.t.}\;\; \mathrm{FPR}(\selector) \leq \alpha, \\
    \textbf{cost-based:}\quad & \min_{\selector}\; c_1\, \mathrm{FPR}(\selector) + c_2\, \big(1 - \mathrm{TPR}(\selector)\big).
\end{aligned}
\end{equation}

\paragraph{Optimal score.} By the Neyman--Pearson lemma \citep{neyman1933problem}, both problems are solved by thresholding the likelihood ratio,
\begin{equation}
    \selector^*(x) = \indicator{\score^*_{\mathrm{OOD}}(x) > \tau}, \qquad \score^*_{\mathrm{OOD}}(x) = \frac{\oodlikelihood}{\idlikelihood},
    \label{eq:ood_opt_app}\end{equation}
with $\tau$ set to attain the level $\alpha$, randomized at ties, or $\tau = c_1/c_2$. The likelihood ratio attains the largest true-positive rate at every false-positive rate, so no other score trades the two errors better.

\paragraph{Metric.} The receiver operating characteristic (ROC) curve of a score traces its true-positive rate against its false-positive rate over the thresholds; the benchmarks report the area under it, AUROC \citep{hendrycks2017baseline, zhang2023openood},
\begin{equation}
    \begin{aligned}
    \AuROC(\score) = \int_0^1 \mathrm{TPR}\big(\selector_{\tau(\alpha)}\big)\, \mathrm{d}\alpha
    &= \Pr_{x \sim \pid,\; x' \sim \pood}\big(\score(x') > \score(x)\big) \\
    &\quad + \tfrac{1}{2} \Pr_{x \sim \pid,\; x' \sim \pood}\big(\score(x') = \score(x)\big),
    \end{aligned}
    \end{equation}
where $\tau(\alpha)$ attains the false-positive rate $\alpha$, randomized at ties. The area is thus the probability that the score ranks an out-of-distribution input above an in-distribution one, ties counted one half, and it is estimated by the fraction of test pairs that the score orders correctly. The benchmarks also report the false-positive rate at a true-positive rate of $95\%$, $\mathrm{FPR}(\selector_{\tau_{95}})$ with $\mathrm{TPR}(\selector_{\tau_{95}}) = 0.95$ \citep{zhang2023openood}. Both metrics depend on the score only through its ROC curve, so both are optimized by the likelihood ratio.

}

\subsection{Active learning}
\label{app:task_al}

{
\paragraph{Given.} A learning algorithm $\learner\colon (\Xspace \times \Yspace)^{*} \to \mathcal{P}(\Yspace)^{\Xspace}$ that maps a labeled set $\lab$, a finite set of input--label pairs, to the predictive distribution $q_\lab = \learner(\lab)$ of the model trained on it, with $\mathcal{P}(\Yspace)$ the distributions over $\Yspace$; an initial labeled set $\lab_0 = \{(x_i, y_i)\}_{i=1}^{n}$ drawn i.i.d.\ from $\ptrue(x, y)$; a pool $\pool = \{x_j\}_{j=1}^{m}$ of unlabeled inputs drawn i.i.d.\ from $\ptrue(x)$; and a budget of $B$ labels, each drawn from $\ptrue(\cdot \mid x)$. The model trained on $\lab$ predicts $\learnedH(x, \lab)$ (\cref{eq:predictors}) and has the realized components $\Tstar(x, \lab)$ and $\Estar(x, \lab)$ (\cref{eq:realized}), written with $\lab$ because the labeled set changes. Its risk is
\begin{equation}
    R(\lab) = \E_{(x, y) \sim \ptrue}\big[\loss(\learnedH(x, \lab), y)\big] = \E_{x \sim \ptrue}\big[\Tstar(x, \lab)\big].
    \end{equation}
An acquisition rule scores the pool, requests labels for the inputs with the highest score, retrains the model and repeats.

\paragraph{Goal.} The \textbf{final-loss} formulation minimizes the expected risk of the model retrained on the acquired labels. The \textbf{information} formulation, for a Bayesian learner with posterior $p(\theta \mid \lab)$ and posterior predictive $q_\lab$, maximizes the expected reduction of the posterior entropy by one label:
\begin{align}
    \textbf{final loss:}\quad & \min_{x_1, \dots, x_B \in \pool}\; \E_{y_1 \sim \ptrue(\cdot \mid x_1),\, \dots,\, y_B \sim \ptrue(\cdot \mid x_B)}\Big[R\big(\lab_0 \cup \{(x_b, y_b)\}_{b=1}^{B}\big)\Big],  \\
    \textbf{information:}\quad & \max_{x \in \pool}\; I(y; \theta \mid x, \lab). 
\end{align}

\paragraph{Optimal score.} The final-loss problem over $B$ labels is intractable in general. Its expectation runs over every combination of the $B$ unknown labels, each of which requires retraining the model and evaluating its risk under the unknown $\ptrue$, and its minimum runs over all subsets of $B$ pool inputs. It is approximated greedily, acquiring the input whose label is expected to reduce the risk most, retraining and repeating, which in general does not solve the problem for $B$ labels. The one-label problem is solved by the input with the largest expected error reduction (EER) \citep{roy2001toward},
\begin{equation}
\begin{aligned}
    \score_{\mathrm{EER}}(x) &= R(\lab) - \E_{y \sim \ptrue(\cdot \mid x)}\big[R(\lab \cup \{(x, y)\})\big] \\
    &= \E_{y \sim \ptrue(\cdot \mid x)}\Big[\E_{x' \sim \ptrue}\big[\Estar(x', \lab) - \Estar(x', \lab \cup \{(x, y)\})\big]\Big].
\end{aligned}
    \end{equation}
The second form follows from $\Tstar = \Astar + \Estar$: the Bayes risk does not depend on the labeled set and cancels in the difference. The score needs the unknown $\ptrue$ and a retrained model for every candidate and label, which is intractable for a neural network. \citet{roy2001toward} approximate $\ptrue(\cdot \mid x)$ by the model's predictive distribution and the test distribution by the pool. When $\ptrue$ is known, as in \cref{fig:sandbox}, the score is exact if the retrained prediction is affine in the new label and the risk has a closed form, as for ridge regression under the squared loss.

The information problem is solved by Bayesian active learning by disagreement (BALD) \citep{houlsby2011bald},
\begin{equation}
    \score_{\mathrm{BALD}}(x) = I(y; \theta \mid x, \lab) = \mathbb{H}\big(q_\lab(\cdot \mid x)\big) - \E_{\theta \sim p(\theta \mid \lab)}\Big[\mathbb{H}\big(p(\cdot \mid x, \theta)\big)\Big],
    \label{eq:bald_app}
\end{equation}
with $\mathbb{H}$ the entropy. By \cref{app:bayesian}, BALD equals the posterior expectation of the log-loss regret of $q_\lab$ at $x$. Acquiring by BALD is therefore optimal for one label in one round. A batch is scored by the joint mutual information of its labels and the parameters, which is at most the sum of the per-input scores \citep{kirsch2019batchbald}, so the inputs with the highest scores need not form the most informative batch. Over several rounds, acquiring by BALD in each round is in general not optimal either.

\paragraph{Metric.} A score is evaluated by the performance of the models retrained on the labels it acquires. The benchmark spends the budget in $T$ rounds, records a performance measure $M(q_{\lab_t})$ of the model after each round $t$, such as its test accuracy in classification or its test error in regression, and averages it over the rounds,
\begin{equation}
    \bar M = \frac{1}{T} \sum_{t=1}^{T} M(q_{\lab_t}).
\end{equation}
The average rewards a score whose models perform well early, with few labels, and not only at the end of the budget. With the test risk as the measure, it evaluates the final-loss formulation at every intermediate budget; with the test accuracy, the usual choice in classification, it is the normalized area under the accuracy--budget curve, $\NAUC$ \citep{settles2009active, gal2017deep}.

}

\subsection{The one-dimensional example}

{
\Cref{fig:sandbox} is computed in a one-dimensional regression problem under the squared loss, in which everything is known. The fitted model contains the model that generates the data, so that the assumption of \cref{sec:setting}, a true parameter $\theta^*$ in the model class, holds exactly: the regret contains no misspecification error, and the posterior on which the information score depends is a posterior over the correct model class. Both models are linear and Gaussian, so that every score has a closed form.

\paragraph{Ground truth.} On the domain $[0.01, 1]$, the regression function is an expansion in $100$ Gaussian radial basis functions of width $0.01$ with equally spaced centers $c_k$,
\begin{equation*}
    h^*(x) = \sum_{k=1}^{100} w_k \exp\big(-(x - c_k)^2 / 0.01^2\big).
\end{equation*}
The weights $w$ are fitted once, by least squares with a ridge penalty of $10^{-3}$, to
\begin{equation*}
    g(x) = \big(1 - 0.8\, e^{-((x - 0.3)/0.08)^2}\big) \sin\big(0.2/(x + 0.16)^3\big),
\end{equation*}
the Rossellini function of \cref{sec:outside_relations} \citep{rossellini2024integrating, jimenez2026position} with its amplitude reduced around $x = 0.3$. It oscillates fast near the left end of the domain (region A). The expansion, unlike $g$ itself, lies in the model class below. The noise is Gaussian with the known variance $\Astar(x) = 4x^4 + 10^{-3}$, large on the right (region C). The training density $\idlikelihood$ is a Beta$(2, 1)$ density restricted to the domain and lowered to $2\%$ in a well around $0.4$ (region B); the deployment density is uniform on the domain.

\paragraph{Fitted model.} The model is a Bayesian linear regression $y \mid x, \theta \sim \mathcal{N}\big(\phi(x)^{\top} \theta, \Astar(x)\big)$ with the known noise variance and the prior $\theta \sim \mathcal{N}(0, 30^{-1} I)$. Its features $\phi(x)$ are Gaussian radial basis functions at the same $100$ centers with the widths $0.01$, $0.05$, $0.1$ and $0.2$. They include the features of the ground truth, so that $h^*(x) = \phi(x)^{\top} \theta^* = h(x, \theta^*)$ for the $\theta^*$ equal to $w$ on the features of width $0.01$ and to zero elsewhere. The model is fitted once to a training set $\Data$ of $50$ points drawn with $x \sim \pid$ and $y \mid x \sim \mathcal{N}\big(h(x, \theta^*), \Astar(x)\big)$. Its posterior has the covariance $\Sigma = (\Phi^{\top} W \Phi + 30\, I)^{-1}$ and the mean $\mu = \Sigma\, \Phi^{\top} W \mathbf{y}$, with $\Phi$ the features of the training inputs, $W$ the diagonal matrix of their inverse noise variances and $\mathbf{y}$ the vector of their labels; the predictor is the posterior mean $\learnedH(x, \Data) = \phi(x)^{\top} \mu$, the weighted ridge regression with the penalty $30$.

\paragraph{Scores.} For this predictor every optimal score is exact:
\begin{itemize}\itemsep 0pt
    \item \emph{selective prediction}: the risk $\Astar(x) + \big(\learnedH(x, \Data) - h(x, \theta^*)\big)^2$;
    \item \emph{regret}: $\big(\learnedH(x, \Data) - h(x, \theta^*)\big)^2$;
    \item \emph{active learning, information}: BALD, here $\tfrac{1}{2}\log\big(1 + \phi(x)^{\top} \Sigma\, \phi(x) / \Astar(x)\big)$;
    \item \emph{active learning, final loss}: the expected error reduction for one label under the deployment density, exact because the retrained posterior mean is affine in the new label;
    \item \emph{OOD detection}: the likelihood ratio \eqref{eq:ood_opt_app} with the uniform deployment density as $\oodlikelihood$.
\end{itemize}
The scores are evaluated on $1{,}500$ equally spaced points of the domain, and each rejects, or selects for labeling, the $20\%$ of the deployment distribution with the highest value. The selected sets overlap little: their Jaccard index, the size of the intersection divided by the size of the union, is $0.26$ between selective prediction and the regret task, at most $0.01$ between selective prediction and each other task, and at most $0.35$ between any other pair.
}

\section{Proofs: The Unified Selector and the Formulation Invariance}
\label{app:proofs}

{
\Cref{sec:family} states the joint task as the problem \eqref{eq:Pc} and two further formulations, and \cref{thm:unified} gives their optimal selectors. The main text analyzes only \eqref{eq:Pc}, because the three formulations have the same optimal selectors. This appendix proves both statements: \cref{app:formulation_invariance} states the formulation invariance, and \cref{sec:proof_thm_unified,sec:proof_thm_unified_normalized} prove \cref{thm:unified}.
}

\subsection{The three formulations and the formulation invariance}
\label{app:formulation_invariance}

{
\paragraph{Setting.} Recall from \cref{sec:setting} that a selector $\selector \colon \Xspace \to [0, 1]$ accepts an input $x$ with probability $\selector(x)$. At $x$, the predictor incurs the risk $\Tstar(x, \Data)$, the Bayes-optimal predictor the Bayes risk $\Astar(x)$, and their difference is the regret $\Estar(x, \Data) = \Tstar(x, \Data) - \Astar(x)$ (\cref{eq:realized}). The three components are non-negative, because the loss is non-negative and the Bayes-optimal predictor has the smallest expected loss. The training set is fixed, so we omit it, as in \cref{sec:setting}, and the components are functions of $x$ alone. The coverage $\coverage(\selector)$ is the probability of accepting an input, and the risk $\Risk(\selector)$ and the regret $\Regret(\selector)$ are $\E_{x \sim \ptrue}[\selector(x)\,\Tstar(x)]$ and $\E_{x \sim \ptrue}[\selector(x)\,\Estar(x)]$ divided by the coverage (\cref{eq:coverage,eq:risk,eq:regret}); they are defined for selectors of positive coverage.

\paragraph{The three formulations.} For budgets $r, \gamma > 0$ on the risk and the regret and a required coverage $\kappa \in (0,1)$, the joint task \eqref{eq:Pc}, repeated as \eqref{eq:Pc_app}, and the two formulations obtained by exchanging its objective with a constraint are
\begin{align}
    (\mathrm{P}_{\coverage}):\quad &
        \max_{\selector}\; \coverage(\selector)
        \;\;\text{s.t.}\;\; \Risk(\selector) \leq r,\; \Regret(\selector) \leq \gamma,
        \label{eq:Pc_app}\\
    (\mathrm{P}_{\Risk}):\quad &
        \min_{\selector}\; \Risk(\selector)
        \;\;\text{s.t.}\;\; \coverage(\selector) \geq \kappa,\; \Regret(\selector) \leq \gamma,
        \label{eq:PR}\\
    (\mathrm{P}_{\Regret}):\quad &
        \min_{\selector}\; \Regret(\selector)
        \;\;\text{s.t.}\;\; \coverage(\selector) \geq \kappa,\; \Risk(\selector) \leq r.
        \label{eq:PG}
\end{align}

\paragraph{Generalized quantities.} The proofs also use the risk and the regret without the division by the coverage, the generalized quantities of \citet{traub2024overcoming},
\begin{align*}
    \Risk_{\mathrm{gen}}(\selector) &= \coverage(\selector)\,\Risk(\selector) = \E_{x \sim \ptrue}[\selector(x)\,\Tstar(x)], \\
    \Regret_{\mathrm{gen}}(\selector) &= \coverage(\selector)\,\Regret(\selector) = \E_{x \sim \ptrue}[\selector(x)\,\Estar(x)].\end{align*}
Unlike the selective quantities, they are linear in $\selector$. The problems with $\Risk_{\mathrm{gen}}$ and $\Regret_{\mathrm{gen}}$ in place of $\Risk$ and $\Regret$ therefore have a linear objective and linear constraints. \Cref{sec:proof_thm_unified} solves these problems, and \cref{sec:proof_thm_unified_normalized} reduces the selective problems to the same form.

\paragraph{Multipliers and coefficients.} The proofs attach a Lagrange multiplier to each constraint: $\nu_r \geq 0$ to the risk budget $r$, $\nu_\gamma \geq 0$ to the regret budget $\gamma$ and $\nu_\kappa \geq 0$ to the coverage constraint with level $\kappa$. Every Lagrangian of the proofs is, up to a constant, $\omega_{\Risk}\,\Risk_{\mathrm{gen}} + \omega_{\Regret}\,\Regret_{\mathrm{gen}} - \omega_{\coverage}\,\coverage$ (\cref{eq:general_unnormalized}), and we call $\omega_{\Risk}$, $\omega_{\Regret}$ and $\omega_{\coverage}$ the coefficients of the risk, the regret and the coverage.

\paragraph{The family.} For $\lambda \in [\nicefrac{1}{2}, 1]$ and $\tau \in \Reals \cup \{+\infty\}$, write
\begin{equation*}
    s_\lambda = (1-\lambda)\Astar + \lambda\Estar = (1-\lambda)\Tstar + (2\lambda - 1)\Estar,
    \qquad
    \selector_{\lambda, \tau} = \indicator{s_\lambda \leq \tau}.\end{equation*}
The second form of $s_\lambda$ follows from $\Astar = \Tstar - \Estar$, and its weights $1 - \lambda$ and $2\lambda - 1$ are non-negative. We call the selectors $\selector_{\lambda, \tau}$ the selectors of the family; $\selector_{\lambda, +\infty}$ accepts every input.

\paragraph{Formulation invariance.} \Cref{thm:unified}, together with its version for the generalized quantities in \cref{sec:proof_thm_unified}, shows that every feasible instance of the three problems, with the selective or with the generalized quantities, is solved by a selector of the family. Its proofs set $\lambda = (\omega_{\Risk} + \omega_{\Regret})/(2\omega_{\Risk} + \omega_{\Regret})$ (\cref{eq:lambda_mapping}), so $\lambda = \nicefrac{1}{2}$ when the regret is neither the objective nor actively constrained ($\omega_{\Regret} = 0$), and $\lambda = 1$ when the risk is neither ($\omega_{\Risk} = 0$). The converse also holds: a selector $\selector = \selector_{\lambda, \tau}$ of the family with a coverage in $(0, 1)$ solves each of the three problems when the budgets are set to its own coverage, risk and regret. These budgets are admissible: $\Estar > 0$ almost surely, since $\Estar = s_1$ has a continuous distribution, so $\Risk(\selector) \geq \Regret(\selector) > 0$. The three formulations therefore have the same optimal selectors with a coverage below one.

The converse holds because $\selector$ accepts the inputs with the lowest values of $s = s_\lambda$. Among all selectors with the same coverage, it therefore has the smallest mean of $s$ over the accepted inputs, which for a selector $\selector'$ of positive coverage is
\begin{equation*}
    \bar s(\selector') = \frac{\E_{x \sim \ptrue}[\selector'(x)\,s(x)]}{\coverage(\selector')} = (1-\lambda)\,\Risk(\selector') + (2\lambda - 1)\,\Regret(\selector'),\end{equation*}
and this smallest mean increases strictly with the coverage, since a larger coverage adds inputs whose values of $s$ exceed the mean of those already accepted, $s$ having no atoms. A selector with at most the risk and the regret of $\selector$ has at most its mean of $s$, and therefore at most its coverage, which settles $(\mathrm{P}_{\coverage})$. A selector with at least the coverage of $\selector$ has at least its mean of $s$, so if its regret is at most that of $\selector$, its risk is at least that of $\selector$, which settles $(\mathrm{P}_{\Risk})$. $(\mathrm{P}_{\Regret})$ follows with the risk and the regret exchanged. When the weight of the objective is zero, at $\lambda = 1$ in $(\mathrm{P}_{\Risk})$ or at $\lambda = \nicefrac{1}{2}$ in $(\mathrm{P}_{\Regret})$, the mean is proportional to the constrained quantity. A selector that meets the constraints then has at most the mean of $\selector$ at a coverage at least that of $\selector$. Since the smallest mean increases strictly, it has the coverage of $\selector$ and the smallest mean at it, so it equals $\selector$ up to inputs of probability zero and has the same risk and regret. For the generalized quantities, the same argument uses, in place of the mean, $\E_{x \sim \ptrue}[\selector'(x)\,(\tau - s(x))] = \tau\,\coverage(\selector') - (1-\lambda)\,\Risk_{\mathrm{gen}}(\selector') - (2\lambda - 1)\,\Regret_{\mathrm{gen}}(\selector')$, which $\selector$ maximizes; here $\tau > 0$, since $s \geq 0$ has no atom at zero and $\selector$ has a positive coverage.

\paragraph{Pareto optimality.} By the converse, a selector with at least the coverage of a selector of the family with a coverage in $(0, 1)$ and at most its risk and regret has exactly its coverage, risk and regret: \eqref{eq:Pc_app} with these budgets bounds the coverage, and \eqref{eq:PR} and \eqref{eq:PG} bound the risk and the regret. No selector therefore improves one of the three quantities without worsening another, which is the Pareto optimality of the surface $\mathcal{S}^*$ of \cref{eq:surface}.

\paragraph{Ties under the 0/1 loss.} Under the 0/1 loss, the components can have atoms, and the continuity assumption of \cref{thm:unified} fails: the regret, for instance, is zero wherever the prediction agrees with the Bayes-optimal one. The inputs with $s_\lambda = \tau$ can then have positive probability, and an optimal selector may accept a random fraction of them, as in the linear-programming construction of \citet{franc2023optimal}.
}

\subsection{Proof for the generalized quantities}
\label{sec:proof_thm_unified}

{
This subsection proves \cref{thm:unified} for the three problems with the generalized quantities $\Risk_{\mathrm{gen}}$ and $\Regret_{\mathrm{gen}}$ of \cref{app:formulation_invariance}. These problems are linear in the selector, so the Karush--Kuhn--Tucker conditions \citep{karush39minima, Kuhn1951Nonlinear} apply directly: each Lagrangian is minimized input by input, and the result is a threshold on $(1-\lambda)\Astar + \lambda\Estar$.

\paragraph{Case 1: maximizing the coverage, \eqref{eq:Pc_app}.} The problem is to find a selector $\selector \colon \Xspace \to [0, 1]$ that maximizes the coverage $\coverage(\selector)$ subject to the budgets $\Risk_{\mathrm{gen}}(\selector) \leq r$ and $\Regret_{\mathrm{gen}}(\selector) \leq \gamma$. We write it as the minimization of $-\coverage(\selector)$ and solve it in four steps.

\emph{The Lagrangian.} Each budget receives a multiplier, $\nu_r \geq 0$ for the risk and $\nu_\gamma \geq 0$ for the regret, and the Lagrangian adds the budgets, weighted by their multipliers, to the objective,
\begin{equation*}
    \mathcal{L}_{\coverage}(\selector) = -\coverage(\selector) + \nu_r\,\big(\Risk_{\mathrm{gen}}(\selector) - r\big) + \nu_\gamma\,\big(\Regret_{\mathrm{gen}}(\selector) - \gamma\big).\end{equation*}

\emph{The multipliers.} The objective and both constraints are linear in $\selector$, and the selector that rejects every input has $\Risk_{\mathrm{gen}} = \Regret_{\mathrm{gen}} = 0$, so it meets both budgets strictly. For such a problem, the Karush--Kuhn--Tucker conditions \citep{karush39minima, Kuhn1951Nonlinear} give multipliers $\nu_r, \nu_\gamma \geq 0$ such that every optimal selector minimizes $\mathcal{L}_{\coverage}$ over all selectors, and the multiplier of a budget that is not tight at the optimum is zero. All terms are finite by the first assumption of \cref{thm:unified}, since $0 \leq \selector \leq 1$ and the components are non-negative.

\emph{The Lagrangian as an expectation.} The coverage and the generalized quantities are expectations over $x \sim \ptrue$: $\coverage(\selector) = \E[\selector(x)]$, $\Risk_{\mathrm{gen}}(\selector) = \E[\selector(x)\,\Tstar(x)]$ and $\Regret_{\mathrm{gen}}(\selector) = \E[\selector(x)\,\Estar(x)]$. Substituting them and $\Tstar = \Astar + \Estar$ gives
\begin{equation*}
    \mathcal{L}_{\coverage}(\selector) = \E_{x \sim \ptrue}\big[\selector(x)\, \psi(x)\big] - \nu_r r - \nu_\gamma\gamma,
    \qquad
    \psi(x) = \nu_r\,\Astar(x) + (\nu_r + \nu_\gamma)\,\Estar(x) - 1.\end{equation*}

\emph{The minimizing selector.} The constant $-\nu_r r - \nu_\gamma\gamma$ does not depend on $\selector$, and the expectation is linear in each value $\selector(x) \in [0, 1]$: accepting $x$ with a larger probability changes the Lagrangian in proportion to $\psi(x)$. The Lagrangian therefore has no minimum where a derivative vanishes. It is smallest when every input with $\psi(x) < 0$ is accepted, $\selector(x) = 1$, since accepting it lowers the Lagrangian, and every input with $\psi(x) > 0$ is rejected, $\selector(x) = 0$, since accepting it would raise it; a selector that differs from this on a set of positive probability has a larger Lagrangian. If $\nu_r = \nu_\gamma = 0$, then $\psi = -1$, and the optimal selector accepts every input; it is $\selector_{\lambda, +\infty}$ for every $\lambda$. Otherwise $2\nu_r + \nu_\gamma > 0$, and the optimal selector accepts $x$ where $\nu_r\,\Astar(x) + (\nu_r + \nu_\gamma)\,\Estar(x) < 1$. Dividing by $2\nu_r + \nu_\gamma$, so that the two weights sum to one, gives
\begin{equation*}
    \frac{\nu_r}{2\nu_r + \nu_\gamma}\,\Astar(x) + \frac{\nu_r + \nu_\gamma}{2\nu_r + \nu_\gamma}\,\Estar(x) < \frac{1}{2\nu_r + \nu_\gamma},\end{equation*}
which is $(1 - \lambda)\,\Astar(x) + \lambda\,\Estar(x) < \tau$ with
\begin{equation*}
    \lambda = \frac{\nu_r + \nu_\gamma}{2\nu_r + \nu_\gamma}, \qquad \tau = \frac{1}{2\nu_r + \nu_\gamma}.\end{equation*}
Since $\nu_\gamma \geq 0$, the numerator of $\lambda$ is at least half of the denominator, so $\lambda \geq \nicefrac{1}{2}$; since $\nu_r \geq 0$, $\lambda \leq 1$. The inputs with $\psi(x) = 0$, where $(1 - \lambda)\,\Astar(x) + \lambda\,\Estar(x) = \tau$, have probability zero by the second assumption of \cref{thm:unified}. The optimal selector is therefore $\selector_{\lambda, \tau} = \indicator{(1-\lambda)\Astar + \lambda\Estar \leq \tau}$, up to inputs of probability zero.

\paragraph{Case 2: minimizing the risk, \eqref{eq:PR}.} The problem minimizes $\Risk_{\mathrm{gen}}(\selector)$ subject to $\kappa - \coverage(\selector) \leq 0$, with multiplier $\nu_\kappa$, and $\Regret_{\mathrm{gen}}(\selector) - \gamma \leq 0$, with multiplier $\nu_\gamma$. In the same way,
\begin{equation*}
\begin{aligned}
    \mathcal{L}_{\Risk}(\selector)
    &= \Risk_{\mathrm{gen}}(\selector) + \nu_\kappa\,\big(\kappa - \coverage(\selector)\big) + \nu_\gamma\,\big(\Regret_{\mathrm{gen}}(\selector) - \gamma\big) \\
    &= \E_{x \sim \ptrue}\big[\selector(x)\,\big(\Astar(x) + (1 + \nu_\gamma)\,\Estar(x) - \nu_\kappa\big)\big] + \nu_\kappa\kappa - \nu_\gamma\gamma.
\end{aligned}\end{equation*}
As in Case 1, the optimal selector accepts the inputs where the term in brackets is negative, and dividing it by $2 + \nu_\gamma$ gives the threshold rule with
\begin{equation*}
    \lambda = \frac{1 + \nu_\gamma}{2 + \nu_\gamma}, \qquad \tau = \frac{\nu_\kappa}{2 + \nu_\gamma}.\end{equation*}
For $\nu_\gamma = 0$, which holds in particular for an inactive regret budget, $\lambda = \nicefrac{1}{2}$; as $\nu_\gamma$ grows, $\lambda$ increases toward $1$, so $\lambda \in [\nicefrac{1}{2}, 1)$. The multipliers exist if some selector has a coverage above $\kappa$ and a regret below $\gamma$. If none has, $\gamma$ is the smallest regret at coverage $\kappa$. Since $\Estar$ has a continuous distribution (the second assumption at $\lambda = 1$), the only feasible selector is then, up to inputs of probability zero, the threshold on $\Estar$ at that coverage, the selector of the family with $\lambda = 1$.

\paragraph{Case 3: minimizing the regret, \eqref{eq:PG}.} The problem minimizes $\Regret_{\mathrm{gen}}(\selector)$ subject to $\kappa - \coverage(\selector) \leq 0$, with multiplier $\nu_\kappa$, and $\Risk_{\mathrm{gen}}(\selector) - r \leq 0$, with multiplier $\nu_r$. In the same way,
\begin{equation*}
\begin{aligned}
    \mathcal{L}_{\Regret}(\selector)
    &= \Regret_{\mathrm{gen}}(\selector) + \nu_\kappa\,\big(\kappa - \coverage(\selector)\big) + \nu_r\,\big(\Risk_{\mathrm{gen}}(\selector) - r\big) \\
    &= \E_{x \sim \ptrue}\big[\selector(x)\,\big(\nu_r\,\Astar(x) + (1 + \nu_r)\,\Estar(x) - \nu_\kappa\big)\big] + \nu_\kappa\kappa - \nu_r r.
\end{aligned}\end{equation*}
As in Case 1, the optimal selector accepts the inputs where the term in brackets is negative, and dividing it by $1 + 2\nu_r$ gives the threshold rule with
\begin{equation*}
    \lambda = \frac{1 + \nu_r}{1 + 2\nu_r}, \qquad \tau = \frac{\nu_\kappa}{1 + 2\nu_r}.\end{equation*}
For $\nu_r = 0$, which holds in particular for an inactive risk budget, $\lambda = 1$; as $\nu_r$ grows, $\lambda$ decreases toward $\nicefrac{1}{2}$, so $\lambda \in (\nicefrac{1}{2}, 1]$. The multipliers exist if some selector has a coverage above $\kappa$ and a risk below $r$. If none has, $r$ is the smallest risk at coverage $\kappa$. Since $\Tstar$ has a continuous distribution (the second assumption at $\lambda = \nicefrac{1}{2}$), the only feasible selector is then, up to inputs of probability zero, the threshold on $\Tstar$ at that coverage, the selector of the family with $\lambda = \nicefrac{1}{2}$.

\paragraph{Conclusion.} In each of the three problems with the generalized quantities, every optimal selector is, up to inputs of probability zero, a selector $\selector_{\lambda, \tau}$ of the family with $\lambda \in [\nicefrac{1}{2}, 1]$; $\tau = +\infty$ occurs only in Case 1 with $\nu_r = \nu_\gamma = 0$. This is \cref{thm:unified} for the generalized quantities.
}

\subsection{Proof for the selective quantities}
\label{sec:proof_thm_unified_normalized}

{
This subsection proves \cref{thm:unified} as the main text states it, with the selective risk and regret. These are ratios, $\Risk = \Risk_{\mathrm{gen}}/\coverage$ and $\Regret = \Regret_{\mathrm{gen}}/\coverage$, so the problems are not linear in the selector. Case 1 reduces to Case 1 of \cref{sec:proof_thm_unified}. For Cases 2 and 3, the idea is to multiply the budget by the coverage and to replace the objective by a linear one with the same minimizer. The Lagrangian then has the form of \cref{sec:proof_thm_unified}, and the argument of its Case 1 applies unchanged.

\paragraph{Case 1: maximizing the coverage, \eqref{eq:Pc_app}.} The problem is to find a selector $\selector$ of positive coverage that maximizes $\coverage(\selector)$ subject to $\Risk(\selector) \leq r$ and $\Regret(\selector) \leq \gamma$. We show in three steps that its solutions also solve a problem with the generalized quantities, to which Case 1 of \cref{sec:proof_thm_unified} applies.

\emph{An optimal selector.} Let $\selector^*$ solve the problem, and write $\coverage^* = \coverage(\selector^*) > 0$ for its coverage. Since $\Risk_{\mathrm{gen}} = \coverage\,\Risk$ and $\Regret_{\mathrm{gen}} = \coverage\,\Regret$, it has $\Risk_{\mathrm{gen}}(\selector^*) \leq r\,\coverage^*$ and $\Regret_{\mathrm{gen}}(\selector^*) \leq \gamma\,\coverage^*$.

\emph{The same selector solves a generalized problem.} Consider the problem of Case 1 of \cref{sec:proof_thm_unified} with the budgets $r\,\coverage^*$ and $\gamma\,\coverage^*$,
\begin{equation*}
    \max_{\selector}\; \coverage(\selector) \;\;\text{s.t.}\;\; \Risk_{\mathrm{gen}}(\selector) \leq r\,\coverage^*,\; \Regret_{\mathrm{gen}}(\selector) \leq \gamma\,\coverage^*.\end{equation*}
By the first step, $\selector^*$ meets these budgets. Suppose that a selector $\selector$ met them with a larger coverage, $\coverage(\selector) > \coverage^*$. Dividing the budgets by its coverage would give $\Risk(\selector) \leq r\,\coverage^*/\coverage(\selector) < r$ and likewise $\Regret(\selector) < \gamma$, so $\selector$ would meet the budgets of the selective problem with a larger coverage than $\selector^*$, which is impossible. Hence $\selector^*$ solves the generalized problem.

\emph{The optimal selector is a threshold.} The generalized budgets $r\,\coverage^*$ and $\gamma\,\coverage^*$ are positive, so Case 1 of \cref{sec:proof_thm_unified} applies to them. By it, every solution of the generalized problem accepts every input or, up to inputs of probability zero, is the selector $\selector_{\lambda, \tau}$ of the family, with $\lambda = (\nu_r + \nu_\gamma)/(2\nu_r + \nu_\gamma) \in [\nicefrac{1}{2}, 1]$ and $\tau = 1/(2\nu_r + \nu_\gamma)$ for the multipliers $\nu_r, \nu_\gamma \geq 0$ of its budgets. Hence so is $\selector^*$.

\paragraph{The common form.} The three Lagrangians of \cref{sec:proof_thm_unified} have the form
\begin{equation}
    \mathcal{L}(\selector) = \omega_{\Risk}\,\Risk_{\mathrm{gen}}(\selector) + \omega_{\Regret}\,\Regret_{\mathrm{gen}}(\selector) - \omega_{\coverage}\,\coverage(\selector) + \text{const.}
    \label{eq:general_unnormalized}\end{equation}
The coefficient of a quantity is $1$ if it is the objective and the multiplier of its constraint otherwise: $(\omega_{\Risk}, \omega_{\Regret}, \omega_{\coverage}) = (\nu_r, \nu_\gamma, 1)$ in Case 1, $(1, \nu_\gamma, \nu_\kappa)$ in Case 2 and $(\nu_r, 1, \nu_\kappa)$ in Case 3. In particular, $\omega_{\Risk}$ and $\omega_{\Regret}$ are non-negative. Substituting $\Tstar = \Astar + \Estar$, the Lagrangian is $\E_{x \sim \ptrue}[\selector(x)\,(\omega_{\Risk}\Astar(x) + (\omega_{\Risk} + \omega_{\Regret})\Estar(x) - \omega_{\coverage})]$ up to a constant. For $\omega_{\Risk}$ and $\omega_{\Regret}$ not both zero, its pointwise minimizer accepts $x$ when
\begin{equation}
    \omega_{\Risk}\,\Astar(x) + (\omega_{\Risk} + \omega_{\Regret})\,\Estar(x) \leq \omega_{\coverage}
    \quad \iff \quad
    (1-\lambda)\,\Astar(x) + \lambda\,\Estar(x) \leq \tau,
    \label{eq:lambda_mapping}\end{equation}
where $\lambda = (\omega_{\Risk} + \omega_{\Regret})/(2\omega_{\Risk} + \omega_{\Regret})$ and $\tau = \omega_{\coverage}/(2\omega_{\Risk} + \omega_{\Regret})$. Since $\omega_{\Risk}, \omega_{\Regret} \geq 0$, $\lambda \in [\nicefrac{1}{2}, 1]$, with $\lambda = \nicefrac{1}{2}$ for $\omega_{\Regret} = 0$ and $\lambda = 1$ for $\omega_{\Risk} = 0$. It therefore suffices to show that the selective problems of Cases 2 and 3 have Lagrangians of the form \eqref{eq:general_unnormalized} with non-negative $\omega_{\Risk}$ and $\omega_{\Regret}$.

\paragraph{Constraints.} The selective quantities are defined for selectors of positive coverage. For such a selector, multiplying a budget by the coverage gives an equivalent linear constraint,
\begin{equation*}
    \Risk(\selector) \leq r \;\iff\; \Risk_{\mathrm{gen}}(\selector) - r\,\coverage(\selector) \leq 0,
    \qquad
    \Regret(\selector) \leq \gamma \;\iff\; \Regret_{\mathrm{gen}}(\selector) - \gamma\,\coverage(\selector) \leq 0.\end{equation*}
With multiplier $\nu_r$, the risk budget contributes $\nu_r\,\Risk_{\mathrm{gen}}(\selector) - \nu_r r\,\coverage(\selector)$ to the Lagrangian: it adds $\nu_r$ to $\omega_{\Risk}$ and $\nu_r r$ to $\omega_{\coverage}$. With multiplier $\nu_\gamma$, the regret budget adds $\nu_\gamma$ to $\omega_{\Regret}$ and $\nu_\gamma \gamma$ to $\omega_{\coverage}$. The coverage constraint $\kappa - \coverage(\selector) \leq 0$ is linear already; with multiplier $\nu_\kappa$, it adds $\nu_\kappa$ to $\omega_{\coverage}$.

\paragraph{Objective.} Let $\selector^*$ solve \eqref{eq:PR} with the selective quantities. Every feasible selector $\selector$ has $\Risk(\selector) \geq \Risk(\selector^*)$, which after multiplication by $\coverage(\selector) > 0$ reads
\begin{equation*}
    \Risk_{\mathrm{gen}}(\selector) - \Risk(\selector^*)\,\coverage(\selector) \geq 0,\end{equation*}
with equality for $\selector = \selector^*$. Hence $\selector^*$ also minimizes the linear objective $\Risk_{\mathrm{gen}} - \Risk(\selector^*)\,\coverage$ over the feasible selectors, with minimum zero \citep{dinkelbach1967on}. This objective contributes $1$ to $\omega_{\Risk}$ and $\Risk(\selector^*)$ to $\omega_{\coverage}$. The regret objective of \eqref{eq:PG} is linearized in the same way, with the optimal regret $\Regret(\selector^*)$.

\paragraph{The Lagrangians of Cases 2 and 3.} The linearized problems have a linear objective and linear constraints. As in \cref{sec:proof_thm_unified}, the multipliers exist if some selector has a coverage above $\kappa$ and a regret below $\gamma$, for \eqref{eq:PG} a risk below $r$. If none has, $\gamma$ is the smallest regret at coverage $\kappa$, since the smallest regret increases strictly with the coverage when $\Estar$ has a continuous distribution. The only feasible selector is then, up to inputs of probability zero, the threshold on $\Estar$ at coverage $\kappa$, the selector of the family with $\lambda = 1$; for \eqref{eq:PG}, it is the threshold on $\Tstar$, with $\lambda = \nicefrac{1}{2}$. Otherwise, collecting the terms, the Lagrangians are, up to constants,
\begin{align*}
    \eqref{eq:PR}:&\quad \Risk_{\mathrm{gen}} - \Risk(\selector^*)\,\coverage - \nu_\kappa\,\coverage + \nu_\gamma\,(\Regret_{\mathrm{gen}} - \gamma\,\coverage), \\
    \eqref{eq:PG}:&\quad \Regret_{\mathrm{gen}} - \Regret(\selector^*)\,\coverage - \nu_\kappa\,\coverage + \nu_r\,(\Risk_{\mathrm{gen}} - r\,\coverage).\end{align*}
The coefficients of the risk and the regret are $(\omega_{\Risk}, \omega_{\Regret}) = (1, \nu_\gamma)$ and $(\nu_r, 1)$, as in Cases 2 and 3 of \cref{sec:proof_thm_unified}; only the coefficient of the coverage changes, to $\omega_{\coverage} = \Risk(\selector^*) + \nu_\kappa + \nu_\gamma\gamma$ and $\Regret(\selector^*) + \nu_\kappa + \nu_r r$. Since $\omega_{\Risk}$ and $\omega_{\Regret}$ are non-negative, \eqref{eq:lambda_mapping} shows that every optimal selector of Cases 2 and 3 is, up to inputs of probability zero, a threshold on $(1-\lambda)\Astar + \lambda\Estar$ with $\lambda \in [\nicefrac{1}{2}, 1]$. With Case 1, this proves \cref{thm:unified}.
}

\section{The Pareto Gap}
\label{app:surface}

{
This appendix gives the computation of the Pareto gap of \cref{sec:family_measure} and the full results for the Pareto paragraph of \cref{sec:experiments}.
}

\subsection{Definition and computation}
{
\paragraph{Definition.}
A method with estimates $(\Ahat, \Ehat)$ attains the points of its achievable surface
\begin{equation}
    \mathcal{S}(\Ahat, \Ehat) = \big\{ (\coverage(\selector), \Risk(\selector), \Regret(\selector)) \;\big|\; \selector = \indicator{w_1 \Ahat + w_2 \Ehat \leq \tau},\; w_1, w_2, \tau \in \Reals \big\},\end{equation}
with the coverage and the selective risk and regret of \cref{eq:coverage,eq:risk,eq:regret}. The Pareto gap is the distance from the Pareto-optimal surface $\mathcal{S}^{*}$ of \cref{eq:surface} to this surface. The weights and the thresholds are chosen on the test set.

\paragraph{Computation.}
On the $N$ test inputs, with the realized components of \cref{eq:realized} against the annotator distribution, a score $\score$ gives the $N$ selectors $\selector^{\score}_{k/N}$, $k = 1, \dots, N$, of \cref{sec:family_endpoints}; the empty selector has no selective risk and is left out. The optimal surface consists of the points of the scores $(1-\lambda)\Astar + \lambda\Estar$ for $51$ values of $\lambda$ equally spaced in $[\nicefrac{1}{2}, 1]$. Since multiplying $w_1$, $w_2$ and $\tau$ by the same positive number leaves the selector unchanged, only the direction of $(w_1, w_2)$ matters. The two estimates can differ in scale by orders of magnitude; for instance, the epistemic estimate of DDU is a negative feature-space log-density. An evenly spaced grid of directions would then hold almost no mixture of the two, so we first normalize each estimate by its standard deviation over the test inputs, $\sigma_{A}$ and $\sigma_{E}$. The achievable surface then consists of the points of the scores $\cos\vartheta\,\Ahat / \sigma_{A} + \sin\vartheta\,\Ehat / \sigma_{E}$ for $361$ angles $\vartheta$ equally spaced in $[0, 2\pi)$, that is, of the weighted sums with $w_1 = \cos\vartheta / \sigma_{A}$ and $w_2 = \sin\vartheta / \sigma_{E}$, so the weights can be negative. This grid contains $\Ahat$ but not $\Ehat$ or $\Ahat + \Ehat$, so we add these two scores; the gap of the pair is then never larger than that of its best single estimate. The gap of a single estimate uses the points of that estimate alone.

The coverage is replaced by $1 - \coverage$, so that smaller is better on every axis, and the selective risk and regret are divided by their values at full coverage, the test means of $\Tstar$ and $\Estar$, so that every axis is unitless. With the one-sided distance $d^{+}(u, v) = \big(\sum_{j=1}^{3} \max(v_j - u_j, 0)^2\big)^{1/2}$ of a point $u$ of the optimal surface to a point $v$ of the achievable one, the modified inverted generational distance of \citet{ishibuchi2015modified} is
\begin{equation}
    \mathrm{IGD}^{+}(\mathcal{S}^{*}, \mathcal{S}) = \frac{1}{|\mathcal{S}^{*}|}\sum_{u \in \mathcal{S}^{*}} \min_{v \in \mathcal{S}} d^{+}(u, v).\end{equation}
A random ordering has, at every coverage, an expected selective risk and regret equal to their full-coverage values, so after the scaling its best point is $(0, 1, 1)$. The Pareto gap is $\mathrm{IGD}^{+}(\mathcal{S}^{*}, \mathcal{S}) / \mathrm{IGD}^{+}(\mathcal{S}^{*}, \{(0, 1, 1)\})$: $0$ means that the method attains $\mathcal{S}^{*}$ and $1$ corresponds to a random ordering.

\paragraph{Results.}
\Cref{tab:pareto_gap} reports the gap of every method on every dataset. On every dataset but QualityMRI, the gap of the best weighted sum of a method's two estimates is, in the median over the methods, only $3\%$ to $9\%$ smaller than that of its best single estimate; a few methods gain more on some datasets, up to $19\%$ on CIFAR-10H and $22\%$ on Treeversity\#6. On QualityMRI the gap is $25\%$ to $52\%$ smaller, but the regret rankings of two seeds are uncorrelated on this dataset (\cref{app:dcic_appa_results}). \Cref{tab:family} gives, on CIFAR-10H, the gaps of the single estimates and the rank correlation of each method's two estimates.
}

\begin{table}[htbp]
\centering
{
\scriptsize
\setlength{\tabcolsep}{2.5pt}
\caption{\textbf{The Pareto gap of every method on every dataset.} Normalized Pareto gap ($0$: the optimal surface, $1$: a random ordering) of the best weighted sum of a method's two estimates, mean over the seeds; in parentheses, how much smaller it is than the gap of the method's best single estimate among $\Ehat$, $\Ahat$ and $\Ahat + \Ehat$, in \%. Median: over the methods. Log loss, absolute loss on APPA-REAL; networks trained on the annotator distribution of each image, except on CIFAR-10H (one label per image). --: not evaluated.}
\label{tab:pareto_gap}
\begin{tabular}{lcccccccccc}
\toprule
 & & \multicolumn{8}{c}{DCIC} & \\
\cmidrule(lr){3-10}
Method & \makecell{CIFAR-\\10H} & Benthic & \makecell{Mice-\\Bone} & Pig & \makecell{Plank-\\ton} & \makecell{Quality-\\MRI} & \makecell{Tree-\\versity\#1} & \makecell{Tree-\\versity\#6} & Turkey & \makecell{APPA-\\REAL} \\
\midrule
BatchEnsemble & 0.27 (6) & 0.22 (4) & 0.13 (7) & 0.40 (7) & 0.14 (3) & 0.29 (44) & 0.14 (4) & 0.19 (22) & 0.11 (5) & 0.26 (3) \\
DDU & 0.26 (17) & 0.23 (5) & 0.16 (7) & 0.40 (7) & 0.14 (8) & 0.32 (38) & 0.16 (6) & 0.28 (16) & 0.18 (5) & 0.27 (4) \\
Deep Ensemble & 0.26 (3) & 0.21 (4) & 0.13 (4) & 0.36 (4) & 0.13 (3) & 0.30 (29) & 0.14 (3) & 0.23 (13) & 0.12 (5) & 0.27 (2) \\
DropConnect ($p{=}0.1$) & -- & 0.20 (9) & 0.16 (7) & 0.38 (3) & 0.13 (4) & 0.29 (29) & 0.15 (3) & 0.30 (5) & 0.12 (7) & -- \\
DropConnect ($p{=}0.5$) & 0.31 (9) & 0.24 (6) & 0.17 (8) & 0.39 (4) & 0.16 (4) & 0.31 (44) & 0.15 (5) & 0.27 (14) & 0.16 (7) & 0.26 (3) \\
Evidential & 0.06 (2) & 0.24 (9) & 0.14 (9) & 0.39 (5) & 0.11 (15) & 0.30 (44) & 0.12 (9) & 0.25 (18) & 0.13 (7) & 0.28 (2) \\
Laplace (last layer) & 0.28 (3) & 0.21 (3) & 0.17 (7) & 0.39 (4) & 0.14 (5) & 0.25 (38) & 0.17 (3) & 0.32 (4) & 0.14 (6) & 0.28 (5) \\
MC-Dropout ($p{=}0.1$) & 0.29 (6) & 0.22 (4) & 0.18 (7) & 0.39 (3) & 0.15 (5) & 0.30 (34) & 0.17 (3) & 0.32 (7) & 0.14 (7) & 0.27 (2) \\
MC-Dropout ($p{=}0.5$) & 0.29 (1) & 0.25 (14) & 0.18 (7) & 0.41 (5) & 0.15 (6) & 0.25 (44) & 0.17 (6) & 0.26 (17) & 0.18 (7) & 0.26 (3) \\
MFVI-BNN & 0.20 (9) & 0.20 (5) & 0.12 (6) & 0.40 (5) & 0.11 (6) & 0.34 (32) & 0.14 (4) & 0.22 (18) & 0.15 (7) & 0.26 (3) \\
NatPN & 0.28 (19) & 0.24 (3) & 0.14 (3) & 0.41 (3) & 0.14 (3) & 0.34 (38) & 0.16 (2) & 0.31 (2) & 0.13 (4) & 0.28 (4) \\
PostNet & 0.29 (2) & 0.25 (1) & 0.16 (5) & 0.40 (2) & 0.16 (1) & 0.24 (28) & 0.18 (1) & 0.33 (1) & 0.17 (6) & 0.16 (3) \\
SNGP & 0.26 (10) & 0.25 (4) & 0.14 (7) & 0.42 (4) & 0.14 (2) & 0.36 (25) & 0.18 (2) & 0.33 (2) & 0.17 (4) & 0.26 (7) \\
SubEnsemble & 0.28 (9) & 0.23 (4) & 0.19 (4) & 0.39 (3) & 0.14 (5) & 0.22 (52) & 0.16 (3) & 0.27 (9) & 0.16 (7) & 0.27 (3) \\
SWAG & 0.13 (1) & 0.25 (5) & 0.20 (6) & 0.41 (9) & 0.17 (5) & 0.29 (39) & 0.20 (3) & 0.31 (8) & 0.18 (7) & 0.27 (2) \\
\midrule
Median & 0.27 (6) & 0.23 (4) & 0.16 (7) & 0.40 (4) & 0.14 (5) & 0.30 (38) & 0.16 (3) & 0.28 (9) & 0.15 (7) & 0.27 (3) \\
\bottomrule
\end{tabular}}
\end{table}

\begin{table}[t]
{
\centering
\footnotesize
\setlength{\tabcolsep}{3pt}
\caption{\textbf{The family of selective tasks on CIFAR-10H under log loss.} Fourteen methods on ResNet-18 networks trained on CIFAR-10, means over three seeds, evaluated against the CIFAR-10H annotator distribution of the $10{,}000$ test images (Brier loss in \cref{tab:family_brier}; the endpoint areas under 0/1 loss in \cref{tab:notion_full_01}). The endpoint columns are the normalized areas under the risk--coverage curve ($n$-$\AuRC$, selective prediction) and under the regret--coverage curve ($n$-$\AuReC$, the regret task), each with the method's epistemic (E) or aleatoric (A) estimate of the log-loss decomposition as the score; $0$ is the optimal score and $1$ a random ordering. The Pareto-gap columns give how closely the method approaches the optimal surface $\mathcal{S}^{*}$ of the family (\cref{sec:family_measure}) with the best weighted sum of its two estimates (Pair) and with its epistemic (E), aleatoric (A) or total (T) estimate alone, normalized as in \cref{tab:pareto_gap} ($0$: the optimal surface, $1$: a random ordering), and Smaller gap is how much smaller, in \%, the gap of the pair is than that of the method's best single estimate. $r_s(\Ahat, \Ehat)$ is the Spearman rank correlation between the two estimates of the same decomposition over the test images, mean over the seeds. Rows are sorted by $n$-$\AuReC$ with the epistemic estimate.}
\label{tab:family}
\begin{tabular}{lcccccccccc}
\toprule
 & \multicolumn{2}{c}{Selective prediction} & \multicolumn{2}{c}{Regret} & \multicolumn{4}{c}{Pareto gap $\downarrow$} & & \\
 & \multicolumn{2}{c}{$n$-$\AuRC$ $\downarrow$} & \multicolumn{2}{c}{$n$-$\AuReC$ $\downarrow$} & \multicolumn{4}{c}{} & Smaller gap (\%) & $r_s(\Ahat, \Ehat)$ \\
\cmidrule(lr){2-3}\cmidrule(lr){4-5}\cmidrule(lr){6-9}
Method & E & A & E & A & Pair & E & A & T & & \\
\midrule
Evidential & 0.058 & 0.059 & 0.153 & 0.152 & 0.064 & 0.066 & 0.066 & 0.066 & $2.0$ & $+0.999$ \\
SWAG & 0.288 & 0.256 & 0.246 & 0.217 & 0.132 & 0.150 & 0.133 & 0.134 & $1.3$ & $+0.946$ \\
SNGP & 0.481 & 0.440 & 0.370 & 0.338 & 0.255 & 0.325 & 0.283 & 0.283 & $9.9$ & $+0.640$ \\
MFVI-BNN & 0.380 & 0.364 & 0.375 & 0.362 & 0.201 & 0.241 & 0.219 & 0.219 & $8.6$ & $+0.962$ \\
DDU & 0.479 & 0.552 & 0.438 & 0.520 & 0.264 & 0.318 & 0.333 & 0.318 & $17.0$ & $+0.643$ \\
BatchEnsemble & 0.474 & 0.456 & 0.457 & 0.442 & 0.265 & 0.306 & 0.282 & 0.284 & $6.0$ & $+0.954$ \\
Laplace (last layer) & 0.506 & 0.495 & 0.476 & 0.466 & 0.282 & 0.291 & 0.309 & 0.309 & $3.2$ & $+0.923$ \\
Deep Ensemble & 0.495 & 0.462 & 0.478 & 0.450 & 0.262 & 0.272 & 0.269 & 0.269 & $2.6$ & $+0.772$ \\
SubEnsemble & 0.540 & 0.524 & 0.503 & 0.491 & 0.279 & 0.326 & 0.307 & 0.307 & $9.0$ & $+0.912$ \\
MC-Dropout ($p{=}0.1$) & 0.573 & 0.532 & 0.540 & 0.501 & 0.285 & 0.332 & 0.305 & 0.305 & $6.4$ & $+0.805$ \\
DropConnect ($p{=}0.5$) & 0.596 & 0.528 & 0.552 & 0.483 & 0.307 & 0.423 & 0.339 & 0.338 & $9.1$ & $+0.850$ \\
MC-Dropout ($p{=}0.5$) & 0.591 & 0.551 & 0.557 & 0.517 & 0.288 & 0.306 & 0.292 & 0.298 & $1.4$ & $+0.892$ \\
PostNet & 0.602 & 0.579 & 0.576 & 0.554 & 0.294 & 0.319 & 0.300 & 0.300 & $2.0$ & $+0.980$ \\
NatPN & 1.716 & 0.602 & 1.858 & 0.526 & 0.280 & 0.959 & 0.372 & 0.348 & $19.5$ & $-0.625$ \\
\bottomrule
\end{tabular}
}
\end{table}

\subsection{When the true risk and the regret differ}
{
\Cref{sec:experiments} states that where the true risk and the regret order the inputs similarly, a single true score is almost optimal for every trade-off. The reason is that every optimal score combines the two with non-negative weights, since $\Astar = \Tstar - \Estar$ gives
\begin{equation*}
    (1-\lambda)\Astar + \lambda\Estar = (1-\lambda)\Tstar + (2\lambda - 1)\Estar, \qquad 1-\lambda \geq 0, \quad 2\lambda - 1 \geq 0 .\end{equation*}
If the true risk and the regret order the inputs in the same way, every such combination gives that ordering, and one score is optimal for both tasks and for every trade-off between them. The gap is reported under the log loss (absolute loss on APPA-REAL) and, on CIFAR-10H, under the Brier loss (\cref{tab:family_brier}). It is not reported under the 0/1 loss, under which the regret is zero on most inputs, so the regret task leaves them tied. \Cref{tab:true_corr} gives the correlation of the Bayes risk and the regret quoted in \cref{sec:experiments}; under the log and the Brier loss, it is higher for networks trained on one label per image on every DCIC dataset.
}
\begin{table}[htbp]
\centering
{
\footnotesize
\caption{\textbf{Correlation of the true components.} Spearman correlation of the Bayes risk $\Astar$ and the regret $\Estar$ over the test inputs under the log, the Brier and the 0/1 loss, for networks trained on one label per image and on the annotator distribution of each image; mean over the seeds, median over the methods with an epistemic estimate. \Cref{tab:summary} gives the correlation of the true risk $\Tstar$ and the regret. APPA-REAL: the absolute and the squared loss, as labeled. Dashes: not evaluated.}
\label{tab:true_corr}
\begin{tabular}{lcccccc}
\toprule
 & \multicolumn{3}{c}{One label per image} & \multicolumn{3}{c}{Annotator distribution} \\
\cmidrule(lr){2-4}\cmidrule(l){5-7}
Dataset & log & Brier & 0/1 & log & Brier & 0/1 \\
\midrule
CIFAR-10H & 0.87 & 0.86 & 0.22 & -- & -- & -- \\
\addlinespace[2pt]
\multicolumn{7}{@{}l}{DCIC} \\
\quad Benthic & 0.53 & 0.43 & 0.13 & 0.34 & 0.31 & 0.19 \\
\quad MiceBone & 0.65 & 0.67 & 0.29 & 0.37 & 0.50 & 0.30 \\
\quad Pig & 0.12 & 0.11 & 0.03 & $-$0.16 & $-$0.10 & 0.09 \\
\quad Plankton & 0.68 & 0.64 & 0.30 & 0.62 & 0.62 & 0.32 \\
\quad QualityMRI & 0.28 & 0.69 & 0.37 & $-$0.14 & 0.02 & 0.35 \\
\quad Treeversity\#1 & 0.59 & 0.54 & 0.28 & 0.51 & 0.51 & 0.31 \\
\quad Treeversity\#6 & 0.44 & 0.58 & 0.20 & 0.15 & 0.22 & 0.27 \\
\quad Turkey & 0.56 & 0.56 & 0.27 & 0.40 & 0.47 & 0.29 \\
\midrule
 & absolute & squared & & absolute & squared & \\
\midrule
APPA-REAL & 0.12 & 0.26 &  & 0.10 & 0.24 &  \\
\bottomrule
\end{tabular}
}
\end{table}

\paragraph{An example.}
{
For instance, consider two test images on which the two orderings are opposite. On the first, the true conditional splits evenly between cat and dog, and the classifier predicts the same split; under the log loss, $\Astar = \log 2 \approx 0.69$ nats and $\Estar = 0$. On the second, the true conditional puts all its mass on cat, and the classifier predicts cat with probability $0.55$; then $\Astar = 0$ and $\Estar = -\log 0.55 \approx 0.60$ nats (\cref{app:bayesian}). Selective prediction, which orders by $\Tstar = \Astar + \Estar$, rejects the first image first, while the regret task rejects the second.
}

\clearpage

\section{Implementation Details}
\label{app:implementation}

{
This appendix gives the methods, the datasets and the training of the networks, and the protocols of the tasks in \cref{sec:experiments}. \Cref{tab:scale} counts the networks we evaluate and the models we retrain in active learning. The code of all experiments is available at \url{https://github.com/paplhjak/neurips2026-evaluating-epistemic-uncertainty}.
}
\begin{table}[htbp]
\centering
{
\footnotesize
\caption{\textbf{Scale of the evaluation.} Trained networks: the networks behind the evaluations; a Deep Ensemble counts its five members, MC-Dropout and DropConnect with $p{=}0.1$ and the last-layer Laplace approximation on DCIC share one network, and on APPA-REAL a network is a head on the frozen backbone. Evaluations: one per method, dataset, seed and kind of training labels, in OOD detection, selective prediction and the regret task. Trajectories: active-learning runs, one per arm, dataset and seed. Retrained models: nine per trajectory on CIFAR-10, eight on DCIC and ten on APPA-REAL, a Deep Ensemble again counting its five members. DCIC: the eight real datasets, including total-estimate arms that no table shows.}
\label{tab:scale}
\begin{tabular}{llrrrr}
\toprule
Datasets & Training labels & \makecell{Trained\\networks} & Evaluations & Trajectories & \makecell{Retrained\\models} \\
\midrule
APPA-REAL & annotator distribution & $90$ & $70$ & $420$ & $5{,}400$ \\
APPA-REAL & one label per image & $90$ & $70$ & $210$ & $2{,}700$ \\
CIFAR-10H & one label per image & $54$ & $42$ & $369$ & $4{,}293$ \\
DCIC & annotator distribution & $408$ & $360$ & $495$ & $5{,}016$ \\
DCIC & one label per image & $408$ & $360$ & $468$ & $4{,}512$ \\
\midrule
Total &  & $1{,}050$ & $902$ & $1{,}962$ & $21{,}921$ \\
\bottomrule
\end{tabular}}
\end{table}

\subsection{Methods}
\label{app:panel}

{
The methods and their uncertainty estimates use the implementations of \texttt{probly} \citep{probly}. The fourteen methods are MC-Dropout \citep{gal2016dropout} with $p{=}0.1$ and with $p{=}0.5$, DropConnect \citep{wan2013regularization} with $p{=}0.5$, a last-layer Laplace approximation \citep{daxberger2021laplace}, SWAG \citep{maddox2019simple}, a mean-field variational Bayesian network (MFVI-BNN) \citep{blundell2015weight}, Deep Ensembles \citep{lakshminarayanan2017simple}, BatchEnsemble \citep{wen2020batchensemble}, SubEnsemble \citep{valdenegro2019deep}, SNGP \citep{liu2020simple}, DDU \citep{mukhoti2023deep}, Evidential networks \citep{sensoy2018evidential}, PostNet \citep{charpentier2020posterior} and NatPN \citep{charpentier2022natural}. On DCIC, DropConnect with $p{=}0.1$ is added.

\paragraph{Settings.} MC-Dropout and DropConnect with $p{=}0.1$ insert dropout before the final linear layer of a trained network, to which the Laplace approximation is also applied; on APPA-REAL, both MC-Dropout variants are trained with dropout. Deep Ensembles have five members; on CIFAR-10 these are the first five of the benchmark's ten. Evidential networks minimize the Dirichlet expected cross-entropy without a KL term on DCIC and APPA-REAL, since every positive KL weight collapses the APPA-REAL predictions to the uniform distribution. SWAG continues the trained network for $40$ epochs of SGD at a constant learning rate, collects one snapshot per epoch and draws $30$ samples. Settings chosen on validation data are the learning rate of SWAG ($0.01$ on DCIC, among $0.002$, $0.01$ and $0.05$ on APPA-REAL), the prior of MFVI-BNN on DCIC (standard deviation $0.05$ and initial posterior standard deviation $0.001$ around the pre-trained weights, with the full KL term), the multi-layer heads of SubEnsemble on DCIC, and the prior precision of the Laplace approximation on APPA-REAL; on DCIC, the Laplace prior precision maximizes the marginal likelihood. The estimates of the sampling methods use $20$ samples, those of the ensembles their members. The Dirichlet methods (Evidential networks, PostNet, NatPN) use $K/\alpha_0$, with $\alpha_0$ the concentration, as the epistemic estimate, except under the Brier loss outside active learning, where they use the total variance of their Dirichlet distribution. On APPA-REAL, the estimates under the absolute loss are the sample versions of the components of \cref{app:bayesian}.
}

\subsection{Datasets and training}

{
\paragraph{CIFAR-10 and CIFAR-10H.}
The networks are trained on one label per image, on $45{,}000$ images of the $50{,}000$-image CIFAR-10 training split with the other $5{,}000$ held out for validation, and evaluated on the $10{,}000$ test images, for which CIFAR-10H provides about fifty human labels each \citep{peterson2019human}. Every method uses a ResNet-18 trained from scratch for $200$ epochs with the recipe of the \texttt{probly} CIFAR-10 benchmark: SGD with a learning rate of $0.1$, momentum $0.9$, weight decay $5\times 10^{-4}$, batches of $128$ and a cosine schedule. We use the benchmark's checkpoints for three seeds, except for SWAG, which we continue from its base networks with a learning rate of $0.05$.

\paragraph{DCIC.}
We use the eight real datasets of the DCIC suite \citep{schmarje2022benchmark}, without its CIFAR-10H duplicate; its synthetic dataset appears only as an OOD set and in \cref{tab:al_true}. Each dataset has five folds. Each of three seeds selects one fold as the test set, and the other four folds are the training data, of which $10\%$ serve for early stopping. A ResNet-18 initialized with ImageNet-1K weights is fine-tuned with a new classification head for up to $100$ epochs, with AdamW at a learning rate of $3\times 10^{-4}$, weight decay $10^{-4}$, batches of $32$, a cosine schedule and early stopping with patience $20$. Deep Ensembles, Evidential networks and DDU use SGD with a learning rate of $3\times 10^{-3}$, momentum $0.9$ and weight decay $10^{-4}$ instead.

\paragraph{Training labels on DCIC.}
Every DCIC dataset is evaluated with two kinds of training labels: the annotator distribution of each training image, the fraction of annotators per class, and the majority label, the most frequent annotator label with ties going to the first class. The two kinds are never mixed within a dataset. The losses of all methods are generalized so that they reduce to the original loss for a one-hot target. The Laplace fit and DDU's class-conditional density use the majority label in both cases.

\paragraph{APPA-REAL.}
APPA-REAL \citep{agustsson2017appareal} is treated as a probability estimation task over the integer ages $0$ to $100$ ($K=101$), with the annotator distribution of each image as its conditional distribution. A frozen ResNet-101, trained for age estimation on CSFD-1.6M \citep{csfd}, gives $2048$-dimensional features of the face crops. On these features we train a head with one hidden layer of $128$ units and dropout $0.1$. The heads are trained on the $4{,}113$ training images, early-stopped on the $1{,}500$ validation images and evaluated on the $1{,}978$ test images, with five seeds and one recipe for every method: AdamW with a learning rate of $10^{-3}$ and weight decay $10^{-4}$ (none for MFVI-BNN), batches of $128$, gradient clipping at $1.0$, a cosine schedule over at most $400$ epochs and early stopping with patience $30$ on the validation log-loss. The targets are the annotator distributions or, with one label per image, the one-hot lower median of each image's votes (\cref{app:bayesian}). The Laplace fit and DDU's class-conditional density use the most frequent age among the votes, or the lower median with one label per image.
}

\subsection{Task protocols}
\label{app:protocols}

{
\paragraph{Regret and selective prediction.} The conditional distribution of every test image is the empirical distribution of its labels. For each loss, the realized components (\cref{eq:realized}) are computed per image. With $N$ test images, the areas of \cref{eq:aurc} are the means of the selective risk and the selective regret over the coverages $k/N$, $k = 1, \dots, N$. Inputs with tied scores are accepted in random order, and we report the expected curve.

\paragraph{OOD detection.} On CIFAR-10, the near-OOD set is CIFAR-100 and the far-OOD sets are those of OpenOOD v1.5 \citep{zhang2023openood}: MNIST, SVHN, Textures and Places365. Each comparison uses equal-sized random subsets of the two test sets, and the scores use the sample counts of the benchmark's OOD configuration. On DCIC, the far-OOD sets of a dataset are the test images of the other DCIC datasets, the synthetic one included; Treeversity\#1 and Treeversity\#6 share their images and are not OOD sets of each other. On APPA-REAL, the near-OOD set consists of faces that RetinaFace \citep{deng2020retinaface} detects in LAION-Face images \citep{laionface} and that do not show a photographed human face, so that no apparent age can be estimated. We annotated the crops manually, separated those without a photographed face into drawings, dolls or statues, and animals with the vision-language model Qwen3-VL-32B \citep{qwen3vl}, and checked the result manually. The near-OOD set is $2{,}000$ crops of these three groups (examples in \cref{fig:appa_ood}); the crops we annotated as a single photographed face are a control (\cref{tab:appa_ceiling}). The far-OOD sets are $2{,}000$ Places365 and $2{,}000$ ImageNet validation images. On every dataset, the AUROC is averaged over the sets of each group.
\begin{figure}[htbp]
\centering
\includegraphics[width=\textwidth]{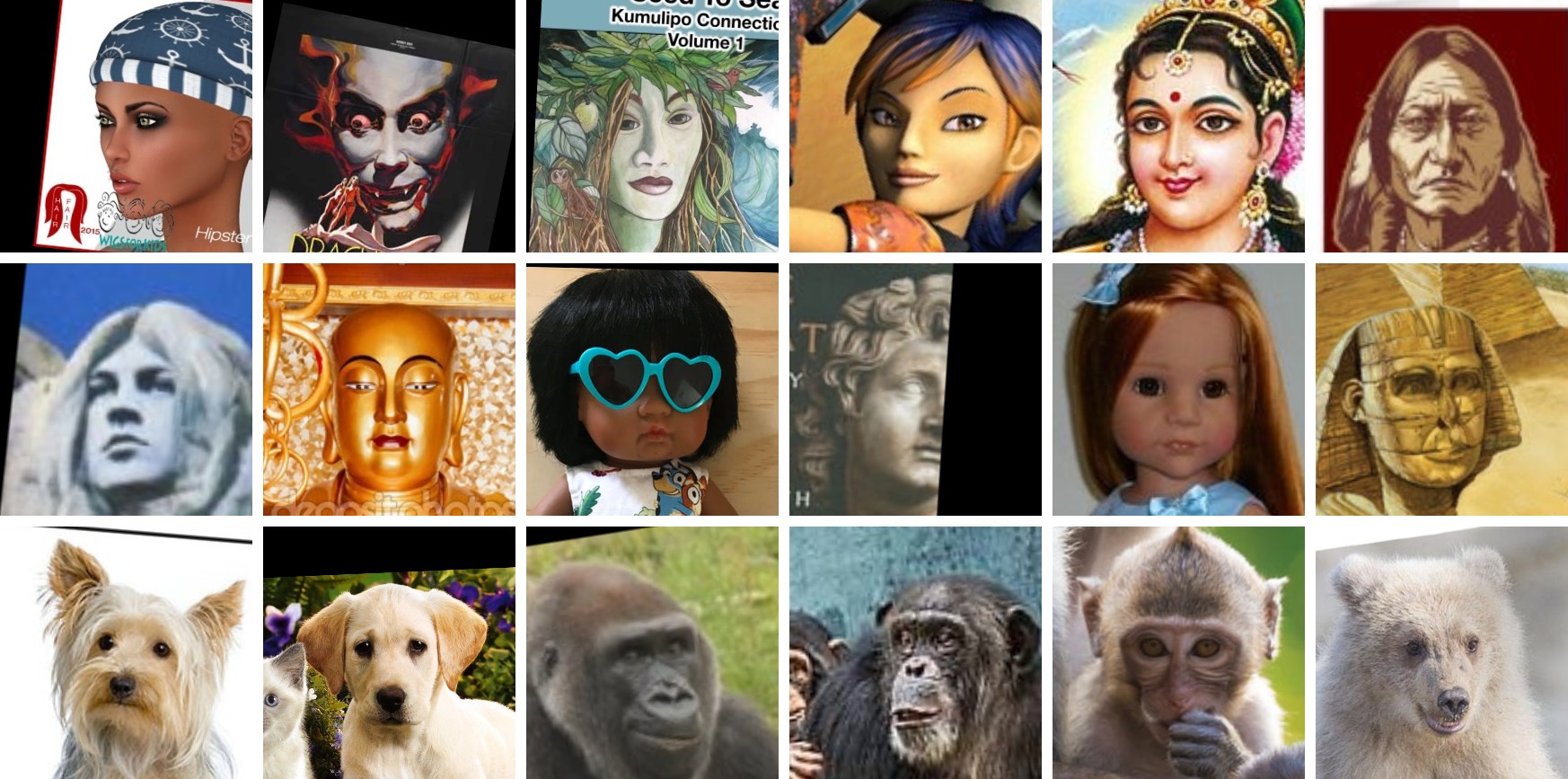}
\caption{\textbf{Examples of the near-OOD set of APPA-REAL.} Faces that RetinaFace \citep{deng2020retinaface} detects in LAION-Face images \citep{laionface} and that do not show a photographed human face, so that no apparent age can be estimated. Each row shows one group: drawings, dolls and statues, and animals. The examples are drawn at random from each group and screened for this figure; the groups come from a vision-language model with a manual check and may contain a few errors.}
\label{fig:appa_ood}
\end{figure}

\paragraph{Separability of the APPA-REAL OOD sets.} \Cref{tab:appa_ceiling} shows how well the backbone's features separate the sets. A logistic regression separates APPA-REAL from the near-OOD set at $0.98$ and from the far sets at $0.995$ and $0.997$, so the near-OOD AUROCs of $0.480$ to $0.675$ that the methods reach are not limited by the features. The control rows show that part of this separation is the style of the source.
\begin{table}[htbp]
\centering
{
\footnotesize
\caption{\textbf{Separability of the APPA-REAL OOD sets.} How well the features of the frozen backbone, on which every method's head is trained, tell two sets of images apart: the AUROC ($0.5$: inseparable, $1$: fully separable) of a logistic regression and of a one-nearest-neighbor classifier on the features, with five-fold cross-validation on balanced sets of $1{,}978$ to $2{,}000$ images per set, and of the feature norm alone, a smaller norm scoring as the second set. A high value means that the features alone allow OOD detection. The control rows use LAION crops that we annotated as a single photographed face, which have an estimable age; they show how much of the separation comes from the source of the images.}
\label{tab:appa_ceiling}
\begin{tabular}{llccc}
\toprule
First set & Second set & Logistic regression & Nearest neighbor & Feature norm \\
\midrule
APPA-REAL & LAION, no age (near) & $0.983$ & $0.857$ & $0.538$ \\
APPA-REAL & ImageNet (far) & $0.995$ & $0.972$ & $0.903$ \\
APPA-REAL & Places365 (far) & $0.997$ & $0.981$ & $0.917$ \\
APPA-REAL & LAION, single face (control) & $0.936$ & $0.663$ & $0.361$ \\
LAION, single face & LAION, no age (control) & $0.956$ & $0.802$ & $0.671$ \\
\bottomrule
\end{tabular}
}
\end{table}

\paragraph{Active learning on CIFAR-10.} The pool is the CIFAR-10 training set of $50{,}000$ images. Each arm starts from the same $1{,}000$ labeled images, acquires the $1{,}000$ pool images with the highest score in each of nine rounds and retrains the method from scratch after every round. The score is the epistemic, aleatoric or total estimate of the method under the log, 0/1 or Brier loss; SubEnsemble has only its epistemic arm under the log loss. Random acquisition and the margin, entropy and least-confidence heuristics are included. The $\NAUC$ of an arm is the mean test accuracy of its retrained networks over the nine rounds. The regret of every retrained model is computed against CIFAR-10H with the method's epistemic estimate as the score; for the arms on the network without an uncertainty method (random acquisition, the heuristics, MC-Dropout and DropConnect with $p{=}0.1$ and the Laplace approximation), it is the regret of the MC-Dropout ($p{=}0.1$) predictive distribution.

\paragraph{Active learning on DCIC.} The same protocol runs on the seven real datasets other than QualityMRI, whose about $250$ non-test images are fewer than the $500$ initial labels, and on the synthetic dataset (\cref{tab:al_true}). The pool is every non-test image of the seed. Each arm starts from $500$ labeled images and acquires $500$ in each of eight rounds; on Benthic, whose non-test images are fewer than the $4{,}500$ labels of the protocol, it starts from $250$ and acquires $250$ per round. An acquired image reveals its training label, its annotator distribution or its majority label. The main experiments run on all seven datasets. They acquire at random and by the true regret, the true aleatoric component and the true risk of the current model under log loss, with both kinds of training labels, and, with the annotator distributions, by the epistemic estimate of each method under log loss, each method retraining its own network; they give the active-learning results of \cref{tab:summary}. Supporting experiments run on four datasets each: with the annotator distributions, acquisition by the total estimates of MC-Dropout and DropConnect with $p{=}0.1$ and of the Laplace approximation and by the three heuristics on MiceBone, Pig and the two Treeversity datasets; with one label per image, acquisition by the epistemic and the total estimates of the methods and by the heuristics on MiceBone, Pig, Plankton and Turkey. They compare the true regret with other scores of the same network (\cref{tab:al_true}) and give the one-label Regret--AL and OOD--AL columns of \cref{tab:summary_full}.

\paragraph{Active learning on APPA-REAL.} The pool is the training split. Every arm of a seed starts from the same $100$ labeled images, acquires the $100$ pool images with the highest score in each of ten rounds and retrains the head after every round. An acquired image reveals its annotator distribution, or the lower median of its votes with one label per image. The objective is the test risk under the absolute loss, averaged over the ten rounds (\cref{tab:al_appareal}). For each of the fourteen methods, the arms are random acquisition, the method's epistemic and total estimates, and the true regret, the true aleatoric component and the true risk of the current model; with one label per image, random acquisition, the epistemic estimate and the true regret.
}

\clearpage
\section{Detailed Results on CIFAR-10H}
\label{app:cifar_results}
\suppressfloats[t]

{
This appendix gives the detailed results on CIFAR-10H that \cref{fig:rank_grid} and the CIFAR-10H row of \cref{tab:summary} summarize, for the fourteen methods of the CIFAR-10 panel (\cref{app:panel}). \Cref{tab:panel_views} gives the value and the rank of every method in every task, with the epistemic estimate as the score, and \cref{tab:view_agreement} the agreement between the rankings of the tasks.
}
\begin{table}[htbp]
\centering
{
\scriptsize
\setlength{\tabcolsep}{3pt}
\caption{\textbf{Values of every task on the CIFAR-10 panel.} Fourteen methods, means over three seeds, the epistemic estimate as the score; rank in parentheses, best value per column in bold; rows in the order of \cref{fig:rank_grid}. OOD detection: $\AuROC$ against the near and the far sets. AL: $\NAUC$, acquiring by the epistemic estimate of the log-loss decomposition. Selective prediction and regret: $n$-$\AuRC$ and $n$-$\AuReC$ under the loss named, against the CIFAR-10H annotator distribution. References: the raw area $\AuReC_{\log}$ and the accuracy against the CIFAR-10H majority label.}
\label{tab:panel_views}
\resizebox{\textwidth}{!}{\begin{tabular}{lccccccccccc}
\toprule
 & \multicolumn{2}{c}{OOD $\AuROC$ $\uparrow$} & AL & \multicolumn{3}{c}{$n$-$\AuRC$ $\downarrow$} & \multicolumn{3}{c}{$n$-$\AuReC$ $\downarrow$} & \multicolumn{2}{c}{references} \\
\cmidrule(lr){2-3} \cmidrule(lr){5-7} \cmidrule(lr){8-10} \cmidrule(lr){11-12}
Method & near & far & $\NAUC$ $\uparrow$ & log & Brier & 0/1 & log & Brier & 0/1 & $\AuReC_{\log}$ $\downarrow$ & Acc.\ $\uparrow$ \\
\midrule
Evidential & $0.861$\,{\tiny(5)} & $0.891$\,{\tiny(6)} & $0.807$\,{\tiny(3)} & $\mathbf{0.058}$\,{\tiny(1)} & $\mathbf{0.046}$\,{\tiny(1)} & $\mathbf{0.188}$\,{\tiny(1)} & $\mathbf{0.153}$\,{\tiny(1)} & $0.264$\,{\tiny(13)} & $0.093$\,{\tiny(2)} & $0.8143$\,{\tiny(14)} & $0.903$\,{\tiny(12)} \\
BatchEnsemble & $0.875$\,{\tiny(3)} & $0.905$\,{\tiny(4)} & $0.815$\,{\tiny(2)} & $0.474$\,{\tiny(4)} & $0.284$\,{\tiny(4)} & $0.550$\,{\tiny(6)} & $0.457$\,{\tiny(6)} & $0.147$\,{\tiny(2)} & $0.353$\,{\tiny(7)} & $0.1633$\,{\tiny(5)} & $0.936$\,{\tiny(9)} \\
SWAG & $0.822$\,{\tiny(11)} & $0.836$\,{\tiny(12)} & $0.798$\,{\tiny(4)} & $0.288$\,{\tiny(2)} & $0.214$\,{\tiny(2)} & $0.291$\,{\tiny(2)} & $0.246$\,{\tiny(2)} & $\mathbf{0.145}$\,{\tiny(1)} & $0.173$\,{\tiny(4)} & $\mathbf{0.1120}$\,{\tiny(1)} & $0.900$\,{\tiny(13)} \\
Deep Ensemble & $\mathbf{0.902}$\,{\tiny(1)} & $0.929$\,{\tiny(2)} & $\mathbf{0.831}$\,{\tiny(1)} & $0.495$\,{\tiny(7)} & $0.343$\,{\tiny(5)} & $0.616$\,{\tiny(7)} & $0.478$\,{\tiny(8)} & $0.163$\,{\tiny(4)} & $0.347$\,{\tiny(6)} & $0.1357$\,{\tiny(3)} & $\mathbf{0.959}$\,{\tiny(1)} \\
DDU & $0.899$\,{\tiny(2)} & $\mathbf{0.935}$\,{\tiny(1)} & $0.673$\,{\tiny(13)} & $0.479$\,{\tiny(5)} & $0.348$\,{\tiny(6)} & $0.311$\,{\tiny(4)} & $0.438$\,{\tiny(5)} & $0.168$\,{\tiny(5)} & $\mathbf{0.089}$\,{\tiny(1)} & $0.1555$\,{\tiny(4)} & $0.952$\,{\tiny(2)} \\
MFVI-BNN & $0.831$\,{\tiny(9)} & $0.816$\,{\tiny(13)} & $0.623$\,{\tiny(14)} & $0.380$\,{\tiny(3)} & $0.221$\,{\tiny(3)} & $0.293$\,{\tiny(3)} & $0.375$\,{\tiny(4)} & $0.156$\,{\tiny(3)} & $0.184$\,{\tiny(5)} & $0.1746$\,{\tiny(6)} & $0.891$\,{\tiny(14)} \\
Laplace (last layer) & $0.872$\,{\tiny(4)} & $0.907$\,{\tiny(3)} & $0.775$\,{\tiny(8)} & $0.506$\,{\tiny(8)} & $0.352$\,{\tiny(7)} & $0.981$\,{\tiny(13)} & $0.476$\,{\tiny(7)} & $0.186$\,{\tiny(6)} & $0.969$\,{\tiny(13)} & $0.1910$\,{\tiny(7)} & $0.948$\,{\tiny(4)} \\
MC-Dropout ($p{=}0.5$) & $0.837$\,{\tiny(7)} & $0.893$\,{\tiny(5)} & $0.777$\,{\tiny(6)} & $0.591$\,{\tiny(11)} & $0.397$\,{\tiny(9)} & $0.820$\,{\tiny(8)} & $0.557$\,{\tiny(12)} & $0.225$\,{\tiny(9)} & $0.726$\,{\tiny(8)} & $0.2275$\,{\tiny(11)} & $0.934$\,{\tiny(11)} \\
MC-Dropout ($p{=}0.1$) & $0.834$\,{\tiny(8)} & $0.874$\,{\tiny(8)} & $0.786$\,{\tiny(5)} & $0.573$\,{\tiny(10)} & $0.400$\,{\tiny(10)} & $0.922$\,{\tiny(10)} & $0.540$\,{\tiny(10)} & $0.214$\,{\tiny(8)} & $0.868$\,{\tiny(10)} & $0.1969$\,{\tiny(9)} & $0.949$\,{\tiny(3)} \\
PostNet & $0.837$\,{\tiny(6)} & $0.888$\,{\tiny(7)} & $0.708$\,{\tiny(12)} & $0.602$\,{\tiny(13)} & $0.406$\,{\tiny(11)} & $0.370$\,{\tiny(5)} & $0.576$\,{\tiny(13)} & $0.241$\,{\tiny(12)} & $0.155$\,{\tiny(3)} & $0.3188$\,{\tiny(12)} & $0.935$\,{\tiny(10)} \\
SNGP & $0.816$\,{\tiny(12)} & $0.852$\,{\tiny(11)} & $0.731$\,{\tiny(10)} & $0.481$\,{\tiny(6)} & $0.435$\,{\tiny(13)} & $0.931$\,{\tiny(11)} & $0.370$\,{\tiny(3)} & $0.238$\,{\tiny(11)} & $0.888$\,{\tiny(11)} & $0.1215$\,{\tiny(2)} & $0.946$\,{\tiny(7)} \\
DropConnect ($p{=}0.5$) & $0.827$\,{\tiny(10)} & $0.872$\,{\tiny(9)} & $0.777$\,{\tiny(7)} & $0.596$\,{\tiny(12)} & $0.407$\,{\tiny(12)} & $0.832$\,{\tiny(9)} & $0.552$\,{\tiny(11)} & $0.229$\,{\tiny(10)} & $0.744$\,{\tiny(9)} & $0.2107$\,{\tiny(10)} & $0.936$\,{\tiny(8)} \\
SubEnsemble & $0.815$\,{\tiny(13)} & $0.863$\,{\tiny(10)} & $0.772$\,{\tiny(9)} & $0.540$\,{\tiny(9)} & $0.373$\,{\tiny(8)} & $0.971$\,{\tiny(12)} & $0.503$\,{\tiny(9)} & $0.186$\,{\tiny(7)} & $0.950$\,{\tiny(12)} & $0.1953$\,{\tiny(8)} & $0.948$\,{\tiny(5)} \\
NatPN & $0.256$\,{\tiny(14)} & $0.219$\,{\tiny(14)} & $0.725$\,{\tiny(11)} & $1.716$\,{\tiny(14)} & $1.560$\,{\tiny(14)} & $2.179$\,{\tiny(14)} & $1.858$\,{\tiny(14)} & $1.820$\,{\tiny(14)} & $2.862$\,{\tiny(14)} & $0.5137$\,{\tiny(13)} & $0.947$\,{\tiny(6)} \\
\bottomrule
\end{tabular}}
}
\end{table}

\begin{table}[t]
\centering
{
\scriptsize
\setlength{\tabcolsep}{4pt}
\caption{\textbf{Rank agreement between the tasks on CIFAR-10H.} Spearman correlation between the method rankings of two tasks over the fourteen methods of \cref{tab:panel_views} ($+1$: identical rankings); the values of \cref{fig:rank_grid}(b). Diagonal (italic): the agreement of two seeds of the same task, mean over the three pairs of seeds. Last two rows: the raw log-loss regret area and the test accuracy, as references. \Crefrange{tab:agreement_benthic}{tab:agreement_appa_real} give the other datasets.}
\label{tab:view_agreement}
\begin{tabular}{lccccccccccc}
\toprule
 & \rotatebox{90}{OOD detection, near} & \rotatebox{90}{OOD detection, far} & \rotatebox{90}{Active learning} & \rotatebox{90}{Selective prediction, log} & \rotatebox{90}{Selective prediction, Brier} & \rotatebox{90}{Selective prediction, 0/1} & \rotatebox{90}{Regret, log} & \rotatebox{90}{Regret, Brier} & \rotatebox{90}{Regret, 0/1} & \rotatebox{90}{Regret, log, raw area} & \rotatebox{90}{Test accuracy} \\
\midrule
OOD detection, near & \textit{0.91} &  &  &  &  &  &  &  &  &  &  \\
OOD detection, far & $0.92$ & \textit{0.91} &  &  &  &  &  &  &  &  &  \\
Active learning & $0.34$ & $0.35$ & \textit{0.89} &  &  &  &  &  &  &  &  \\
Selective prediction, log & $0.30$ & $0.12$ & $0.27$ & \textit{0.93} &  &  &  &  &  &  &  \\
Selective prediction, Brier & $0.48$ & $0.28$ & $0.39$ & $0.86$ & \textit{0.89} &  &  &  &  &  &  \\
Selective prediction, 0/1 & $0.40$ & $0.16$ & $0.15$ & $0.68$ & $0.75$ & \textit{0.99} &  &  &  &  &  \\
Regret, log & $0.18$ & $0.04$ & $0.18$ & $0.96$ & $0.71$ & $0.57$ & \textit{0.93} &  &  &  &  \\
Regret, Brier & $0.34$ & $0.22$ & $0.24$ & $0.55$ & $0.64$ & $0.34$ & $0.42$ & \textit{0.82} &  &  &  \\
Regret, 0/1 & $0.53$ & $0.34$ & $0.05$ & $0.53$ & $0.62$ & $0.95$ & $0.44$ & $0.25$ & \textit{0.98} &  &  \\
Regret, log, raw area & $0.14$ & $0.10$ & $0.10$ & $0.51$ & $0.30$ & $0.14$ & $0.54$ & $0.76$ & $0.12$ & \textit{0.95} &  \\
Test accuracy & $0.28$ & $0.47$ & $0.05$ & $-0.33$ & $-0.34$ & $-0.55$ & $-0.28$ & $-0.01$ & $-0.34$ & $0.18$ & \textit{0.88} \\
\bottomrule
\end{tabular}
}
\end{table}

{
\paragraph{The aleatoric and the total estimate.} \Cref{tab:notion_full_CE,tab:notion_full_Br,tab:notion_full_01} repeat selective prediction and the regret task, and \cref{tab:al_loss_notion} active learning, with the epistemic, aleatoric and total estimate of each loss as the score.
}
\begin{table}[ht]
\centering
{
\scriptsize
\setlength{\tabcolsep}{3.5pt}
\caption{\textbf{The regret task and selective prediction with the epistemic (E), aleatoric (A) and total (T) estimate as the score, log loss.} CIFAR-10 panel, fourteen methods, means over three seeds, against the CIFAR-10H annotator distribution. Rows sorted by $n$-$\AuReC$ of the epistemic estimate.}
\label{tab:notion_full_CE}
\begin{tabular}{lcccccc}
\toprule
 & \multicolumn{3}{c}{Regret} & \multicolumn{3}{c}{Selective prediction} \\
 & \multicolumn{3}{c}{$n$-$\AuReC$ $\downarrow$} & \multicolumn{3}{c}{$n$-$\AuRC$ $\downarrow$} \\
\cmidrule(lr){2-4}\cmidrule(lr){5-7}
Method & E & A & T & E & A & T \\
\midrule
Evidential & 0.153 & 0.152 & 0.153 & 0.058 & 0.059 & 0.058 \\
SWAG & 0.246 & 0.217 & 0.216 & 0.288 & 0.256 & 0.256 \\
SNGP & 0.370 & 0.338 & 0.338 & 0.481 & 0.440 & 0.440 \\
MFVI-BNN & 0.375 & 0.362 & 0.360 & 0.380 & 0.364 & 0.364 \\
DDU & 0.438 & 0.520 & 0.438 & 0.479 & 0.552 & 0.479 \\
BatchEnsemble & 0.457 & 0.442 & 0.441 & 0.474 & 0.456 & 0.456 \\
Laplace (last layer) & 0.476 & 0.466 & 0.466 & 0.506 & 0.495 & 0.495 \\
Deep Ensemble & 0.478 & 0.450 & 0.450 & 0.495 & 0.462 & 0.461 \\
SubEnsemble & 0.503 & 0.491 & 0.491 & 0.540 & 0.524 & 0.524 \\
MC-Dropout ($p{=}0.1$) & 0.540 & 0.501 & 0.501 & 0.573 & 0.532 & 0.532 \\
DropConnect ($p{=}0.5$) & 0.552 & 0.483 & 0.485 & 0.596 & 0.528 & 0.530 \\
MC-Dropout ($p{=}0.5$) & 0.557 & 0.517 & 0.518 & 0.591 & 0.551 & 0.553 \\
PostNet & 0.576 & 0.554 & 0.554 & 0.602 & 0.579 & 0.579 \\
NatPN & 1.858 & 0.526 & 0.468 & 1.716 & 0.602 & 0.551 \\
\bottomrule
\end{tabular}
}
\end{table}

\begin{table}[ht]
\centering
{
\scriptsize
\setlength{\tabcolsep}{3.5pt}
\caption{\textbf{The regret task and selective prediction with the epistemic (E), aleatoric (A) and total (T) estimate as the score, Brier loss.} CIFAR-10 panel, fourteen methods, means over three seeds, against the CIFAR-10H annotator distribution. Rows sorted by $n$-$\AuReC$ of the epistemic estimate.}
\label{tab:notion_full_Br}
\begin{tabular}{lcccccc}
\toprule
 & \multicolumn{3}{c}{Regret} & \multicolumn{3}{c}{Selective prediction} \\
 & \multicolumn{3}{c}{$n$-$\AuReC$ $\downarrow$} & \multicolumn{3}{c}{$n$-$\AuRC$ $\downarrow$} \\
\cmidrule(lr){2-4}\cmidrule(lr){5-7}
Method & E & A & T & E & A & T \\
\midrule
SWAG & 0.145 & 0.119 & 0.113 & 0.214 & 0.185 & 0.182 \\
BatchEnsemble & 0.147 & 0.145 & 0.145 & 0.284 & 0.280 & 0.280 \\
MFVI-BNN & 0.156 & 0.138 & 0.134 & 0.221 & 0.200 & 0.199 \\
Deep Ensemble & 0.163 & 0.152 & 0.151 & 0.343 & 0.316 & 0.315 \\
DDU & 0.168 & 0.186 & 0.168 & 0.348 & 0.383 & 0.348 \\
Laplace (last layer) & 0.186 & 0.186 & 0.186 & 0.352 & 0.350 & 0.350 \\
SubEnsemble & 0.186 & 0.193 & 0.193 & 0.373 & 0.370 & 0.370 \\
MC-Dropout ($p{=}0.1$) & 0.214 & 0.201 & 0.201 & 0.400 & 0.378 & 0.378 \\
MC-Dropout ($p{=}0.5$) & 0.225 & 0.210 & 0.210 & 0.397 & 0.377 & 0.377 \\
DropConnect ($p{=}0.5$) & 0.229 & 0.203 & 0.203 & 0.407 & 0.374 & 0.375 \\
SNGP & 0.238 & 0.206 & 0.206 & 0.435 & 0.386 & 0.386 \\
PostNet & 0.241 & 0.236 & 0.236 & 0.406 & 0.398 & 0.398 \\
Evidential & 0.264 & 0.264 & 0.263 & 0.046 & 0.048 & 0.046 \\
NatPN & 1.820 & 0.324 & 0.323 & 1.560 & 0.510 & 0.510 \\
\bottomrule
\end{tabular}
}
\end{table}

\begin{table}[ht]
\centering
{
\scriptsize
\setlength{\tabcolsep}{3.5pt}
\caption{\textbf{The regret task and selective prediction with the epistemic (E), aleatoric (A) and total (T) estimate as the score, 0/1 loss.} CIFAR-10 panel, fourteen methods, means over three seeds, against the CIFAR-10H annotator distribution. Rows sorted by $n$-$\AuReC$ of the epistemic estimate.}
\label{tab:notion_full_01}
\begin{tabular}{lcccccc}
\toprule
 & \multicolumn{3}{c}{Regret} & \multicolumn{3}{c}{Selective prediction} \\
 & \multicolumn{3}{c}{$n$-$\AuReC$ $\downarrow$} & \multicolumn{3}{c}{$n$-$\AuRC$ $\downarrow$} \\
\cmidrule(lr){2-4}\cmidrule(lr){5-7}
Method & E & A & T & E & A & T \\
\midrule
DDU & 0.089 & 0.084 & 0.089 & 0.311 & 0.340 & 0.311 \\
Evidential & 0.093 & 0.091 & 0.092 & 0.188 & 0.188 & 0.188 \\
PostNet & 0.155 & 0.145 & 0.145 & 0.370 & 0.354 & 0.354 \\
SWAG & 0.173 & 0.126 & 0.113 & 0.291 & 0.203 & 0.194 \\
MFVI-BNN & 0.184 & 0.130 & 0.121 & 0.293 & 0.201 & 0.196 \\
Deep Ensemble & 0.347 & 0.072 & 0.066 & 0.616 & 0.278 & 0.275 \\
BatchEnsemble & 0.353 & 0.082 & 0.081 & 0.550 & 0.236 & 0.236 \\
MC-Dropout ($p{=}0.5$) & 0.726 & 0.115 & 0.115 & 0.820 & 0.325 & 0.325 \\
DropConnect ($p{=}0.5$) & 0.744 & 0.112 & 0.112 & 0.832 & 0.325 & 0.325 \\
MC-Dropout ($p{=}0.1$) & 0.868 & 0.100 & 0.100 & 0.922 & 0.336 & 0.336 \\
SNGP & 0.888 & 0.128 & 0.128 & 0.931 & 0.351 & 0.351 \\
SubEnsemble & 0.950 & 0.100 & 0.100 & 0.971 & 0.332 & 0.332 \\
Laplace (last layer) & 0.969 & 0.103 & 0.103 & 0.981 & 0.315 & 0.315 \\
NatPN & 2.862 & 0.205 & 0.153 & 2.179 & 0.464 & 0.421 \\
\bottomrule
\end{tabular}
}
\end{table}

\begin{table}[htbp]
\centering
{
\scriptsize
\setlength{\tabcolsep}{3.5pt}
\caption{\textbf{Active learning with the decomposition of each loss.} $\NAUC$ on CIFAR-10 (rounds 1--9, means over three seeds) with the epistemic (E), aleatoric (A) or total (T) estimate of each loss as the acquisition score; the best estimate per method and loss in bold; random acquisition: $0.776$. Some methods have one epistemic estimate for every loss; it is repeated.}
\label{tab:al_loss_notion}
\begin{tabular}{lccccccccc}
\toprule
Method & \multicolumn{3}{c}{log loss} & \multicolumn{3}{c}{0/1 loss} & \multicolumn{3}{c}{Brier loss} \\
\cmidrule(lr){2-4}\cmidrule(lr){5-7}\cmidrule(lr){8-10}
 & E & A & T & E & A & T & E & A & T \\
\midrule
Evidential & 0.807 & \textbf{0.809} & 0.807 & 0.807 & \textbf{0.808} & 0.806 & \textbf{0.807} & 0.805 & 0.806 \\
SWAG & \textbf{0.798} & 0.793 & 0.794 & 0.800 & 0.793 & \textbf{0.804} & 0.796 & 0.797 & \textbf{0.800} \\
SNGP & 0.731 & \textbf{0.740} & 0.717 & 0.740 & \textbf{0.749} & 0.731 & 0.734 & \textbf{0.754} & 0.719 \\
MFVI-BNN & \textbf{0.623} & 0.612 & 0.606 & \textbf{0.624} & 0.619 & 0.622 & 0.614 & \textbf{0.618} & 0.608 \\
DDU & 0.673 & \textbf{0.733} & 0.673 & 0.673 & \textbf{0.733} & 0.673 & 0.673 & \textbf{0.719} & 0.673 \\
BatchEnsemble & 0.815 & 0.819 & \textbf{0.824} & 0.808 & 0.816 & \textbf{0.819} & 0.814 & 0.824 & \textbf{0.827} \\
Laplace (last layer) & 0.775 & 0.774 & \textbf{0.786} & \textbf{0.793} & 0.784 & 0.782 & \textbf{0.785} & 0.780 & 0.778 \\
Deep Ensemble & \textbf{0.831} & 0.823 & 0.826 & \textbf{0.834} & 0.826 & 0.831 & 0.831 & 0.827 & \textbf{0.833} \\

MC-Dropout ($p{=}0.1$) & 0.786 & \textbf{0.788} & 0.787 & 0.791 & 0.779 & \textbf{0.797} & 0.783 & \textbf{0.793} & 0.781 \\
DropConnect ($p{=}0.5$) & 0.777 & \textbf{0.804} & 0.796 & 0.761 & 0.798 & \textbf{0.811} & 0.755 & \textbf{0.808} & 0.804 \\
MC-Dropout ($p{=}0.5$) & 0.777 & \textbf{0.793} & 0.787 & 0.778 & \textbf{0.795} & 0.790 & 0.759 & \textbf{0.790} & 0.788 \\
PostNet & 0.708 & 0.726 & \textbf{0.732} & 0.708 & \textbf{0.738} & 0.736 & 0.708 & \textbf{0.737} & 0.726 \\
NatPN & 0.725 & \textbf{0.780} & 0.774 & 0.725 & \textbf{0.787} & 0.785 & 0.725 & \textbf{0.789} & 0.758 \\
\bottomrule
\end{tabular}
}
\end{table}

{
\paragraph{The family under the Brier loss.} \Cref{tab:family_brier} repeats \cref{tab:family} under the Brier loss. In the median over the methods, the Pareto gap of the best weighted sum of a method's two estimates is only $2\%$ smaller than that of its best single estimate under the Brier loss, against $6\%$ under the log loss (\cref{tab:family}).
}
\begin{table}[t]
{
\centering
\footnotesize
\setlength{\tabcolsep}{3pt}
\caption{\textbf{The family of selective tasks on CIFAR-10H under Brier loss.} Fourteen methods on ResNet-18 networks trained on CIFAR-10, means over three seeds, evaluated against the CIFAR-10H annotator distribution of the $10{,}000$ test images. The endpoint columns are the normalized areas under the risk--coverage curve ($n$-$\AuRC$, selective prediction) and under the regret--coverage curve ($n$-$\AuReC$, the regret task), each with the method's epistemic (E) or aleatoric (A) estimate of the Brier-loss decomposition as the score; $0$ is the optimal score and $1$ a random ordering. The Pareto-gap columns give how closely the method approaches the optimal surface $\mathcal{S}^{*}$ of the family (\cref{sec:family_measure}) with the best weighted sum of its two estimates (Pair) and with its epistemic (E), aleatoric (A) or total (T) estimate alone, normalized as in \cref{tab:pareto_gap} ($0$: the optimal surface, $1$: a random ordering), and Smaller gap is how much smaller, in \%, the gap of the pair is than that of the method's best single estimate. $r_s(\Ahat, \Ehat)$ is the Spearman rank correlation between the two estimates of the same decomposition over the test images, mean over the seeds. Rows are sorted by $n$-$\AuReC$ with the epistemic estimate.}
\label{tab:family_brier}
\begin{tabular}{lcccccccccc}
\toprule
 & \multicolumn{2}{c}{Selective prediction} & \multicolumn{2}{c}{Regret} & \multicolumn{4}{c}{Pareto gap $\downarrow$} & & \\
 & \multicolumn{2}{c}{$n$-$\AuRC$ $\downarrow$} & \multicolumn{2}{c}{$n$-$\AuReC$ $\downarrow$} & \multicolumn{4}{c}{} & Smaller gap (\%) & $r_s(\Ahat, \Ehat)$ \\
\cmidrule(lr){2-3}\cmidrule(lr){4-5}\cmidrule(lr){6-9}
Method & E & A & E & A & Pair & E & A & T & & \\
\midrule
SWAG & 0.214 & 0.185 & 0.145 & 0.119 & 0.088 & 0.102 & 0.091 & 0.090 & $1.4$ & $+0.947$ \\
BatchEnsemble & 0.284 & 0.280 & 0.147 & 0.145 & 0.152 & 0.157 & 0.157 & 0.157 & $3.1$ & $+0.971$ \\
MFVI-BNN & 0.221 & 0.200 & 0.156 & 0.138 & 0.100 & 0.111 & 0.102 & 0.102 & $1.7$ & $+0.958$ \\
Deep Ensemble & 0.343 & 0.316 & 0.163 & 0.152 & 0.172 & 0.186 & 0.173 & 0.173 & $0.8$ & $+0.824$ \\
DDU & 0.348 & 0.383 & 0.168 & 0.186 & 0.178 & 0.201 & 0.207 & 0.201 & $11.6$ & $+0.645$ \\
Laplace (last layer) & 0.352 & 0.350 & 0.186 & 0.186 & 0.199 & 0.202 & 0.205 & 0.205 & $1.7$ & $+0.955$ \\
SubEnsemble & 0.373 & 0.370 & 0.186 & 0.193 & 0.194 & 0.203 & 0.205 & 0.205 & $4.4$ & $+0.919$ \\
MC-Dropout ($p{=}0.1$) & 0.400 & 0.378 & 0.214 & 0.201 & 0.202 & 0.214 & 0.209 & 0.209 & $2.9$ & $+0.895$ \\
MC-Dropout ($p{=}0.5$) & 0.397 & 0.377 & 0.225 & 0.210 & 0.205 & 0.217 & 0.206 & 0.206 & $0.4$ & $+0.933$ \\
DropConnect ($p{=}0.5$) & 0.407 & 0.374 & 0.229 & 0.203 & 0.217 & 0.249 & 0.221 & 0.221 & $1.6$ & $+0.911$ \\
SNGP & 0.435 & 0.386 & 0.238 & 0.206 & 0.209 & 0.250 & 0.229 & 0.229 & $8.8$ & $+0.627$ \\
PostNet & 0.406 & 0.398 & 0.241 & 0.236 & 0.217 & 0.221 & 0.217 & 0.217 & $0.2$ & $+0.994$ \\
Evidential & 0.046 & 0.048 & 0.264 & 0.264 & 0.110 & 0.112 & 0.112 & 0.112 & $1.2$ & $+0.994$ \\
NatPN & 1.560 & 0.510 & 1.820 & 0.324 & 0.227 & 0.834 & 0.282 & 0.272 & $16.6$ & $-0.412$ \\
\bottomrule
\end{tabular}
}
\end{table}

{
\paragraph{The regret of the retrained models.} The regret contains the estimation error, which depends on the training data, together with the approximation and the optimization error (\cref{sec:conclusion}). \Cref{tab:al_trajectory} shows how it falls as labels are added, for the networks retrained by random acquisition in the active-learning experiment (\cref{app:protocols}). Their mean log-loss regret falls from $2.19$ nats at $2{,}000$ labels to $0.59$ at $10{,}000$; the panel's network, trained on $45{,}000$ images, reaches $0.34$ nats. The estimation error thus shrinks as the labeled set grows.
}
\begin{table}[htbp]
\centering
{
\footnotesize
\caption{\textbf{The regret of the retrained models.} Random acquisition on CIFAR-10, means over three seeds: after each round, the test accuracy of the retrained ResNet-18 and, under the log loss against CIFAR-10H, its mean regret in nats and its raw and normalized regret areas, with the epistemic estimate of post-hoc MC-Dropout ($p{=}0.1$) as the score. Last row: the panel's MC-Dropout ($p{=}0.1$) network (\cref{tab:panel_views}), trained with the benchmark's recipe on $45{,}000$ images.}
\label{tab:al_trajectory}
\begin{tabular}{rcccc}
\toprule
Labels & Accuracy & Mean regret & $\AuReC$ & $n$-$\AuReC$ \\
\midrule
$2{,}000$ & 0.590 & 2.188 & 1.753 & 0.753 \\
$3{,}000$ & 0.675 & 1.582 & 1.102 & 0.639 \\
$4{,}000$ & 0.728 & 1.239 & 0.763 & 0.552 \\
$5{,}000$ & 0.773 & 1.049 & 0.612 & 0.520 \\
$6{,}000$ & 0.806 & 0.780 & 0.415 & 0.464 \\
$7{,}000$ & 0.835 & 0.656 & 0.341 & 0.453 \\
$8{,}000$ & 0.838 & 0.634 & 0.324 & 0.441 \\
$9{,}000$ & 0.867 & 0.539 & 0.266 & 0.429 \\
$10{,}000$ & 0.863 & 0.592 & 0.302 & 0.448 \\
\addlinespace[2pt]
$45{,}000$ & 0.949 & 0.336 & 0.197 & 0.540 \\
\bottomrule
\end{tabular}
}
\end{table}

{
\paragraph{What active learning rewards.} \Cref{tab:al_regret_agreement} compares the $31$ arms of the active-learning experiment that retrain the same ResNet-18 and differ only in the acquisition score. Under the log and the Brier loss, the $\NAUC$ of an arm correlates with the mean regret of its retrained models at $0.76$ and $0.98$, and with how well the epistemic estimate orders their regret at only $0.17$ and $0.47$. Active learning thus rewards a model with a small regret, and much less a model whose epistemic estimate orders its regret well.
}
\begin{table}[htbp]
\centering
{
\footnotesize
\caption{\textbf{What active learning rewards.} Spearman correlation between the $\NAUC$ of an arm and the regret of the models it retrains, over the $31$ arms of the CIFAR-10 experiment that retrain the same ResNet-18 and differ only in the acquisition score: random acquisition, the three heuristics, and the epistemic, aleatoric and total estimates of the three post-hoc methods under the three losses (\cref{app:protocols}). Mean regret: of the retrained models; $n$-$\AuReC$: how well the epistemic estimate of MC-Dropout ($p{=}0.1$) orders their regret. Both are averaged over the nine rounds and three seeds, under the loss of the row. Positive: the arms with the higher $\NAUC$ have the lower value. The $\NAUC$ follows the mean regret closely and the ordering of the regret much less: active learning rewards models with a small regret, and much less models whose epistemic estimate identifies the inputs with a large regret.}
\label{tab:al_regret_agreement}
\begin{tabular}{lcc}
\toprule
Loss & Mean regret & $n$-$\AuReC$ \\
\midrule
log & 0.76 & 0.17 \\
Brier & 0.98 & 0.47 \\
0/1 & 1.00 & 0.11 \\
\bottomrule
\end{tabular}
}
\end{table}

{
\paragraph{Seed variation.} \Cref{tab:seed_sd} gives the standard deviation over the three seeds of every task and method. Its median over the methods is at most $0.021$ for every task, small compared with the differences between the methods, and the rankings of two seeds agree at $0.82$ to $0.99$.
}
\begin{table}[htbp]
\centering
{
\scriptsize
\setlength{\tabcolsep}{3pt}
\caption{\textbf{Seed variation on the CIFAR-10 panel.} Standard deviation over the three seeds of every task of \cref{fig:rank_grid}, epistemic estimate as the score; rows sorted by $n$-$\AuReC$ under the log loss (values in \cref{tab:panel_views}). Last rows: the median over the methods, and the Spearman correlation between the rankings of two seeds, mean over the three pairs.}
\label{tab:seed_sd}
\begin{tabular}{@{}lccccccccc@{}}
\toprule
 & \multicolumn{2}{c}{OOD} & AL & \multicolumn{3}{c}{Selective prediction} & \multicolumn{3}{c}{Regret} \\
\cmidrule(lr){2-3}\cmidrule(lr){5-7}\cmidrule(lr){8-10}
Method & near & far &  & log & Brier & 0/1 & log & Brier & 0/1 \\
\midrule
Evidential & 0.009 & 0.027 & 0.004 & 0.020 & 0.017 & 0.015 & 0.063 & 0.094 & 0.007 \\
SWAG & 0.006 & 0.007 & 0.007 & 0.008 & 0.002 & 0.004 & 0.008 & 0.004 & 0.005 \\
SNGP & 0.002 & 0.000 & 0.002 & 0.011 & 0.009 & 0.010 & 0.012 & 0.016 & 0.015 \\
MFVI-BNN & 0.008 & 0.022 & 0.006 & 0.039 & 0.003 & 0.009 & 0.053 & 0.007 & 0.014 \\
DDU & 0.001 & 0.002 & 0.020 & 0.017 & 0.012 & 0.009 & 0.015 & 0.003 & 0.004 \\
BatchEnsemble & 0.002 & 0.003 & 0.009 & 0.026 & 0.020 & 0.031 & 0.026 & 0.013 & 0.030 \\
Laplace (last layer) & 0.002 & 0.010 & 0.012 & 0.034 & 0.027 & 0.005 & 0.033 & 0.016 & 0.011 \\
Deep Ensemble & 0.001 & 0.001 & 0.001 & 0.012 & 0.010 & 0.016 & 0.012 & 0.004 & 0.029 \\
SubEnsemble & 0.029 & 0.023 & 0.001 & 0.039 & 0.030 & 0.004 & 0.037 & 0.011 & 0.008 \\
MC-Dropout ($p{=}0.1$) & 0.001 & 0.016 & 0.002 & 0.015 & 0.016 & 0.007 & 0.016 & 0.013 & 0.014 \\
DropConnect ($p{=}0.5$) & 0.003 & 0.002 & 0.025 & 0.011 & 0.006 & 0.009 & 0.014 & 0.005 & 0.013 \\
MC-Dropout ($p{=}0.5$) & 0.002 & 0.006 & 0.011 & 0.013 & 0.008 & 0.018 & 0.009 & 0.007 & 0.022 \\
PostNet & 0.002 & 0.008 & 0.011 & 0.071 & 0.058 & 0.069 & 0.083 & 0.037 & 0.024 \\
NatPN & 0.013 & 0.010 & 0.011 & 0.042 & 0.063 & 0.031 & 0.051 & 0.086 & 0.041 \\
\midrule
\textit{median} & \textit{0.002} & \textit{0.008} & \textit{0.008} & \textit{0.019} & \textit{0.014} & \textit{0.010} & \textit{0.021} & \textit{0.012} & \textit{0.014} \\
Two seeds & 0.91 & 0.91 & 0.89 & 0.93 & 0.89 & 0.99 & 0.93 & 0.82 & 0.98 \\
\bottomrule
\end{tabular}
}
\end{table}

\clearpage
\subsection{Selective and generalized areas}
\label{app:curve_forms}

{
The main text uses the selective risk and regret of \cref{sec:setting}, which average the loss over the accepted inputs (\cref{eq:risk,eq:regret}). This subsection reports the results under the generalized risk and regret of \citet{traub2024overcoming} (\cref{app:proofs}), which divide the loss on the accepted inputs by the number of all test inputs instead; the normalized areas are formed in the same way. \Cref{tab:curve_forms_agreement} compares the method rankings by the two forms, and the values on the generalized areas are in \cref{tab:curve_forms_values} for CIFAR-10H, \cref{tab:regret_dcic_gen} for DCIC and \cref{tab:curve_forms_appa} for APPA-REAL. The two forms rank the methods almost identically, so the conclusions do not change. The exception is the regret task under the Brier loss on CIFAR-10H, whose agreement with the other tasks rises under the generalized form.
}
\begin{table}[htbp]
\centering
{
\footnotesize
\caption{\textbf{CIFAR-10 panel on the generalized areas.} Normalized areas under the generalized regret-- and risk--coverage curves (\cref{app:curve_forms}), fourteen methods, means over three seeds, epistemic estimate as the score; best value per column in bold; rows sorted by the generalized $n$-$\AuReC$ under the log loss. The selective areas are in \cref{tab:panel_views}.}
\label{tab:curve_forms_values}
\begin{tabular}{lcccccc}
\toprule
 & \multicolumn{3}{c}{$n$-$\AuReC$, generalized} & \multicolumn{3}{c}{$n$-$\AuRC$, generalized} \\
\cmidrule(lr){2-4}\cmidrule(lr){5-7}
Method & log & Brier & 0/1 & log & Brier & 0/1 \\
\midrule
Evidential & $\mathbf{0.124}$ & $\mathbf{0.206}$ & $0.162$ & $\mathbf{0.072}$ & $\mathbf{0.058}$ & $\mathbf{0.240}$ \\
SWAG & $0.292$ & $0.239$ & $0.239$ & $0.328$ & $0.285$ & $0.323$ \\
MFVI-BNN & $0.406$ & $0.255$ & $0.255$ & $0.412$ & $0.295$ & $0.329$ \\
SNGP & $0.416$ & $0.288$ & $0.915$ & $0.499$ & $0.445$ & $0.931$ \\
DDU & $0.497$ & $0.239$ & $\mathbf{0.150}$ & $0.519$ & $0.380$ & $0.340$ \\
BatchEnsemble & $0.504$ & $0.220$ & $0.371$ & $0.509$ & $0.325$ & $0.529$ \\
Deep Ensemble & $0.520$ & $0.215$ & $0.368$ & $0.522$ & $0.358$ & $0.590$ \\
SubEnsemble & $0.528$ & $0.239$ & $0.988$ & $0.552$ & $0.389$ & $0.979$ \\
Laplace (last layer) & $0.531$ & $0.247$ & $1.000$ & $0.550$ & $0.387$ & $0.986$ \\
MC-Dropout ($p{=}0.1$) & $0.568$ & $0.263$ & $0.903$ & $0.589$ & $0.413$ & $0.923$ \\
DropConnect ($p{=}0.5$) & $0.583$ & $0.291$ & $0.747$ & $0.611$ & $0.427$ & $0.814$ \\
MC-Dropout ($p{=}0.5$) & $0.586$ & $0.282$ & $0.729$ & $0.609$ & $0.416$ & $0.804$ \\
PostNet & $0.607$ & $0.299$ & $0.210$ & $0.621$ & $0.426$ & $0.386$ \\
NatPN & $1.301$ & $1.125$ & $1.642$ & $1.258$ & $1.079$ & $1.426$ \\
\bottomrule
\end{tabular}
}
\end{table}

\begin{table}[htbp]
\centering
{
\footnotesize
\caption{\textbf{Selective and generalized areas rank the methods alike.} Spearman correlation between the method ranking by the normalized selective area of \cref{eq:aurc} and by the normalized generalized area of \citet{traub2024overcoming} (\cref{app:curve_forms}), for the same task and loss, with the epistemic estimate as the score; $1$: identical rankings. CIFAR-10H: fourteen methods, means over three seeds. DCIC: fifteen methods trained on the annotator distributions, means over three seeds. APPA-REAL: fourteen methods, means over five seeds. The two areas rank the methods almost identically, so the agreements between the tasks in \cref{sec:experiments} hold under either. The exception is the regret task under the Brier loss on CIFAR-10H: there the generalized areas raise its agreement with selective prediction, far-OOD detection and active learning from $0.64$, $0.22$ and $0.24$ to $0.87$, $0.53$ and $0.56$.}
\label{tab:curve_forms_agreement}
\begin{tabular}{lcccccc}
\toprule
 & \multicolumn{3}{c}{Regret} & \multicolumn{3}{c}{Selective prediction} \\
\cmidrule(lr){2-4}\cmidrule(l){5-7}
Dataset & log & Brier & 0/1 & log & Brier & 0/1 \\
\midrule
CIFAR-10H & 0.98 & 0.56 & 1.00 & 0.99 & 1.00 & 1.00 \\
\addlinespace[2pt]
\multicolumn{7}{@{}l}{DCIC} \\
\quad Benthic & 0.92 & 0.86 & 0.95 & 0.86 & 0.85 & 0.96 \\
\quad MiceBone & 0.99 & 0.99 & 1.00 & 0.99 & 0.99 & 1.00 \\
\quad Pig & 0.96 & 0.94 & 0.99 & 0.98 & 0.97 & 0.99 \\
\quad Plankton & 0.94 & 0.98 & 0.92 & 0.92 & 0.86 & 0.99 \\
\quad QualityMRI & 0.86 & 0.80 & 0.96 & 0.97 & 0.98 & 0.95 \\
\quad Treeversity\#1 & 0.92 & 0.99 & 0.97 & 0.92 & 0.96 & 0.90 \\
\quad Treeversity\#6 & 1.00 & 0.98 & 0.90 & 0.97 & 0.93 & 1.00 \\
\quad Turkey & 0.99 & 1.00 & 0.99 & 0.93 & 0.95 & 1.00 \\
\midrule
 & absolute & squared & & absolute & squared & \\
\midrule
APPA-REAL & 0.99 & 0.92 &  & 0.94 & 1.00 &  \\
\bottomrule
\end{tabular}}
\end{table}

\clearpage
\section{Detailed Results on DCIC and APPA-REAL}
\label{app:dcic_appa_results}
\suppressfloats[t]

{
This appendix reports the DCIC and APPA-REAL results of \cref{sec:experiments} in detail. \Cref{tab:summary_full} repeats \cref{tab:summary} for both kinds of training labels. \Crefrange{tab:regret_dcic}{tab:results-turkey} give the regret and OOD detection of every method on DCIC and \cref{tab:regret_appareal} on APPA-REAL; \cref{tab:regret_dcic_gen,tab:curve_forms_appa} use the generalized areas of \cref{app:curve_forms}, \cref{tab:al_true,tab:al_appareal} report active learning, and \crefrange{tab:agreement_benthic}{tab:agreement_appa_real} give the agreement between the tasks on each dataset, as \cref{tab:view_agreement} does on CIFAR-10H.

\paragraph{Training labels.} As in the main text, the DCIC networks and the APPA-REAL heads are trained on the annotator distribution of each image. \Cref{tab:regret_dcic_one} and the one-label blocks of \cref{tab:summary_full,tab:al_true} use one label per image instead: the majority label on DCIC and the median of the votes on APPA-REAL (\cref{app:implementation}).

\paragraph{Losses.} With one label per image, the regret ranking depends on the loss: the rankings under the log and the Brier loss correlate at $0.42$ on CIFAR-10H and at a median of $0.40$ on DCIC. With the annotator distributions they agree more, at a median of $0.79$ on DCIC (\cref{tab:summary_full}).

\paragraph{Near- and far-OOD detection.} On APPA-REAL, the rankings by near- and far-OOD detection correlate at $0.74$ under both kinds of training labels; each ranking uses the AUROC averaged over five seeds and, for far OOD, over ImageNet and Places365.

\paragraph{Agreement of two seeds.} The agreement of two seeds is the rank correlation between their regret rankings, averaged over all pairs of seeds. With the fifteen methods of DCIC, a rank correlation below $0.51$ cannot be told apart from zero at the $5\%$ level (below $0.53$ with the fourteen methods of CIFAR-10H and APPA-REAL). On DCIC, whose test fold changes with the seed, the agreement falls below this threshold on QualityMRI, Plankton and Turkey with the annotator distributions, and on Benthic and MiceBone with one label per image (\cref{tab:summary_full}). \Cref{tab:summary} shows the first three in gray and leaves them out of its medians.

\paragraph{Active learning on DCIC.} \Cref{tab:al_true} compares acquisition by the true regret $E^{*}_{\log}$ with random acquisition and with the standard scores of the same ResNet-18. With the annotator distributions, $E^{*}_{\log}$ stays within $3\%$ of random acquisition in $\NAUC$ on every dataset. With one label per image, it falls below random acquisition in every seed on six of the seven real datasets. On the seven real datasets, a standard score of the same network reaches a higher $\NAUC$ than $E^{*}_{\log}$ in $79$ of $81$ comparisons.

\paragraph{Active learning on APPA-REAL.} In \cref{tab:al_appareal} each method retrains its own head. Acquisition by the true regret $E^*$ gives a lower test risk than random acquisition for every method, and the method's own epistemic or total estimate gives a lower one still for ten of the fourteen methods. With one label per image, $E^*$ instead gives a lower test risk than the method's epistemic estimate for ten of the fourteen methods. The last-layer Laplace head then diverges under both scores: its test risk is $26.9$ years under acquisition by $E^*$ and $17.4$ years under $\hat E$, against $7.3$ under random acquisition.
}

\begin{table}[htbp]
\centering
{
\footnotesize
\setlength{\tabcolsep}{3pt}
\caption{\textbf{\Cref{tab:summary} for both kinds of training labels.} The columns of \cref{tab:summary}, with the agreement of the regret rankings under the log and the Brier loss added (Log--Brier); near/far OOD where both exist. Gray: the regret rankings of two seeds are not significantly correlated at the $5\%$ level under that kind of training labels; left out of the median. The one-label medians of \cref{tab:summary} differ from those below, since they are taken over the datasets that are not gray under the annotator distributions: $0.88$ (Regret--risk), $0.59$ (Regret share) and $-4.0$ (AL by the true regret). On APPA-REAL, AL by the true regret is the median over the methods of the change for each method's head. --: not available; QualityMRI has no active learning, since its about $250$ non-test images are fewer than the $500$ initial labels. Under both kinds of training labels, the regret task agrees closely with selective prediction and much less with OOD detection and active learning; with one label per image, its rankings under the log and the Brier loss agree less (a median of $0.40$ against $0.79$ on DCIC).}
\label{tab:summary_full}
\resizebox{\textwidth}{!}{\begin{tabular}{@{}l cccccc @{\hspace{7pt}} cc @{\hspace{7pt}} c@{}}
\toprule
 & \multicolumn{6}{c}{Agreement of the method rankings} & \multicolumn{2}{c}{True risk and regret} & \makecell{AL by the\\true regret} \\
\cmidrule(lr){2-7}\cmidrule(lr){8-9}\cmidrule(l){10-10}
Dataset & Two seeds & Regret--OOD & Regret--AL & OOD--AL & Regret--risk & Log--Brier & Regret share & Correlation & vs random, \% \\
\midrule
\multicolumn{10}{@{}l}{\textit{Trained on one label per image}} \\
CIFAR-10H & 0.93 & 0.18/0.04 & 0.18 & 0.34/0.35 & 0.96 & 0.42 & 0.69 & 0.99 & -- \\
\addlinespace[2pt]
\multicolumn{10}{@{}l}{DCIC} \\
\quad \textcolor{black!45}{Benthic} & \textcolor{black!45}{0.48} & \textcolor{black!45}{0.25} & \textcolor{black!45}{--} & \textcolor{black!45}{--} & \textcolor{black!45}{0.88} & \textcolor{black!45}{0.31} & \textcolor{black!45}{0.71} & \textcolor{black!45}{0.96} & \textcolor{black!45}{$-$7.4} \\
\quad \textcolor{black!45}{MiceBone} & \textcolor{black!45}{0.50} & \textcolor{black!45}{0.11} & \textcolor{black!45}{0.00} & \textcolor{black!45}{$-$0.73} & \textcolor{black!45}{0.87} & \textcolor{black!45}{0.44} & \textcolor{black!45}{0.55} & \textcolor{black!45}{0.95} & \textcolor{black!45}{$+$0.5} \\
\quad Pig & 0.67 & 0.26 & $-$0.36 & $-$0.30 & 0.64 & 0.53 & 0.44 & 0.79 & $-$6.4 \\
\quad Plankton & 0.63 & $-$0.21 & 0.15 & $-$0.08 & 0.98 & 0.44 & 0.67 & 0.99 & $-$1.8 \\
\quad QualityMRI & 0.74 & $-$0.30 & -- & -- & 0.90 & 0.24 & 0.48 & 0.93 & -- \\
\quad Treeversity\#1 & 0.66 & 0.19 & -- & -- & 0.97 & 0.36 & 0.65 & 0.95 & $-$3.9 \\
\quad Treeversity\#6 & 0.74 & $-$0.04 & -- & -- & 0.98 & $-$0.19 & 0.59 & 0.95 & $-$4.0 \\
\quad Turkey & 0.69 & 0.10 & 0.05 & 0.10 & 0.96 & 0.52 & 0.65 & 0.96 & $-$1.0 \\
\cmidrule(l){2-10}
\quad \textit{median} & \textit{0.68} & \textit{0.03} & \textit{0.05} & \textit{$-$0.08} & \textit{0.96} & \textit{0.40} & \textit{0.62} & \textit{0.95} & \textit{$-$3.9} \\
\addlinespace[2pt]
APPA-REAL & 0.89 & 0.17/0.02 & 0.06 & 0.61/0.55 & 0.90 & -- & 0.37 & 0.72 & $+$7.5 \\
\midrule
\multicolumn{10}{@{}l}{\textit{Trained on the annotator distribution of each image}} \\
\addlinespace[2pt]
\multicolumn{10}{@{}l}{DCIC} \\
\quad Benthic & 0.71 & 0.27 & $-$0.06 & $-$0.30 & 0.87 & 0.69 & 0.62 & 0.87 & $-$1.2 \\
\quad MiceBone & 0.90 & $-$0.05 & 0.12 & $-$0.21 & 0.98 & 0.90 & 0.37 & 0.76 & $-$0.2 \\
\quad Pig & 0.56 & 0.21 & $-$0.26 & $-$0.26 & 0.81 & 0.79 & 0.34 & 0.50 & $-$2.8 \\
\quad \textcolor{black!45}{Plankton} & \textcolor{black!45}{0.41} & \textcolor{black!45}{$-$0.01} & \textcolor{black!45}{$-$0.01} & \textcolor{black!45}{$-$0.17} & \textcolor{black!45}{0.86} & \textcolor{black!45}{0.31} & \textcolor{black!45}{0.54} & \textcolor{black!45}{0.97} & \textcolor{black!45}{0.0} \\
\quad \textcolor{black!45}{QualityMRI} & \textcolor{black!45}{0.02} & \textcolor{black!45}{0.00} & \textcolor{black!45}{--} & \textcolor{black!45}{--} & \textcolor{black!45}{$-$0.27} & \textcolor{black!45}{0.59} & \textcolor{black!45}{0.05} & \textcolor{black!45}{0.12} & \textcolor{black!45}{--} \\
\quad Treeversity\#1 & 0.68 & 0.10 & $-$0.14 & 0.64 & 0.64 & 0.31 & 0.56 & 0.90 & $+$0.1 \\
\quad Treeversity\#6 & 0.82 & 0.52 & 0.09 & 0.39 & 0.25 & 0.90 & 0.23 & 0.52 & $+$0.3 \\
\quad \textcolor{black!45}{Turkey} & \textcolor{black!45}{0.50} & \textcolor{black!45}{0.30} & \textcolor{black!45}{0.46} & \textcolor{black!45}{0.02} & \textcolor{black!45}{0.86} & \textcolor{black!45}{0.82} & \textcolor{black!45}{0.52} & \textcolor{black!45}{0.86} & \textcolor{black!45}{$+$1.4} \\
\cmidrule(l){2-10}
\quad \textit{median} & \textit{0.71} & \textit{0.21} & \textit{$-$0.06} & \textit{$-$0.21} & \textit{0.81} & \textit{0.79} & \textit{0.37} & \textit{0.76} & \textit{$-$0.2} \\
\addlinespace[2pt]
APPA-REAL & 0.92 & 0.64/0.47 & $-$0.08 & 0.21/0.31 & 0.92 & -- & 0.35 & 0.69 & $+$2.8 \\
\bottomrule
\end{tabular}}
}
\end{table}

\begin{table}[htbp]
\centering
{
\small
\caption{\textbf{Regret and OOD detection on DCIC, networks trained on the annotator distributions.} Fifteen methods, ResNet-18 fine-tuned from ImageNet weights on each of the eight real DCIC datasets, means over three seeds, the epistemic estimate as the score. \textit{Best}: the number of datasets on which the method has the lowest $n$-$\AuReC$, or the highest $\AuROC$ against the test images of the other DCIC datasets; \textit{Mean}: the mean over the eight datasets. Rows by the mean $n$-$\AuReC$ under the log loss. Per-dataset results are in \cref{tab:results-benthic,tab:results-mice-bone,tab:results-pig,tab:results-plankton,tab:results-quality-mri,tab:results-treeversity-1,tab:results-treeversity-6,tab:results-turkey}. No method has the lowest $n$-$\AuReC$ under the log loss on more than three of the eight datasets. DDU, which has the highest mean $\AuROC$, is fourteenth of the fifteen methods in the regret task.}
\label{tab:regret_dcic}
\begin{tabular}{lcccccccc}
\toprule
 & \multicolumn{6}{c}{Regret, $n$-$\AuReC$ $\downarrow$} & \multicolumn{2}{c}{OOD, $\AuROC$ $\uparrow$} \\
\cmidrule(lr){2-7}\cmidrule(lr){8-9}
 & \multicolumn{2}{c}{log} & \multicolumn{2}{c}{Brier} & \multicolumn{2}{c}{0/1} & \multicolumn{2}{c}{far} \\
\cmidrule(lr){2-3}\cmidrule(lr){4-5}\cmidrule(lr){6-7}\cmidrule(lr){8-9}
Method & Best & Mean & Best & Mean & Best & Mean & Best & Mean \\
\midrule
BatchEnsemble & 3 & $0.469$ & 1 & $0.402$ & 1 & $0.454$ & 0 & $0.808$ \\
MFVI-BNN & 1 & $0.511$ & 0 & $0.439$ & 0 & $0.583$ & 0 & $0.500$ \\
Deep Ensemble & 0 & $0.513$ & 1 & $0.401$ & 0 & $0.462$ & 0 & $0.790$ \\
PostNet & 1 & $0.517$ & 2 & $0.381$ & 6 & $0.292$ & 0 & $0.662$ \\
Laplace (last layer) & 1 & $0.523$ & 0 & $0.411$ & 1 & $0.488$ & 2 & $0.843$ \\
Evidential & 0 & $0.533$ & 2 & $0.385$ & 0 & $0.513$ & 2 & $0.717$ \\
SubEnsemble & 0 & $0.550$ & 0 & $0.419$ & 0 & $0.753$ & 0 & $0.755$ \\
MC-Dropout ($p{=}0.5$) & 0 & $0.552$ & 0 & $0.404$ & 0 & $0.678$ & 0 & $0.680$ \\
MC-Dropout ($p{=}0.1$) & 0 & $0.552$ & 1 & $0.435$ & 0 & $0.770$ & 0 & $0.666$ \\
SNGP & 2 & $0.553$ & 0 & $0.492$ & 0 & $0.646$ & 1 & $0.793$ \\
DropConnect ($p{=}0.5$) & 0 & $0.555$ & 0 & $0.444$ & 0 & $0.643$ & 0 & $0.629$ \\
DropConnect ($p{=}0.1$) & 0 & $0.571$ & 0 & $0.457$ & 0 & $0.684$ & 0 & $0.623$ \\
SWAG & 0 & $0.585$ & 1 & $0.413$ & 0 & $0.617$ & 0 & $0.690$ \\
DDU & 0 & $0.949$ & 0 & $0.946$ & 0 & $0.902$ & 3 & $0.921$ \\
NatPN & 0 & $1.165$ & 0 & $0.925$ & 0 & $1.130$ & 0 & $0.521$ \\
\bottomrule
\end{tabular}
}
\end{table}

\begin{table}[htbp]
\centering
{
\small
\caption{\textbf{Regret and OOD detection on DCIC, networks trained on one label per image.} Fifteen methods, ResNet-18 fine-tuned from ImageNet weights on each of the eight real DCIC datasets, all methods trained on the majority label of each image, means over three seeds, the epistemic estimate as the score. \textit{Best}: the number of datasets on which the method has the lowest $n$-$\AuReC$, or the highest $\AuROC$ against the test images of the other DCIC datasets; \textit{Mean}: the mean over the eight datasets. Rows by the mean $n$-$\AuReC$ under the log loss. The same with every method trained on the annotator distributions is \cref{tab:regret_dcic}. Evidential has the lowest $n$-$\AuReC$ under the log loss on seven of the eight datasets. DDU, which has the highest mean $\AuROC$, is ninth of the fifteen methods in the regret task.}
\label{tab:regret_dcic_one}
\begin{tabular}{lcccccccc}
\toprule
 & \multicolumn{6}{c}{Regret, $n$-$\AuReC$ $\downarrow$} & \multicolumn{2}{c}{OOD, $\AuROC$ $\uparrow$} \\
\cmidrule(lr){2-7}\cmidrule(lr){8-9}
 & \multicolumn{2}{c}{log} & \multicolumn{2}{c}{Brier} & \multicolumn{2}{c}{0/1} & \multicolumn{2}{c}{far} \\
\cmidrule(lr){2-3}\cmidrule(lr){4-5}\cmidrule(lr){6-7}\cmidrule(lr){8-9}
Method & Best & Mean & Best & Mean & Best & Mean & Best & Mean \\
\midrule
Evidential & 7 & $0.589$ & 3 & $0.468$ & 6 & $0.290$ & 2 & $0.812$ \\
SNGP & 0 & $0.745$ & 0 & $0.594$ & 0 & $0.773$ & 2 & $0.812$ \\
MFVI-BNN & 0 & $0.811$ & 0 & $0.523$ & 0 & $0.578$ & 0 & $0.529$ \\
Deep Ensemble & 0 & $0.820$ & 1 & $0.489$ & 0 & $0.406$ & 0 & $0.702$ \\
PostNet & 0 & $0.829$ & 0 & $0.576$ & 1 & $0.369$ & 0 & $0.755$ \\
BatchEnsemble & 0 & $0.865$ & 0 & $0.545$ & 0 & $0.530$ & 0 & $0.753$ \\
DropConnect ($p{=}0.5$) & 0 & $0.868$ & 0 & $0.533$ & 0 & $0.699$ & 0 & $0.611$ \\
MC-Dropout ($p{=}0.5$) & 0 & $0.872$ & 1 & $0.490$ & 0 & $0.690$ & 0 & $0.645$ \\
DDU & 0 & $0.884$ & 0 & $0.876$ & 0 & $0.847$ & 3 & $0.891$ \\
SubEnsemble & 0 & $0.887$ & 0 & $0.498$ & 0 & $0.732$ & 0 & $0.633$ \\
DropConnect ($p{=}0.1$) & 0 & $0.896$ & 0 & $0.511$ & 0 & $0.836$ & 0 & $0.632$ \\
MC-Dropout ($p{=}0.1$) & 0 & $0.918$ & 0 & $0.507$ & 0 & $0.855$ & 0 & $0.657$ \\
Laplace (last layer) & 0 & $0.930$ & 0 & $0.515$ & 0 & $0.591$ & 1 & $0.819$ \\
SWAG & 0 & $1.017$ & 3 & $0.465$ & 1 & $0.416$ & 0 & $0.740$ \\
NatPN & 1 & $1.083$ & 0 & $0.931$ & 0 & $1.214$ & 0 & $0.531$ \\
\bottomrule
\end{tabular}
}
\end{table}

\begin{table}[htbp]
\centering
{
\small
\caption{\textbf{Regret and OOD detection on DCIC, generalized areas.} Fifteen methods, ResNet-18 fine-tuned from ImageNet weights on each of the eight real DCIC datasets, all methods trained on the annotator distributions, means over three seeds, the epistemic estimate as the score; normalized areas under the generalized regret--coverage curves (\cref{app:curve_forms}) instead of the selective ones of \cref{tab:regret_dcic}. The OOD columns are those of \cref{tab:regret_dcic}. \textit{Best}: the number of datasets on which the method has the lowest $n$-$\AuReC$, or the highest $\AuROC$ against the test images of the other DCIC datasets; \textit{Mean}: the mean over the eight datasets. Rows by the mean $n$-$\AuReC$ under the log loss.  The ranking is close to that of the selective areas in \cref{tab:regret_dcic} (\cref{tab:curve_forms_agreement}).}
\label{tab:regret_dcic_gen}
\begin{tabular}{lcccccccc}
\toprule
 & \multicolumn{6}{c}{Regret, $n$-$\AuReC$ $\downarrow$} & \multicolumn{2}{c}{OOD, $\AuROC$ $\uparrow$} \\
\cmidrule(lr){2-7}\cmidrule(lr){8-9}
 & \multicolumn{2}{c}{log} & \multicolumn{2}{c}{Brier} & \multicolumn{2}{c}{0/1} & \multicolumn{2}{c}{far} \\
\cmidrule(lr){2-3}\cmidrule(lr){4-5}\cmidrule(lr){6-7}\cmidrule(lr){8-9}
Method & Best & Mean & Best & Mean & Best & Mean & Best & Mean \\
\midrule
BatchEnsemble & 3 & $0.558$ & 1 & $0.512$ & 2 & $0.523$ & 0 & $0.808$ \\
Deep Ensemble & 1 & $0.593$ & 1 & $0.510$ & 0 & $0.507$ & 0 & $0.790$ \\
PostNet & 1 & $0.593$ & 2 & $0.487$ & 5 & $0.413$ & 0 & $0.662$ \\
MFVI-BNN & 1 & $0.596$ & 0 & $0.551$ & 0 & $0.618$ & 0 & $0.500$ \\
Evidential & 0 & $0.607$ & 3 & $0.461$ & 0 & $0.623$ & 2 & $0.717$ \\
Laplace (last layer) & 1 & $0.621$ & 0 & $0.536$ & 1 & $0.539$ & 2 & $0.843$ \\
MC-Dropout ($p{=}0.5$) & 0 & $0.638$ & 0 & $0.527$ & 0 & $0.703$ & 0 & $0.680$ \\
SNGP & 1 & $0.644$ & 0 & $0.598$ & 0 & $0.657$ & 1 & $0.793$ \\
SubEnsemble & 0 & $0.645$ & 0 & $0.537$ & 0 & $0.779$ & 0 & $0.755$ \\
SWAG & 0 & $0.646$ & 0 & $0.517$ & 0 & $0.655$ & 0 & $0.690$ \\
DropConnect ($p{=}0.5$) & 0 & $0.647$ & 0 & $0.571$ & 0 & $0.676$ & 0 & $0.629$ \\
MC-Dropout ($p{=}0.1$) & 0 & $0.666$ & 1 & $0.567$ & 0 & $0.813$ & 0 & $0.666$ \\
DropConnect ($p{=}0.1$) & 0 & $0.693$ & 0 & $0.595$ & 0 & $0.720$ & 0 & $0.623$ \\
DDU & 0 & $0.964$ & 0 & $0.962$ & 0 & $0.943$ & 3 & $0.921$ \\
NatPN & 0 & $1.091$ & 0 & $0.902$ & 0 & $1.085$ & 0 & $0.521$ \\
\bottomrule
\end{tabular}
}
\end{table}

\begin{table}[htbp]
\centering
{
\footnotesize
\caption{\textbf{APPA-REAL on the generalized areas.} Normalized areas under the generalized regret-- and risk--coverage curves (\cref{app:curve_forms}), fourteen methods, heads trained on the annotator distributions, means over five seeds, the epistemic estimate as the score; best value per column in bold; rows by the selective $n$-$\AuReC$ under the absolute loss. The selective areas are in \cref{tab:regret_appareal}.}
\label{tab:curve_forms_appa}
\begin{tabular}{lcccc}
\toprule
 & \multicolumn{2}{c}{$n$-$\AuReC$, generalized} & \multicolumn{2}{c}{$n$-$\AuRC$, generalized} \\
\cmidrule(lr){2-3}\cmidrule(lr){4-5}
Method & absolute & squared & absolute & squared \\
\midrule
PostNet & $\mathbf{0.332}$ & $\mathbf{0.108}$ & $\mathbf{0.398}$ & $\mathbf{0.118}$ \\
MC-Dropout ($p{=}0.5$) & $0.544$ & $0.392$ & $0.405$ & $0.368$ \\
MFVI-BNN & $0.563$ & $0.410$ & $0.449$ & $0.408$ \\
SWAG & $0.618$ & $0.415$ & $0.501$ & $0.376$ \\
MC-Dropout ($p{=}0.1$) & $0.627$ & $0.437$ & $0.512$ & $0.401$ \\
BatchEnsemble & $0.650$ & $0.482$ & $0.552$ & $0.491$ \\
DropConnect ($p{=}0.5$) & $0.640$ & $0.418$ & $0.646$ & $0.428$ \\
SNGP & $0.660$ & $0.402$ & $0.681$ & $0.443$ \\
SubEnsemble & $0.723$ & $0.435$ & $0.687$ & $0.426$ \\
Evidential & $0.715$ & $0.446$ & $0.757$ & $0.464$ \\
Deep Ensemble & $0.763$ & $0.541$ & $0.711$ & $0.510$ \\
Laplace (last layer) & $0.809$ & $0.604$ & $0.732$ & $0.591$ \\
NatPN & $0.833$ & $0.832$ & $0.894$ & $0.846$ \\
DDU & $1.040$ & $1.029$ & $1.085$ & $1.021$ \\
\bottomrule
\end{tabular}
}
\end{table}

\begin{table}[t]
{
\centering
\footnotesize
\setlength{\tabcolsep}{3pt}
\caption{\textbf{Active learning on DCIC with the true quantities as the acquisition score.} The seven real DCIC datasets with active learning and the synthetic one; from $500$ labeled images, eight rounds of $500$ acquired images, the ResNet-18 without an uncertainty method retrained after every round, three seeds (\cref{app:protocols}). Benthic, whose $3{,}877$ to $3{,}912$ non-test images are fewer than the $4{,}500$ labels of the protocol, starts from $250$ labeled images and acquires $250$ per round; QualityMRI has no active learning, since its about $250$ non-test images are fewer than the $500$ initial labels. $\NAUC$: the mean test accuracy over the eight rounds; Final accuracy: after the last round. $E^{*}_{\log}$, $A^{*}_{\log}$, $T^{*}_{\log}$: acquisition by the true regret, aleatoric component and risk of the current model under the log loss. Best other arm: the best of the other arms on the same network, the epistemic (E) or total (T) estimate of post-hoc MC-Dropout, DropConnect or a last-layer Laplace approximation, or the margin, entropy or least-confidence heuristic. Above: how many of these arms reach a higher $\NAUC$ than $E^{*}_{\log}$ and than random acquisition, out of those run on the dataset. $\dagger$: below random acquisition in every seed; --: not run. Acquisition by the true regret stays within $3\%$ of random acquisition with the annotator distributions and falls below it with one label per image; on the real datasets, a standard score of the same network beats it in $79$ of the $81$ comparisons.}
\label{tab:al_true}
\begin{tabular}{lcccccl@{\hspace{2\tabcolsep}}cc@{\hspace{2\tabcolsep}}cc}
\toprule
 & \multicolumn{6}{c}{$\NAUC$ $\uparrow$} & \multicolumn{2}{c}{Above} & \multicolumn{2}{c}{Final accuracy $\uparrow$} \\
\cmidrule(lr){2-7}\cmidrule(lr){8-9}\cmidrule(lr){10-11}
Dataset & Random & $E^{*}_{\log}$ & $A^{*}_{\log}$ & $T^{*}_{\log}$ & \multicolumn{2}{c}{Best other arm} & $E^{*}_{\log}$ & Random & Random & $E^{*}_{\log}$ \\
\midrule
\multicolumn{11}{@{}l}{\textit{Trained on one label per image}} \\
Benthic & 0.653 & 0.605$^{\dagger}$ & 0.581$^{\dagger}$ & 0.589$^{\dagger}$ & -- & -- & -- & -- & 0.672 & 0.684 \\
MiceBone & 0.771 & 0.775 & 0.765$^{\dagger}$ & 0.771 & 0.791 & Entropy & 9/9 & 9/9 & 0.780 & 0.804 \\
Pig & 0.510 & 0.477$^{\dagger}$ & 0.460$^{\dagger}$ & 0.449$^{\dagger}$ & 0.513 & MC-Dropout, E & 9/9 & 3/9 & 0.519 & 0.523 \\
Plankton & 0.872 & 0.856$^{\dagger}$ & 0.800$^{\dagger}$ & 0.851$^{\dagger}$ & 0.887 & Margin & 9/9 & 9/9 & 0.875 & 0.899 \\
Treeversity\#1 & 0.772 & 0.741$^{\dagger}$ & 0.723$^{\dagger}$ & 0.734$^{\dagger}$ & -- & -- & -- & -- & 0.789 & 0.785 \\
Treeversity\#6 & 0.762 & 0.732$^{\dagger}$ & 0.717$^{\dagger}$ & 0.731$^{\dagger}$ & -- & -- & -- & -- & 0.779 & 0.776 \\
Turkey & 0.817 & 0.809$^{\dagger}$ & 0.776$^{\dagger}$ & 0.797$^{\dagger}$ & 0.843 & Laplace, E & 9/9 & 9/9 & 0.829 & 0.855 \\
Synthetic & 0.972 & 0.947$^{\dagger}$ & 0.965$^{\dagger}$ & 0.948$^{\dagger}$ & 0.984 & Laplace, E & 9/9 & 8/9 & 0.987 & 0.980 \\
\midrule
\multicolumn{11}{@{}l}{\textit{Trained on the annotator distribution of each image}} \\
Benthic & 0.693 & 0.685 & 0.626$^{\dagger}$ & 0.664$^{\dagger}$ & 0.713 & MC-Dropout, E & 3/3 & 3/3 & 0.740 & 0.736 \\
MiceBone & 0.809 & 0.808 & 0.813 & 0.814 & 0.827 & DropConnect, T & 9/9 & 8/9 & 0.827 & 0.822 \\
Pig & 0.568 & 0.552$^{\dagger}$ & 0.525$^{\dagger}$ & 0.515$^{\dagger}$ & 0.575 & DropConnect, T & 8/9 & 7/9 & 0.594 & 0.593 \\
Plankton & 0.902 & 0.902 & 0.889$^{\dagger}$ & 0.903 & 0.916 & Laplace, E & 3/3 & 3/3 & 0.913 & 0.924 \\
Treeversity\#1 & 0.811 & 0.812 & 0.800$^{\dagger}$ & 0.808 & 0.828 & Least confidence & 9/9 & 9/9 & 0.831 & 0.836 \\
Treeversity\#6 & 0.783 & 0.785 & 0.767$^{\dagger}$ & 0.784 & 0.799 & Margin & 8/9 & 9/9 & 0.799 & 0.821 \\
Turkey & 0.853 & 0.865 & 0.842 & 0.858 & 0.873 & MC-Dropout, E & 3/3 & 3/3 & 0.839 & 0.866 \\
Synthetic & 0.977 & 0.967$^{\dagger}$ & 0.972$^{\dagger}$ & 0.977 & 0.980 & Margin & 9/9 & 6/9 & 0.989 & 0.984 \\
\bottomrule
\end{tabular}
}
\end{table}

\begin{table}[htbp]
\centering
{
\footnotesize
\caption{\textbf{Active learning on APPA-REAL.} Fourteen methods, heads trained on the annotator distributions; from $100$ labeled images, ten rounds of $100$ acquired images, each method retraining its own head after every round, five seeds (\cref{app:protocols}). Test risk under the absolute loss in years, mean over the ten rounds and the seeds, for acquisition at random, by the method's epistemic and total estimates $\hat E$ and $\hat T$, and by the true regret $E^*$, aleatoric component $A^*$ and risk $T^*$ of the current model. Rows by the test risk under $\hat E$; the lowest test risk of each row in bold. Acquisition by the true regret gives a lower test risk than random acquisition for every method, and the method's own epistemic or total estimate gives a lower one still for ten of the fourteen methods.}
\label{tab:al_appareal}
\begin{tabular}{lcccccc}
\toprule
Method & Random & $\hat E$ & $\hat T$ & $E^*$ & $A^*$ & $T^*$ \\
\midrule
SubEnsemble & $6.33$ & $\mathbf{6.05}$ & $6.14$ & $6.16$ & $6.75$ & $6.20$ \\
Deep Ensemble & $6.32$ & $\mathbf{6.09}$ & $6.14$ & $6.12$ & $6.77$ & $6.17$ \\
MC-Dropout ($p{=}0.1$) & $6.33$ & $\mathbf{6.13}$ & $6.14$ & $6.15$ & $6.81$ & $6.18$ \\
MC-Dropout ($p{=}0.5$) & $6.38$ & $\mathbf{6.14}$ & $6.16$ & $6.28$ & $7.00$ & $6.30$ \\
Laplace (last layer) & $6.29$ & $6.14$ & $6.14$ & $\mathbf{6.10}$ & $6.75$ & $6.13$ \\
DropConnect ($p{=}0.5$) & $6.41$ & $\mathbf{6.16}$ & $6.22$ & $6.27$ & $7.40$ & $6.32$ \\
SWAG & $6.32$ & $6.17$ & $\mathbf{6.15}$ & $6.17$ & $6.87$ & $6.22$ \\
BatchEnsemble & $6.41$ & $\mathbf{6.17}$ & $6.20$ & $6.24$ & $7.32$ & $6.29$ \\
Evidential & $6.30$ & $6.22$ & $6.17$ & $\mathbf{6.15}$ & $6.87$ & $6.19$ \\
DDU & $6.32$ & $6.41$ & $6.38$ & $\mathbf{6.14}$ & $6.68$ & $6.16$ \\
SNGP & $6.87$ & $\mathbf{6.42}$ & $6.43$ & $6.62$ & $9.04$ & $6.64$ \\
MFVI-BNN & $6.81$ & $6.48$ & $6.43$ & $\mathbf{6.24}$ & $8.08$ & $6.28$ \\
NatPN & $11.02$ & $\mathbf{7.82}$ & $8.02$ & $8.59$ & $11.74$ & $8.39$ \\
PostNet & $10.89$ & $8.24$ & $\mathbf{8.12}$ & $8.63$ & $12.07$ & $8.72$ \\
\bottomrule
\end{tabular}
}
\end{table}

\clearpage
\subsection{Results on each dataset}
For each DCIC dataset and for APPA-REAL, one table gives the values of every task, as \cref{tab:panel_views} does on CIFAR-10H, and the next one the agreement between the tasks, as \cref{tab:view_agreement} does.

\begin{table}[htbp]
\centering
\scriptsize
\setlength{\tabcolsep}{3pt}
\caption{\textbf{Values of every task on Benthic (DCIC).} Fifteen methods, networks trained on the annotator distribution of each image, means over three seeds, the epistemic estimate as the score; rank in parentheses, best value per column in bold; rows by $n$-$\AuReC$ under the log loss. OOD detection: $\AuROC$ against the test images of the other DCIC datasets. AL: $\NAUC$, acquiring by the epistemic estimate (\cref{app:protocols}). Selective prediction and regret: $n$-$\AuRC$ and $n$-$\AuReC$ under the loss named, against the annotator distribution. References: the raw area $\AuReC_{\log}$ and the test accuracy against the majority label.}
\label{tab:results-benthic}
\resizebox{\textwidth}{!}{\begin{tabular}{lcccccccccc}
\toprule
 & OOD & AL & \multicolumn{3}{c}{$n$-$\AuRC$ $\downarrow$} & \multicolumn{3}{c}{$n$-$\AuReC$ $\downarrow$} & \multicolumn{2}{c}{references} \\
\cmidrule(lr){4-6} \cmidrule(lr){7-9} \cmidrule(lr){10-11}
Method & far & $\NAUC$ $\uparrow$ & log & Brier & 0/1 & log & Brier & 0/1 & $\AuReC_{\log}$ $\downarrow$ & Acc.\ $\uparrow$ \\
\midrule
MFVI-BNN & $0.573$\,{\tiny(13)} & $0.668$\,{\tiny(12)} & $\mathbf{0.432}$\,{\tiny(1)} & $0.391$\,{\tiny(8)} & $0.649$\,{\tiny(5)} & $\mathbf{0.434}$\,{\tiny(1)} & $0.390$\,{\tiny(2)} & $0.570$\,{\tiny(6)} & $0.3377$\,{\tiny(7)} & $0.739$\,{\tiny(12)} \\
Laplace (last layer) & $0.933$\,{\tiny(2)} & $0.703$\,{\tiny(8)} & $0.455$\,{\tiny(2)} & $0.366$\,{\tiny(2)} & $0.577$\,{\tiny(2)} & $0.499$\,{\tiny(2)} & $0.393$\,{\tiny(3)} & $0.497$\,{\tiny(2)} & $0.3207$\,{\tiny(3)} & $0.743$\,{\tiny(9)} \\
Deep Ensemble & $0.830$\,{\tiny(8)} & $\mathbf{0.752}$\,{\tiny(1)} & $0.461$\,{\tiny(4)} & $0.377$\,{\tiny(4)} & $0.610$\,{\tiny(3)} & $0.502$\,{\tiny(3)} & $0.395$\,{\tiny(4)} & $0.503$\,{\tiny(3)} & $\mathbf{0.2831}$\,{\tiny(1)} & $\mathbf{0.782}$\,{\tiny(1)} \\
BatchEnsemble & $0.889$\,{\tiny(3)} & $0.654$\,{\tiny(13)} & $0.511$\,{\tiny(10)} & $0.455$\,{\tiny(13)} & $0.620$\,{\tiny(4)} & $0.506$\,{\tiny(4)} & $0.454$\,{\tiny(11)} & $0.537$\,{\tiny(4)} & $0.3864$\,{\tiny(10)} & $0.721$\,{\tiny(14)} \\
MC-Dropout ($p{=}0.1$) & $0.727$\,{\tiny(11)} & $0.713$\,{\tiny(4)} & $0.457$\,{\tiny(3)} & $0.368$\,{\tiny(3)} & $0.828$\,{\tiny(12)} & $0.510$\,{\tiny(5)} & $0.408$\,{\tiny(6)} & $0.779$\,{\tiny(12)} & $0.3265$\,{\tiny(4)} & $0.742$\,{\tiny(10)} \\
SNGP & $\mathbf{0.939}$\,{\tiny(1)} & $0.565$\,{\tiny(15)} & $0.499$\,{\tiny(9)} & $0.438$\,{\tiny(11)} & $0.780$\,{\tiny(9)} & $0.522$\,{\tiny(6)} & $0.475$\,{\tiny(12)} & $0.716$\,{\tiny(9)} & $0.4327$\,{\tiny(13)} & $0.705$\,{\tiny(15)} \\
SWAG & $0.847$\,{\tiny(6)} & $0.735$\,{\tiny(2)} & $0.470$\,{\tiny(5)} & $\mathbf{0.345}$\,{\tiny(1)} & $0.653$\,{\tiny(6)} & $0.532$\,{\tiny(7)} & $\mathbf{0.375}$\,{\tiny(1)} & $0.548$\,{\tiny(5)} & $0.3175$\,{\tiny(2)} & $0.777$\,{\tiny(2)} \\
PostNet & $0.818$\,{\tiny(9)} & $0.640$\,{\tiny(14)} & $0.498$\,{\tiny(8)} & $0.379$\,{\tiny(5)} & $\mathbf{0.357}$\,{\tiny(1)} & $0.546$\,{\tiny(8)} & $0.406$\,{\tiny(5)} & $\mathbf{0.304}$\,{\tiny(1)} & $0.3803$\,{\tiny(8)} & $0.745$\,{\tiny(7)} \\
SubEnsemble & $0.838$\,{\tiny(7)} & $0.715$\,{\tiny(3)} & $0.493$\,{\tiny(7)} & $0.385$\,{\tiny(7)} & $0.840$\,{\tiny(13)} & $0.546$\,{\tiny(9)} & $0.423$\,{\tiny(8)} & $0.800$\,{\tiny(13)} & $0.3297$\,{\tiny(5)} & $0.769$\,{\tiny(3)} \\
DropConnect ($p{=}0.1$) & $0.652$\,{\tiny(12)} & $0.699$\,{\tiny(10)} & $0.489$\,{\tiny(6)} & $0.385$\,{\tiny(6)} & $0.792$\,{\tiny(11)} & $0.552$\,{\tiny(10)} & $0.437$\,{\tiny(9)} & $0.743$\,{\tiny(11)} & $0.3374$\,{\tiny(6)} & $0.743$\,{\tiny(8)} \\
MC-Dropout ($p{=}0.5$) & $0.772$\,{\tiny(10)} & $0.704$\,{\tiny(7)} & $0.527$\,{\tiny(11)} & $0.407$\,{\tiny(9)} & $0.772$\,{\tiny(8)} & $0.582$\,{\tiny(11)} & $0.453$\,{\tiny(10)} & $0.719$\,{\tiny(10)} & $0.3815$\,{\tiny(9)} & $0.741$\,{\tiny(11)} \\
DropConnect ($p{=}0.5$) & $0.547$\,{\tiny(14)} & $0.708$\,{\tiny(5)} & $0.551$\,{\tiny(12)} & $0.448$\,{\tiny(12)} & $0.714$\,{\tiny(7)} & $0.623$\,{\tiny(12)} & $0.521$\,{\tiny(13)} & $0.670$\,{\tiny(8)} & $0.4052$\,{\tiny(11)} & $0.735$\,{\tiny(13)} \\
Evidential & $0.860$\,{\tiny(5)} & $0.702$\,{\tiny(9)} & $0.757$\,{\tiny(13)} & $0.433$\,{\tiny(10)} & $0.783$\,{\tiny(10)} & $0.676$\,{\tiny(13)} & $0.409$\,{\tiny(7)} & $0.625$\,{\tiny(7)} & $0.4308$\,{\tiny(12)} & $0.755$\,{\tiny(4)} \\
DDU & $0.866$\,{\tiny(4)} & $0.689$\,{\tiny(11)} & $1.051$\,{\tiny(14)} & $1.042$\,{\tiny(14)} & $1.024$\,{\tiny(14)} & $1.023$\,{\tiny(14)} & $0.983$\,{\tiny(14)} & $0.934$\,{\tiny(14)} & $0.5587$\,{\tiny(14)} & $0.751$\,{\tiny(6)} \\
NatPN & $0.282$\,{\tiny(15)} & $0.705$\,{\tiny(6)} & $1.390$\,{\tiny(15)} & $1.137$\,{\tiny(15)} & $1.327$\,{\tiny(15)} & $1.430$\,{\tiny(15)} & $1.196$\,{\tiny(15)} & $1.386$\,{\tiny(15)} & $0.6774$\,{\tiny(15)} & $0.751$\,{\tiny(5)} \\
\bottomrule
\end{tabular}}\end{table}

\begin{table}[htbp]
\centering
\scriptsize
\setlength{\tabcolsep}{4pt}
\caption{\textbf{Rank agreement between the tasks on Benthic (DCIC).} Spearman correlation between the method rankings of two tasks over the fifteen methods of \cref{tab:results-benthic} ($+1$: identical rankings). Diagonal (italic): the agreement of two seeds of the same task, mean over the three pairs of seeds. Last two rows: the raw log-loss regret area and the test accuracy, as references.}
\label{tab:agreement_benthic}
\begin{tabular}{lcccccccccc}
\toprule
 & \rotatebox{90}{OOD detection, far} & \rotatebox{90}{Active learning} & \rotatebox{90}{Selective prediction, log} & \rotatebox{90}{Selective prediction, Brier} & \rotatebox{90}{Selective prediction, 0/1} & \rotatebox{90}{Regret, log} & \rotatebox{90}{Regret, Brier} & \rotatebox{90}{Regret, 0/1} & \rotatebox{90}{Regret, log, raw area} & \rotatebox{90}{Test accuracy} \\
\midrule
OOD detection, far & \textit{0.84} &  &  &  &  &  &  &  &  &  \\
Active learning & $-0.30$ & \textit{0.86} &  &  &  &  &  &  &  &  \\
Selective prediction, log & $0.03$ & $0.18$ & \textit{0.62} &  &  &  &  &  &  &  \\
Selective prediction, Brier & $0.09$ & $0.40$ & $0.82$ & \textit{0.71} &  &  &  &  &  &  \\
Selective prediction, 0/1 & $0.22$ & $-0.16$ & $0.49$ & $0.46$ & \textit{0.94} &  &  &  &  &  \\
Regret, log & $0.27$ & $-0.06$ & $0.87$ & $0.57$ & $0.67$ & \textit{0.71} &  &  &  &  \\
Regret, Brier & $0.10$ & $0.24$ & $0.83$ & $0.87$ & $0.60$ & $0.69$ & \textit{0.64} &  &  &  \\
Regret, 0/1 & $0.32$ & $-0.16$ & $0.45$ & $0.48$ & $0.97$ & $0.63$ & $0.65$ & \textit{0.87} &  &  \\
Regret, log, raw area & $0.03$ & $0.54$ & $0.85$ & $0.92$ & $0.46$ & $0.66$ & $0.82$ & $0.46$ & \textit{0.87} &  \\
Test accuracy & $-0.03$ & $0.60$ & $0.01$ & $0.34$ & $-0.16$ & $-0.24$ & $0.31$ & $-0.04$ & $0.35$ & \textit{0.69} \\
\bottomrule
\end{tabular}\end{table}

\begin{table}[htbp]
\centering
\scriptsize
\setlength{\tabcolsep}{3pt}
\caption{\textbf{Values of every task on MiceBone (DCIC).} Fifteen methods, networks trained on the annotator distribution of each image, means over three seeds, the epistemic estimate as the score; rank in parentheses, best value per column in bold; rows by $n$-$\AuReC$ under the log loss. OOD detection: $\AuROC$ against the test images of the other DCIC datasets. AL: $\NAUC$, acquiring by the epistemic estimate (\cref{app:protocols}). Selective prediction and regret: $n$-$\AuRC$ and $n$-$\AuReC$ under the loss named, against the annotator distribution. References: the raw area $\AuReC_{\log}$ and the test accuracy against the majority label.}
\label{tab:results-mice-bone}
\resizebox{\textwidth}{!}{\begin{tabular}{lcccccccccc}
\toprule
 & OOD & AL & \multicolumn{3}{c}{$n$-$\AuRC$ $\downarrow$} & \multicolumn{3}{c}{$n$-$\AuReC$ $\downarrow$} & \multicolumn{2}{c}{references} \\
\cmidrule(lr){4-6} \cmidrule(lr){7-9} \cmidrule(lr){10-11}
Method & far & $\NAUC$ $\uparrow$ & log & Brier & 0/1 & log & Brier & 0/1 & $\AuReC_{\log}$ $\downarrow$ & Acc.\ $\uparrow$ \\
\midrule
SNGP & $0.715$\,{\tiny(4)} & $0.779$\,{\tiny(15)} & $\mathbf{0.232}$\,{\tiny(1)} & $0.206$\,{\tiny(4)} & $0.664$\,{\tiny(6)} & $\mathbf{0.322}$\,{\tiny(1)} & $0.256$\,{\tiny(4)} & $0.508$\,{\tiny(6)} & $0.1041$\,{\tiny(3)} & $0.818$\,{\tiny(12)} \\
MFVI-BNN & $0.433$\,{\tiny(15)} & $0.813$\,{\tiny(9)} & $0.260$\,{\tiny(4)} & $0.213$\,{\tiny(6)} & $0.620$\,{\tiny(4)} & $0.323$\,{\tiny(2)} & $0.250$\,{\tiny(2)} & $0.474$\,{\tiny(4)} & $0.1041$\,{\tiny(2)} & $0.823$\,{\tiny(10)} \\
BatchEnsemble & $0.747$\,{\tiny(2)} & $0.812$\,{\tiny(10)} & $0.244$\,{\tiny(2)} & $0.205$\,{\tiny(3)} & $0.641$\,{\tiny(5)} & $0.326$\,{\tiny(3)} & $0.255$\,{\tiny(3)} & $0.489$\,{\tiny(5)} & $0.1067$\,{\tiny(5)} & $0.814$\,{\tiny(15)} \\
PostNet & $0.559$\,{\tiny(12)} & $0.816$\,{\tiny(6)} & $0.249$\,{\tiny(3)} & $0.196$\,{\tiny(2)} & $\mathbf{0.227}$\,{\tiny(1)} & $0.385$\,{\tiny(4)} & $0.266$\,{\tiny(5)} & $\mathbf{0.223}$\,{\tiny(1)} & $0.1077$\,{\tiny(6)} & $0.824$\,{\tiny(9)} \\
Deep Ensemble & $0.655$\,{\tiny(6)} & $\mathbf{0.832}$\,{\tiny(1)} & $0.276$\,{\tiny(5)} & $0.208$\,{\tiny(5)} & $0.615$\,{\tiny(3)} & $0.387$\,{\tiny(5)} & $0.271$\,{\tiny(6)} & $0.442$\,{\tiny(3)} & $\mathbf{0.0870}$\,{\tiny(1)} & $\mathbf{0.854}$\,{\tiny(1)} \\
Evidential & $0.620$\,{\tiny(8)} & $0.824$\,{\tiny(4)} & $0.324$\,{\tiny(7)} & $\mathbf{0.186}$\,{\tiny(1)} & $0.361$\,{\tiny(2)} & $0.387$\,{\tiny(6)} & $\mathbf{0.239}$\,{\tiny(1)} & $0.341$\,{\tiny(2)} & $0.1066$\,{\tiny(4)} & $0.843$\,{\tiny(2)} \\
Laplace (last layer) & $0.723$\,{\tiny(3)} & $0.815$\,{\tiny(7)} & $0.286$\,{\tiny(6)} & $0.228$\,{\tiny(8)} & $0.707$\,{\tiny(7)} & $0.417$\,{\tiny(7)} & $0.298$\,{\tiny(8)} & $0.569$\,{\tiny(7)} & $0.1167$\,{\tiny(7)} & $0.818$\,{\tiny(11)} \\
DropConnect ($p{=}0.5$) & $0.529$\,{\tiny(14)} & $0.814$\,{\tiny(8)} & $0.336$\,{\tiny(8)} & $0.236$\,{\tiny(10)} & $0.814$\,{\tiny(9)} & $0.487$\,{\tiny(8)} & $0.318$\,{\tiny(11)} & $0.713$\,{\tiny(9)} & $0.1218$\,{\tiny(9)} & $0.834$\,{\tiny(7)} \\
SWAG & $0.625$\,{\tiny(7)} & $0.825$\,{\tiny(3)} & $0.348$\,{\tiny(9)} & $0.217$\,{\tiny(7)} & $0.780$\,{\tiny(8)} & $0.491$\,{\tiny(9)} & $0.280$\,{\tiny(7)} & $0.650$\,{\tiny(8)} & $0.1219$\,{\tiny(10)} & $0.838$\,{\tiny(3)} \\
MC-Dropout ($p{=}0.1$) & $0.584$\,{\tiny(10)} & $0.811$\,{\tiny(11)} & $0.425$\,{\tiny(11)} & $0.260$\,{\tiny(11)} & $0.883$\,{\tiny(13)} & $0.515$\,{\tiny(10)} & $0.308$\,{\tiny(9)} & $0.826$\,{\tiny(13)} & $0.1355$\,{\tiny(11)} & $0.816$\,{\tiny(13)} \\
SubEnsemble & $0.607$\,{\tiny(9)} & $0.830$\,{\tiny(2)} & $0.357$\,{\tiny(10)} & $0.234$\,{\tiny(9)} & $0.871$\,{\tiny(12)} & $0.517$\,{\tiny(11)} & $0.324$\,{\tiny(12)} & $0.789$\,{\tiny(11)} & $0.1187$\,{\tiny(8)} & $0.838$\,{\tiny(4)} \\
MC-Dropout ($p{=}0.5$) & $0.575$\,{\tiny(11)} & $0.816$\,{\tiny(5)} & $0.452$\,{\tiny(12)} & $0.276$\,{\tiny(12)} & $0.854$\,{\tiny(10)} & $0.536$\,{\tiny(12)} & $0.318$\,{\tiny(10)} & $0.761$\,{\tiny(10)} & $0.1356$\,{\tiny(12)} & $0.835$\,{\tiny(5)} \\
DropConnect ($p{=}0.1$) & $0.550$\,{\tiny(13)} & $0.809$\,{\tiny(12)} & $0.516$\,{\tiny(13)} & $0.296$\,{\tiny(13)} & $0.866$\,{\tiny(11)} & $0.616$\,{\tiny(13)} & $0.365$\,{\tiny(13)} & $0.799$\,{\tiny(12)} & $0.1481$\,{\tiny(13)} & $0.816$\,{\tiny(14)} \\
NatPN & $0.679$\,{\tiny(5)} & $0.791$\,{\tiny(14)} & $1.054$\,{\tiny(14)} & $0.726$\,{\tiny(14)} & $0.974$\,{\tiny(14)} & $1.071$\,{\tiny(14)} & $0.712$\,{\tiny(14)} & $0.902$\,{\tiny(14)} & $0.2006$\,{\tiny(14)} & $0.834$\,{\tiny(8)} \\
DDU & $\mathbf{0.974}$\,{\tiny(1)} & $0.806$\,{\tiny(13)} & $1.313$\,{\tiny(15)} & $1.337$\,{\tiny(15)} & $1.344$\,{\tiny(15)} & $1.184$\,{\tiny(15)} & $1.243$\,{\tiny(15)} & $1.273$\,{\tiny(15)} & $0.2209$\,{\tiny(15)} & $0.834$\,{\tiny(6)} \\
\bottomrule
\end{tabular}}\end{table}

\begin{table}[htbp]
\centering
\scriptsize
\setlength{\tabcolsep}{4pt}
\caption{\textbf{Rank agreement between the tasks on MiceBone (DCIC).} Spearman correlation between the method rankings of two tasks over the fifteen methods of \cref{tab:results-mice-bone} ($+1$: identical rankings). Diagonal (italic): the agreement of two seeds of the same task, mean over the three pairs of seeds. Last two rows: the raw log-loss regret area and the test accuracy, as references.}
\begin{tabular}{lcccccccccc}
\toprule
 & \rotatebox{90}{OOD detection, far} & \rotatebox{90}{Active learning} & \rotatebox{90}{Selective prediction, log} & \rotatebox{90}{Selective prediction, Brier} & \rotatebox{90}{Selective prediction, 0/1} & \rotatebox{90}{Regret, log} & \rotatebox{90}{Regret, Brier} & \rotatebox{90}{Regret, 0/1} & \rotatebox{90}{Regret, log, raw area} & \rotatebox{90}{Test accuracy} \\
\midrule
OOD detection, far & \textit{0.64} &  &  &  &  &  &  &  &  &  \\
Active learning & $-0.21$ & \textit{0.70} &  &  &  &  &  &  &  &  \\
Selective prediction, log & $0.03$ & $0.15$ & \textit{0.85} &  &  &  &  &  &  &  \\
Selective prediction, Brier & $0.01$ & $0.39$ & $0.89$ & \textit{0.65} &  &  &  &  &  &  \\
Selective prediction, 0/1 & $-0.15$ & $0.42$ & $0.85$ & $0.92$ & \textit{0.82} &  &  &  &  &  \\
Regret, log & $-0.05$ & $0.12$ & $0.98$ & $0.88$ & $0.85$ & \textit{0.90} &  &  &  &  \\
Regret, Brier & $-0.05$ & $0.26$ & $0.85$ & $0.92$ & $0.89$ & $0.90$ & \textit{0.66} &  &  &  \\
Regret, 0/1 & $-0.13$ & $0.46$ & $0.86$ & $0.94$ & $1.00$ & $0.86$ & $0.89$ & \textit{0.86} &  &  \\
Regret, log, raw area & $-0.05$ & $0.36$ & $0.90$ & $0.87$ & $0.86$ & $0.92$ & $0.86$ & $0.88$ & \textit{0.68} &  \\
Test accuracy & $0.00$ & $0.71$ & $-0.20$ & $0.07$ & $0.11$ & $-0.19$ & $-0.03$ & $0.14$ & $0.08$ & \textit{0.45} \\
\bottomrule
\end{tabular}\end{table}

\begin{table}[htbp]
\centering
\scriptsize
\setlength{\tabcolsep}{3pt}
\caption{\textbf{Values of every task on Pig (DCIC).} Fifteen methods, networks trained on the annotator distribution of each image, means over three seeds, the epistemic estimate as the score; rank in parentheses, best value per column in bold; rows by $n$-$\AuReC$ under the log loss. OOD detection: $\AuROC$ against the test images of the other DCIC datasets. AL: $\NAUC$, acquiring by the epistemic estimate (\cref{app:protocols}). Selective prediction and regret: $n$-$\AuRC$ and $n$-$\AuReC$ under the loss named, against the annotator distribution. References: the raw area $\AuReC_{\log}$ and the test accuracy against the majority label.}
\label{tab:results-pig}
\resizebox{\textwidth}{!}{\begin{tabular}{lcccccccccc}
\toprule
 & OOD & AL & \multicolumn{3}{c}{$n$-$\AuRC$ $\downarrow$} & \multicolumn{3}{c}{$n$-$\AuReC$ $\downarrow$} & \multicolumn{2}{c}{references} \\
\cmidrule(lr){4-6} \cmidrule(lr){7-9} \cmidrule(lr){10-11}
Method & far & $\NAUC$ $\uparrow$ & log & Brier & 0/1 & log & Brier & 0/1 & $\AuReC_{\log}$ $\downarrow$ & Acc.\ $\uparrow$ \\
\midrule
Laplace (last layer) & $\mathbf{0.912}$\,{\tiny(1)} & $0.569$\,{\tiny(9)} & $0.585$\,{\tiny(2)} & $0.512$\,{\tiny(3)} & $0.749$\,{\tiny(5)} & $\mathbf{0.752}$\,{\tiny(1)} & $0.677$\,{\tiny(3)} & $0.714$\,{\tiny(5)} & $0.3296$\,{\tiny(3)} & $0.613$\,{\tiny(10)} \\
PostNet & $0.559$\,{\tiny(10)} & $0.579$\,{\tiny(4)} & $\mathbf{0.493}$\,{\tiny(1)} & $\mathbf{0.409}$\,{\tiny(1)} & $\mathbf{0.475}$\,{\tiny(1)} & $0.769$\,{\tiny(2)} & $\mathbf{0.651}$\,{\tiny(1)} & $\mathbf{0.536}$\,{\tiny(1)} & $0.3278$\,{\tiny(2)} & $0.635$\,{\tiny(4)} \\
BatchEnsemble & $0.752$\,{\tiny(4)} & $0.522$\,{\tiny(14)} & $0.723$\,{\tiny(6)} & $0.649$\,{\tiny(11)} & $0.714$\,{\tiny(4)} & $0.776$\,{\tiny(3)} & $0.730$\,{\tiny(4)} & $0.679$\,{\tiny(4)} & $0.3468$\,{\tiny(6)} & $0.598$\,{\tiny(12)} \\
MFVI-BNN & $0.369$\,{\tiny(14)} & $0.548$\,{\tiny(12)} & $0.752$\,{\tiny(8)} & $0.642$\,{\tiny(9)} & $0.786$\,{\tiny(6)} & $0.812$\,{\tiny(4)} & $0.742$\,{\tiny(7)} & $0.752$\,{\tiny(7)} & $0.3662$\,{\tiny(11)} & $0.591$\,{\tiny(14)} \\
Deep Ensemble & $0.551$\,{\tiny(11)} & $\mathbf{0.615}$\,{\tiny(1)} & $0.704$\,{\tiny(5)} & $0.606$\,{\tiny(6)} & $0.667$\,{\tiny(2)} & $0.818$\,{\tiny(5)} & $0.732$\,{\tiny(6)} & $0.621$\,{\tiny(2)} & $\mathbf{0.3171}$\,{\tiny(1)} & $0.638$\,{\tiny(2)} \\
Evidential & $0.512$\,{\tiny(12)} & $0.552$\,{\tiny(10)} & $0.837$\,{\tiny(10)} & $0.417$\,{\tiny(2)} & $0.894$\,{\tiny(13)} & $0.822$\,{\tiny(6)} & $0.675$\,{\tiny(2)} & $0.860$\,{\tiny(13)} & $0.3444$\,{\tiny(4)} & $0.605$\,{\tiny(11)} \\
SNGP & $0.906$\,{\tiny(2)} & $0.496$\,{\tiny(15)} & $0.675$\,{\tiny(4)} & $0.614$\,{\tiny(8)} & $0.806$\,{\tiny(9)} & $0.830$\,{\tiny(7)} & $0.780$\,{\tiny(11)} & $0.782$\,{\tiny(9)} & $0.3844$\,{\tiny(12)} & $0.581$\,{\tiny(15)} \\
MC-Dropout ($p{=}0.1$) & $0.592$\,{\tiny(7)} & $0.572$\,{\tiny(6)} & $0.627$\,{\tiny(3)} & $0.538$\,{\tiny(4)} & $0.807$\,{\tiny(10)} & $0.831$\,{\tiny(8)} & $0.744$\,{\tiny(8)} & $0.785$\,{\tiny(10)} & $0.3481$\,{\tiny(8)} & $0.614$\,{\tiny(9)} \\
DDU & $0.877$\,{\tiny(3)} & $0.541$\,{\tiny(13)} & $0.966$\,{\tiny(12)} & $0.975$\,{\tiny(14)} & $0.976$\,{\tiny(14)} & $0.850$\,{\tiny(9)} & $0.848$\,{\tiny(14)} & $0.866$\,{\tiny(14)} & $0.3859$\,{\tiny(13)} & $0.594$\,{\tiny(13)} \\
DropConnect ($p{=}0.1$) & $0.566$\,{\tiny(9)} & $0.551$\,{\tiny(11)} & $0.726$\,{\tiny(7)} & $0.595$\,{\tiny(5)} & $0.805$\,{\tiny(8)} & $0.863$\,{\tiny(10)} & $0.773$\,{\tiny(10)} & $0.778$\,{\tiny(8)} & $0.3489$\,{\tiny(9)} & $0.614$\,{\tiny(8)} \\
SWAG & $0.592$\,{\tiny(8)} & $0.591$\,{\tiny(2)} & $0.879$\,{\tiny(11)} & $0.646$\,{\tiny(10)} & $0.686$\,{\tiny(3)} & $0.869$\,{\tiny(11)} & $0.730$\,{\tiny(5)} & $0.641$\,{\tiny(3)} & $0.3476$\,{\tiny(7)} & $0.637$\,{\tiny(3)} \\
SubEnsemble & $0.717$\,{\tiny(5)} & $0.581$\,{\tiny(3)} & $0.815$\,{\tiny(9)} & $0.608$\,{\tiny(7)} & $0.883$\,{\tiny(12)} & $0.898$\,{\tiny(12)} & $0.759$\,{\tiny(9)} & $0.855$\,{\tiny(12)} & $0.3459$\,{\tiny(5)} & $\mathbf{0.639}$\,{\tiny(1)} \\
DropConnect ($p{=}0.5$) & $0.354$\,{\tiny(15)} & $0.570$\,{\tiny(8)} & $0.997$\,{\tiny(13)} & $0.706$\,{\tiny(12)} & $0.790$\,{\tiny(7)} & $0.971$\,{\tiny(13)} & $0.807$\,{\tiny(12)} & $0.737$\,{\tiny(6)} & $0.3590$\,{\tiny(10)} & $0.627$\,{\tiny(5)} \\
MC-Dropout ($p{=}0.5$) & $0.414$\,{\tiny(13)} & $0.572$\,{\tiny(7)} & $1.137$\,{\tiny(14)} & $0.786$\,{\tiny(13)} & $0.827$\,{\tiny(11)} & $1.026$\,{\tiny(14)} & $0.818$\,{\tiny(13)} & $0.792$\,{\tiny(11)} & $0.3859$\,{\tiny(14)} & $0.620$\,{\tiny(7)} \\
NatPN & $0.594$\,{\tiny(6)} & $0.578$\,{\tiny(5)} & $1.219$\,{\tiny(15)} & $1.164$\,{\tiny(15)} & $1.174$\,{\tiny(15)} & $1.128$\,{\tiny(15)} & $1.100$\,{\tiny(15)} & $1.135$\,{\tiny(15)} & $0.4076$\,{\tiny(15)} & $0.625$\,{\tiny(6)} \\
\bottomrule
\end{tabular}}\end{table}

\begin{table}[htbp]
\centering
\scriptsize
\setlength{\tabcolsep}{4pt}
\caption{\textbf{Rank agreement between the tasks on Pig (DCIC).} Spearman correlation between the method rankings of two tasks over the fifteen methods of \cref{tab:results-pig} ($+1$: identical rankings). Diagonal (italic): the agreement of two seeds of the same task, mean over the three pairs of seeds. Last two rows: the raw log-loss regret area and the test accuracy, as references.}
\begin{tabular}{lcccccccccc}
\toprule
 & \rotatebox{90}{OOD detection, far} & \rotatebox{90}{Active learning} & \rotatebox{90}{Selective prediction, log} & \rotatebox{90}{Selective prediction, Brier} & \rotatebox{90}{Selective prediction, 0/1} & \rotatebox{90}{Regret, log} & \rotatebox{90}{Regret, Brier} & \rotatebox{90}{Regret, 0/1} & \rotatebox{90}{Regret, log, raw area} & \rotatebox{90}{Test accuracy} \\
\midrule
OOD detection, far & \textit{0.80} &  &  &  &  &  &  &  &  &  \\
Active learning & $-0.26$ & \textit{0.88} &  &  &  &  &  &  &  &  \\
Selective prediction, log & $0.32$ & $-0.06$ & \textit{0.84} &  &  &  &  &  &  &  \\
Selective prediction, Brier & $-0.01$ & $0.15$ & $0.78$ & \textit{0.76} &  &  &  &  &  &  \\
Selective prediction, 0/1 & $-0.15$ & $0.23$ & $0.59$ & $0.38$ & \textit{0.92} &  &  &  &  &  \\
Regret, log & $0.21$ & $-0.26$ & $0.81$ & $0.63$ & $0.60$ & \textit{0.56} &  &  &  &  \\
Regret, Brier & $-0.03$ & $0.15$ & $0.65$ & $0.76$ & $0.65$ & $0.79$ & \textit{0.61} &  &  &  \\
Regret, 0/1 & $-0.16$ & $0.24$ & $0.57$ & $0.37$ & $1.00$ & $0.56$ & $0.63$ & \textit{0.89} &  &  \\
Regret, log, raw area & $0.01$ & $0.39$ & $0.65$ & $0.77$ & $0.62$ & $0.65$ & $0.86$ & $0.62$ & \textit{0.63} &  \\
Test accuracy & $-0.27$ & $0.92$ & $-0.14$ & $0.08$ & $0.24$ & $-0.36$ & $0.06$ & $0.27$ & $0.38$ & \textit{0.61} \\
\bottomrule
\end{tabular}\end{table}

\begin{table}[htbp]
\centering
\scriptsize
\setlength{\tabcolsep}{3pt}
\caption{\textbf{Values of every task on Plankton (DCIC).} Fifteen methods, networks trained on the annotator distribution of each image, means over three seeds, the epistemic estimate as the score; rank in parentheses, best value per column in bold; rows by $n$-$\AuReC$ under the log loss. OOD detection: $\AuROC$ against the test images of the other DCIC datasets. AL: $\NAUC$, acquiring by the epistemic estimate (\cref{app:protocols}). Selective prediction and regret: $n$-$\AuRC$ and $n$-$\AuReC$ under the loss named, against the annotator distribution. References: the raw area $\AuReC_{\log}$ and the test accuracy against the majority label.}
\label{tab:results-plankton}
\resizebox{\textwidth}{!}{\begin{tabular}{lcccccccccc}
\toprule
 & OOD & AL & \multicolumn{3}{c}{$n$-$\AuRC$ $\downarrow$} & \multicolumn{3}{c}{$n$-$\AuReC$ $\downarrow$} & \multicolumn{2}{c}{references} \\
\cmidrule(lr){4-6} \cmidrule(lr){7-9} \cmidrule(lr){10-11}
Method & far & $\NAUC$ $\uparrow$ & log & Brier & 0/1 & log & Brier & 0/1 & $\AuReC_{\log}$ $\downarrow$ & Acc.\ $\uparrow$ \\
\midrule
SNGP & $0.990$\,{\tiny(2)} & $0.846$\,{\tiny(13)} & $0.233$\,{\tiny(3)} & $0.169$\,{\tiny(9)} & $0.678$\,{\tiny(7)} & $\mathbf{0.280}$\,{\tiny(1)} & $0.179$\,{\tiny(12)} & $0.558$\,{\tiny(7)} & $0.0875$\,{\tiny(11)} & $0.906$\,{\tiny(14)} \\
BatchEnsemble & $0.540$\,{\tiny(9)} & $0.906$\,{\tiny(10)} & $0.254$\,{\tiny(8)} & $0.178$\,{\tiny(11)} & $0.576$\,{\tiny(5)} & $0.280$\,{\tiny(2)} & $0.149$\,{\tiny(3)} & $0.413$\,{\tiny(3)} & $0.0736$\,{\tiny(6)} & $0.919$\,{\tiny(11)} \\
Deep Ensemble & $0.838$\,{\tiny(3)} & $\mathbf{0.931}$\,{\tiny(1)} & $0.232$\,{\tiny(2)} & $0.158$\,{\tiny(4)} & $0.614$\,{\tiny(6)} & $0.286$\,{\tiny(3)} & $\mathbf{0.141}$\,{\tiny(1)} & $0.438$\,{\tiny(6)} & $\mathbf{0.0559}$\,{\tiny(1)} & $\mathbf{0.939}$\,{\tiny(1)} \\
DropConnect ($p{=}0.1$) & $0.524$\,{\tiny(12)} & $0.911$\,{\tiny(9)} & $\mathbf{0.228}$\,{\tiny(1)} & $0.165$\,{\tiny(7)} & $0.819$\,{\tiny(12)} & $0.286$\,{\tiny(4)} & $0.167$\,{\tiny(10)} & $0.737$\,{\tiny(13)} & $0.0702$\,{\tiny(3)} & $0.926$\,{\tiny(8)} \\
SubEnsemble & $0.333$\,{\tiny(14)} & $0.919$\,{\tiny(3)} & $0.237$\,{\tiny(5)} & $0.155$\,{\tiny(2)} & $0.786$\,{\tiny(11)} & $0.298$\,{\tiny(5)} & $0.152$\,{\tiny(5)} & $0.670$\,{\tiny(11)} & $0.0681$\,{\tiny(2)} & $0.926$\,{\tiny(6)} \\
Laplace (last layer) & $0.612$\,{\tiny(7)} & $0.916$\,{\tiny(6)} & $0.235$\,{\tiny(4)} & $0.164$\,{\tiny(6)} & $0.570$\,{\tiny(4)} & $0.302$\,{\tiny(6)} & $0.169$\,{\tiny(11)} & $0.420$\,{\tiny(4)} & $0.0729$\,{\tiny(4)} & $0.925$\,{\tiny(10)} \\
PostNet & $0.751$\,{\tiny(5)} & $0.811$\,{\tiny(15)} & $0.250$\,{\tiny(7)} & $0.156$\,{\tiny(3)} & $\mathbf{0.146}$\,{\tiny(1)} & $0.308$\,{\tiny(7)} & $0.152$\,{\tiny(4)} & $\mathbf{0.105}$\,{\tiny(1)} & $0.0861$\,{\tiny(10)} & $0.914$\,{\tiny(12)} \\
MC-Dropout ($p{=}0.1$) & $0.534$\,{\tiny(10)} & $0.913$\,{\tiny(8)} & $0.239$\,{\tiny(6)} & $0.163$\,{\tiny(5)} & $0.880$\,{\tiny(14)} & $0.309$\,{\tiny(8)} & $0.166$\,{\tiny(9)} & $0.825$\,{\tiny(14)} & $0.0733$\,{\tiny(5)} & $0.925$\,{\tiny(9)} \\
MFVI-BNN & $0.179$\,{\tiny(15)} & $0.826$\,{\tiny(14)} & $0.341$\,{\tiny(12)} & $0.285$\,{\tiny(13)} & $0.522$\,{\tiny(3)} & $0.316$\,{\tiny(9)} & $0.251$\,{\tiny(13)} & $0.429$\,{\tiny(5)} & $0.1687$\,{\tiny(14)} & $0.859$\,{\tiny(15)} \\
MC-Dropout ($p{=}0.5$) & $0.562$\,{\tiny(8)} & $0.918$\,{\tiny(5)} & $0.257$\,{\tiny(10)} & $0.167$\,{\tiny(8)} & $0.726$\,{\tiny(8)} & $0.319$\,{\tiny(10)} & $0.156$\,{\tiny(6)} & $0.574$\,{\tiny(8)} & $0.0750$\,{\tiny(7)} & $0.926$\,{\tiny(7)} \\
DropConnect ($p{=}0.5$) & $0.635$\,{\tiny(6)} & $0.918$\,{\tiny(4)} & $0.255$\,{\tiny(9)} & $0.170$\,{\tiny(10)} & $0.737$\,{\tiny(9)} & $0.322$\,{\tiny(11)} & $0.161$\,{\tiny(7)} & $0.618$\,{\tiny(10)} & $0.0770$\,{\tiny(8)} & $0.927$\,{\tiny(5)} \\
Evidential & $0.478$\,{\tiny(13)} & $0.915$\,{\tiny(7)} & $0.354$\,{\tiny(13)} & $0.181$\,{\tiny(12)} & $0.372$\,{\tiny(2)} & $0.342$\,{\tiny(12)} & $0.162$\,{\tiny(8)} & $0.309$\,{\tiny(2)} & $0.1371$\,{\tiny(13)} & $0.930$\,{\tiny(3)} \\
SWAG & $0.526$\,{\tiny(11)} & $0.923$\,{\tiny(2)} & $0.262$\,{\tiny(11)} & $\mathbf{0.155}$\,{\tiny(1)} & $0.824$\,{\tiny(13)} & $0.343$\,{\tiny(13)} & $0.147$\,{\tiny(2)} & $0.722$\,{\tiny(12)} & $0.0782$\,{\tiny(9)} & $0.933$\,{\tiny(2)} \\
DDU & $\mathbf{0.999}$\,{\tiny(1)} & $0.899$\,{\tiny(11)} & $0.770$\,{\tiny(14)} & $0.769$\,{\tiny(15)} & $0.740$\,{\tiny(10)} & $0.727$\,{\tiny(14)} & $0.677$\,{\tiny(15)} & $0.609$\,{\tiny(9)} & $0.1323$\,{\tiny(12)} & $0.930$\,{\tiny(4)} \\
NatPN & $0.816$\,{\tiny(4)} & $0.867$\,{\tiny(12)} & $0.953$\,{\tiny(15)} & $0.390$\,{\tiny(14)} & $0.920$\,{\tiny(15)} & $0.972$\,{\tiny(15)} & $0.413$\,{\tiny(14)} & $0.883$\,{\tiny(15)} & $0.2060$\,{\tiny(15)} & $0.914$\,{\tiny(13)} \\
\bottomrule
\end{tabular}}\end{table}

\begin{table}[htbp]
\centering
\scriptsize
\setlength{\tabcolsep}{4pt}
\caption{\textbf{Rank agreement between the tasks on Plankton (DCIC).} Spearman correlation between the method rankings of two tasks over the fifteen methods of \cref{tab:results-plankton} ($+1$: identical rankings). Diagonal (italic): the agreement of two seeds of the same task, mean over the three pairs of seeds. Last two rows: the raw log-loss regret area and the test accuracy, as references.}
\begin{tabular}{lcccccccccc}
\toprule
 & \rotatebox{90}{OOD detection, far} & \rotatebox{90}{Active learning} & \rotatebox{90}{Selective prediction, log} & \rotatebox{90}{Selective prediction, Brier} & \rotatebox{90}{Selective prediction, 0/1} & \rotatebox{90}{Regret, log} & \rotatebox{90}{Regret, Brier} & \rotatebox{90}{Regret, 0/1} & \rotatebox{90}{Regret, log, raw area} & \rotatebox{90}{Test accuracy} \\
\midrule
OOD detection, far & \textit{0.74} &  &  &  &  &  &  &  &  &  \\
Active learning & $-0.17$ & \textit{0.92} &  &  &  &  &  &  &  &  \\
Selective prediction, log & $0.02$ & $0.22$ & \textit{0.60} &  &  &  &  &  &  &  \\
Selective prediction, Brier & $-0.20$ & $0.51$ & $0.60$ & \textit{0.69} &  &  &  &  &  &  \\
Selective prediction, 0/1 & $-0.04$ & $-0.26$ & $0.06$ & $-0.09$ & \textit{0.91} &  &  &  &  &  \\
Regret, log & $-0.01$ & $-0.01$ & $0.86$ & $0.36$ & $0.30$ & \textit{0.41} &  &  &  &  \\
Regret, Brier & $-0.17$ & $0.58$ & $0.31$ & $0.69$ & $0.16$ & $0.31$ & \textit{0.58} &  &  &  \\
Regret, 0/1 & $0.03$ & $-0.23$ & $0.04$ & $-0.07$ & $0.98$ & $0.31$ & $0.24$ & \textit{0.91} &  &  \\
Regret, log, raw area & $-0.12$ & $0.62$ & $0.80$ & $0.66$ & $-0.12$ & $0.61$ & $0.54$ & $-0.10$ & \textit{0.80} &  \\
Test accuracy & $-0.01$ & $0.80$ & $-0.02$ & $0.29$ & $-0.17$ & $-0.25$ & $0.47$ & $-0.12$ & $0.36$ & \textit{0.71} \\
\bottomrule
\end{tabular}\end{table}

\begin{table}[htbp]
\centering
\scriptsize
\setlength{\tabcolsep}{3pt}
\caption{\textbf{Values of every task on QualityMRI (DCIC).} Fifteen methods, networks trained on the annotator distribution of each image, means over three seeds, the epistemic estimate as the score; rank in parentheses, best value per column in bold; rows by $n$-$\AuReC$ under the log loss. OOD detection: $\AuROC$ against the test images of the other DCIC datasets. AL: $\NAUC$, acquiring by the epistemic estimate (\cref{app:protocols}). QualityMRI has no active learning (\cref{app:protocols}). Selective prediction and regret: $n$-$\AuRC$ and $n$-$\AuReC$ under the loss named, against the annotator distribution. References: the raw area $\AuReC_{\log}$ and the test accuracy against the majority label.}
\label{tab:results-quality-mri}
\resizebox{\textwidth}{!}{\begin{tabular}{lccccccccc}
\toprule
 & OOD & \multicolumn{3}{c}{$n$-$\AuRC$ $\downarrow$} & \multicolumn{3}{c}{$n$-$\AuReC$ $\downarrow$} & \multicolumn{2}{c}{references} \\
\cmidrule(lr){3-5} \cmidrule(lr){6-8} \cmidrule(lr){9-10}
Method & far & log & Brier & 0/1 & log & Brier & 0/1 & $\AuReC_{\log}$ $\downarrow$ & Acc.\ $\uparrow$ \\
\midrule
PostNet & $0.705$\,{\tiny(6)} & $\mathbf{0.443}$\,{\tiny(1)} & $0.341$\,{\tiny(2)} & $0.413$\,{\tiny(2)} & $\mathbf{0.621}$\,{\tiny(1)} & $\mathbf{0.452}$\,{\tiny(1)} & $0.220$\,{\tiny(2)} & $0.0352$\,{\tiny(11)} & $0.803$\,{\tiny(14)} \\
MC-Dropout ($p{=}0.5$) & $0.543$\,{\tiny(8)} & $0.959$\,{\tiny(9)} & $0.565$\,{\tiny(7)} & $0.717$\,{\tiny(7)} & $0.655$\,{\tiny(2)} & $0.592$\,{\tiny(2)} & $0.408$\,{\tiny(7)} & $\mathbf{0.0155}$\,{\tiny(1)} & $\mathbf{0.920}$\,{\tiny(1)} \\
DropConnect ($p{=}0.1$) & $0.415$\,{\tiny(12)} & $1.105$\,{\tiny(13)} & $0.550$\,{\tiny(6)} & $0.681$\,{\tiny(4)} & $0.732$\,{\tiny(3)} & $0.716$\,{\tiny(4)} & $0.286$\,{\tiny(3)} & $0.0242$\,{\tiny(5)} & $0.883$\,{\tiny(8)} \\
NatPN & $0.505$\,{\tiny(10)} & $1.366$\,{\tiny(15)} & $1.299$\,{\tiny(15)} & $1.355$\,{\tiny(15)} & $0.770$\,{\tiny(4)} & $0.918$\,{\tiny(13)} & $1.074$\,{\tiny(15)} & $0.0229$\,{\tiny(3)} & $0.878$\,{\tiny(9)} \\
Evidential & $0.536$\,{\tiny(9)} & $0.915$\,{\tiny(7)} & $\mathbf{0.254}$\,{\tiny(1)} & $0.935$\,{\tiny(14)} & $0.770$\,{\tiny(5)} & $0.838$\,{\tiny(10)} & $0.677$\,{\tiny(11)} & $0.0264$\,{\tiny(7)} & $0.878$\,{\tiny(10)} \\
DropConnect ($p{=}0.5$) & $0.493$\,{\tiny(11)} & $1.170$\,{\tiny(14)} & $0.607$\,{\tiny(10)} & $0.696$\,{\tiny(5)} & $0.808$\,{\tiny(6)} & $0.835$\,{\tiny(9)} & $0.397$\,{\tiny(6)} & $0.0279$\,{\tiny(10)} & $0.904$\,{\tiny(3)} \\
Laplace (last layer) & $0.982$\,{\tiny(2)} & $0.645$\,{\tiny(2)} & $0.382$\,{\tiny(3)} & $\mathbf{0.270}$\,{\tiny(1)} & $0.822$\,{\tiny(7)} & $0.668$\,{\tiny(3)} & $\mathbf{0.178}$\,{\tiny(1)} & $0.0265$\,{\tiny(8)} & $0.888$\,{\tiny(7)} \\
MC-Dropout ($p{=}0.1$) & $0.409$\,{\tiny(13)} & $1.042$\,{\tiny(12)} & $0.592$\,{\tiny(9)} & $0.736$\,{\tiny(8)} & $0.848$\,{\tiny(8)} & $0.779$\,{\tiny(6)} & $0.709$\,{\tiny(12)} & $0.0250$\,{\tiny(6)} & $0.894$\,{\tiny(5)} \\
SubEnsemble & $0.868$\,{\tiny(4)} & $0.929$\,{\tiny(8)} & $0.524$\,{\tiny(5)} & $0.864$\,{\tiny(12)} & $0.871$\,{\tiny(9)} & $0.740$\,{\tiny(5)} & $0.647$\,{\tiny(10)} & $0.0215$\,{\tiny(2)} & $0.899$\,{\tiny(4)} \\
BatchEnsemble & $0.923$\,{\tiny(3)} & $1.015$\,{\tiny(11)} & $0.715$\,{\tiny(11)} & $0.770$\,{\tiny(9)} & $0.903$\,{\tiny(10)} & $0.859$\,{\tiny(11)} & $0.453$\,{\tiny(8)} & $0.0235$\,{\tiny(4)} & $0.894$\,{\tiny(6)} \\
DDU & $\mathbf{0.988}$\,{\tiny(1)} & $1.006$\,{\tiny(10)} & $0.988$\,{\tiny(14)} & $0.878$\,{\tiny(13)} & $1.010$\,{\tiny(11)} & $0.986$\,{\tiny(14)} & $0.642$\,{\tiny(9)} & $0.0485$\,{\tiny(13)} & $0.846$\,{\tiny(11)} \\
Deep Ensemble & $0.750$\,{\tiny(5)} & $0.777$\,{\tiny(3)} & $0.589$\,{\tiny(8)} & $0.697$\,{\tiny(6)} & $1.011$\,{\tiny(12)} & $0.834$\,{\tiny(8)} & $0.292$\,{\tiny(4)} & $0.0365$\,{\tiny(12)} & $0.835$\,{\tiny(13)} \\
MFVI-BNN & $0.228$\,{\tiny(15)} & $0.907$\,{\tiny(6)} & $0.722$\,{\tiny(12)} & $0.863$\,{\tiny(11)} & $1.036$\,{\tiny(13)} & $0.900$\,{\tiny(12)} & $0.841$\,{\tiny(14)} & $0.0593$\,{\tiny(14)} & $0.841$\,{\tiny(12)} \\
SWAG & $0.243$\,{\tiny(14)} & $0.780$\,{\tiny(4)} & $0.417$\,{\tiny(4)} & $0.659$\,{\tiny(3)} & $1.046$\,{\tiny(14)} & $0.789$\,{\tiny(7)} & $0.370$\,{\tiny(5)} & $0.0276$\,{\tiny(9)} & $0.905$\,{\tiny(2)} \\
SNGP & $0.595$\,{\tiny(7)} & $0.849$\,{\tiny(5)} & $0.877$\,{\tiny(13)} & $0.826$\,{\tiny(10)} & $1.100$\,{\tiny(15)} & $1.050$\,{\tiny(15)} & $0.800$\,{\tiny(13)} & $0.1442$\,{\tiny(15)} & $0.495$\,{\tiny(15)} \\
\bottomrule
\end{tabular}}\end{table}

\begin{table}[htbp]
\centering
\scriptsize
\setlength{\tabcolsep}{4pt}
\caption{\textbf{Rank agreement between the tasks on QualityMRI (DCIC).} Spearman correlation between the method rankings of two tasks over the fifteen methods of \cref{tab:results-quality-mri} ($+1$: identical rankings). Diagonal (italic): the agreement of two seeds of the same task, mean over the three pairs of seeds. Last two rows: the raw log-loss regret area and the test accuracy, as references.}
\begin{tabular}{lccccccccc}
\toprule
 & \rotatebox{90}{OOD detection, far} & \rotatebox{90}{Selective prediction, log} & \rotatebox{90}{Selective prediction, Brier} & \rotatebox{90}{Selective prediction, 0/1} & \rotatebox{90}{Regret, log} & \rotatebox{90}{Regret, Brier} & \rotatebox{90}{Regret, 0/1} & \rotatebox{90}{Regret, log, raw area} & \rotatebox{90}{Test accuracy} \\
\midrule
OOD detection, far & \textit{0.77} &  &  &  &  &  &  &  &  \\
Selective prediction, log & $0.26$ & \textit{0.49} &  &  &  &  &  &  &  \\
Selective prediction, Brier & $0.02$ & $0.52$ & \textit{0.65} &  &  &  &  &  &  \\
Selective prediction, 0/1 & $-0.04$ & $0.42$ & $0.51$ & \textit{0.70} &  &  &  &  &  \\
Regret, log & $0.00$ & $-0.27$ & $0.37$ & $0.17$ & \textit{0.02} &  &  &  &  \\
Regret, Brier & $0.01$ & $0.29$ & $0.75$ & $0.70$ & $0.59$ & \textit{0.07} &  &  &  \\
Regret, 0/1 & $0.29$ & $0.41$ & $0.61$ & $0.87$ & $0.30$ & $0.70$ & \textit{0.44} &  &  \\
Regret, log, raw area & $0.02$ & $-0.47$ & $0.21$ & $-0.11$ & $0.56$ & $0.43$ & $0.00$ & \textit{0.64} &  \\
Test accuracy & $-0.20$ & $-0.39$ & $0.20$ & $0.20$ & $0.20$ & $0.39$ & $0.15$ & $0.68$ & \textit{0.62} \\
\bottomrule
\end{tabular}\end{table}

\begin{table}[htbp]
\centering
\scriptsize
\setlength{\tabcolsep}{3pt}
\caption{\textbf{Values of every task on Treeversity\#1 (DCIC).} Fifteen methods, networks trained on the annotator distribution of each image, means over three seeds, the epistemic estimate as the score; rank in parentheses, best value per column in bold; rows by $n$-$\AuReC$ under the log loss. OOD detection: $\AuROC$ against the test images of the other DCIC datasets. The other Treeversity dataset, which shares its images, is not an OOD set. AL: $\NAUC$, acquiring by the epistemic estimate (\cref{app:protocols}). Selective prediction and regret: $n$-$\AuRC$ and $n$-$\AuReC$ under the loss named, against the annotator distribution. References: the raw area $\AuReC_{\log}$ and the test accuracy against the majority label.}
\label{tab:results-treeversity-1}
\resizebox{\textwidth}{!}{\begin{tabular}{lcccccccccc}
\toprule
 & OOD & AL & \multicolumn{3}{c}{$n$-$\AuRC$ $\downarrow$} & \multicolumn{3}{c}{$n$-$\AuReC$ $\downarrow$} & \multicolumn{2}{c}{references} \\
\cmidrule(lr){4-6} \cmidrule(lr){7-9} \cmidrule(lr){10-11}
Method & far & $\NAUC$ $\uparrow$ & log & Brier & 0/1 & log & Brier & 0/1 & $\AuReC_{\log}$ $\downarrow$ & Acc.\ $\uparrow$ \\
\midrule
BatchEnsemble & $0.845$\,{\tiny(7)} & $0.759$\,{\tiny(14)} & $0.320$\,{\tiny(3)} & $0.278$\,{\tiny(11)} & $0.486$\,{\tiny(2)} & $\mathbf{0.302}$\,{\tiny(1)} & $0.255$\,{\tiny(7)} & $0.385$\,{\tiny(2)} & $0.1830$\,{\tiny(11)} & $0.804$\,{\tiny(13)} \\
Deep Ensemble & $0.921$\,{\tiny(2)} & $\mathbf{0.849}$\,{\tiny(1)} & $\mathbf{0.294}$\,{\tiny(1)} & $\mathbf{0.223}$\,{\tiny(1)} & $0.617$\,{\tiny(6)} & $0.317$\,{\tiny(2)} & $0.215$\,{\tiny(2)} & $0.502$\,{\tiny(5)} & $\mathbf{0.1202}$\,{\tiny(1)} & $\mathbf{0.863}$\,{\tiny(1)} \\
MFVI-BNN & $0.735$\,{\tiny(13)} & $0.797$\,{\tiny(11)} & $0.338$\,{\tiny(10)} & $0.292$\,{\tiny(13)} & $0.591$\,{\tiny(5)} & $0.331$\,{\tiny(3)} & $0.286$\,{\tiny(12)} & $0.517$\,{\tiny(6)} & $0.1983$\,{\tiny(12)} & $0.803$\,{\tiny(14)} \\
DropConnect ($p{=}0.5$) & $0.757$\,{\tiny(11)} & $0.819$\,{\tiny(8)} & $0.321$\,{\tiny(4)} & $0.248$\,{\tiny(8)} & $0.734$\,{\tiny(9)} & $0.343$\,{\tiny(4)} & $0.258$\,{\tiny(8)} & $0.656$\,{\tiny(9)} & $0.1577$\,{\tiny(3)} & $0.841$\,{\tiny(6)} \\
MC-Dropout ($p{=}0.5$) & $0.852$\,{\tiny(6)} & $0.817$\,{\tiny(9)} & $0.343$\,{\tiny(12)} & $0.253$\,{\tiny(10)} & $0.756$\,{\tiny(10)} & $0.352$\,{\tiny(5)} & $0.242$\,{\tiny(4)} & $0.675$\,{\tiny(10)} & $0.1651$\,{\tiny(5)} & $0.837$\,{\tiny(7)} \\
SubEnsemble & $0.876$\,{\tiny(4)} & $0.824$\,{\tiny(4)} & $0.330$\,{\tiny(8)} & $0.246$\,{\tiny(6)} & $0.797$\,{\tiny(12)} & $0.358$\,{\tiny(6)} & $0.251$\,{\tiny(5)} & $0.742$\,{\tiny(12)} & $0.1530$\,{\tiny(2)} & $0.841$\,{\tiny(5)} \\
PostNet & $0.761$\,{\tiny(10)} & $0.775$\,{\tiny(13)} & $0.306$\,{\tiny(2)} & $0.229$\,{\tiny(3)} & $\mathbf{0.224}$\,{\tiny(1)} & $0.365$\,{\tiny(7)} & $0.255$\,{\tiny(6)} & $\mathbf{0.218}$\,{\tiny(1)} & $0.1734$\,{\tiny(9)} & $0.825$\,{\tiny(11)} \\
SNGP & $0.748$\,{\tiny(12)} & $0.715$\,{\tiny(15)} & $0.339$\,{\tiny(11)} & $0.283$\,{\tiny(12)} & $0.707$\,{\tiny(7)} & $0.368$\,{\tiny(8)} & $0.309$\,{\tiny(13)} & $0.637$\,{\tiny(7)} & $0.2168$\,{\tiny(13)} & $0.791$\,{\tiny(15)} \\
Laplace (last layer) & $0.856$\,{\tiny(5)} & $0.820$\,{\tiny(7)} & $0.324$\,{\tiny(5)} & $0.246$\,{\tiny(7)} & $0.573$\,{\tiny(4)} & $0.373$\,{\tiny(9)} & $0.263$\,{\tiny(10)} & $0.501$\,{\tiny(4)} & $0.1662$\,{\tiny(6)} & $0.835$\,{\tiny(10)} \\
MC-Dropout ($p{=}0.1$) & $0.790$\,{\tiny(9)} & $0.822$\,{\tiny(6)} & $0.326$\,{\tiny(6)} & $0.244$\,{\tiny(5)} & $0.828$\,{\tiny(13)} & $0.374$\,{\tiny(10)} & $0.260$\,{\tiny(9)} & $0.774$\,{\tiny(13)} & $0.1664$\,{\tiny(7)} & $0.835$\,{\tiny(8)} \\
DropConnect ($p{=}0.1$) & $0.688$\,{\tiny(14)} & $0.822$\,{\tiny(5)} & $0.334$\,{\tiny(9)} & $0.251$\,{\tiny(9)} & $0.788$\,{\tiny(11)} & $0.382$\,{\tiny(11)} & $0.276$\,{\tiny(11)} & $0.736$\,{\tiny(11)} & $0.1676$\,{\tiny(8)} & $0.835$\,{\tiny(9)} \\
SWAG & $0.891$\,{\tiny(3)} & $0.835$\,{\tiny(2)} & $0.330$\,{\tiny(7)} & $0.227$\,{\tiny(2)} & $0.729$\,{\tiny(8)} & $0.384$\,{\tiny(12)} & $0.234$\,{\tiny(3)} & $0.652$\,{\tiny(8)} & $0.1626$\,{\tiny(4)} & $0.850$\,{\tiny(2)} \\
Evidential & $\mathbf{0.932}$\,{\tiny(1)} & $0.833$\,{\tiny(3)} & $0.511$\,{\tiny(13)} & $0.231$\,{\tiny(4)} & $0.564$\,{\tiny(3)} & $0.391$\,{\tiny(13)} & $\mathbf{0.187}$\,{\tiny(1)} & $0.437$\,{\tiny(3)} & $0.1779$\,{\tiny(10)} & $0.848$\,{\tiny(3)} \\
DDU & $0.808$\,{\tiny(8)} & $0.812$\,{\tiny(10)} & $1.012$\,{\tiny(14)} & $1.033$\,{\tiny(15)} & $1.022$\,{\tiny(14)} & $0.868$\,{\tiny(14)} & $0.848$\,{\tiny(14)} & $0.828$\,{\tiny(14)} & $0.2782$\,{\tiny(14)} & $0.844$\,{\tiny(4)} \\
NatPN & $0.309$\,{\tiny(15)} & $0.791$\,{\tiny(12)} & $1.378$\,{\tiny(15)} & $0.984$\,{\tiny(14)} & $1.372$\,{\tiny(15)} & $1.420$\,{\tiny(15)} & $1.032$\,{\tiny(15)} & $1.422$\,{\tiny(15)} & $0.4920$\,{\tiny(15)} & $0.817$\,{\tiny(12)} \\
\bottomrule
\end{tabular}}\end{table}

\begin{table}[htbp]
\centering
\scriptsize
\setlength{\tabcolsep}{4pt}
\caption{\textbf{Rank agreement between the tasks on Treeversity\#1 (DCIC).} Spearman correlation between the method rankings of two tasks over the fifteen methods of \cref{tab:results-treeversity-1} ($+1$: identical rankings). Diagonal (italic): the agreement of two seeds of the same task, mean over the three pairs of seeds. Last two rows: the raw log-loss regret area and the test accuracy, as references.}
\begin{tabular}{lcccccccccc}
\toprule
 & \rotatebox{90}{OOD detection, far} & \rotatebox{90}{Active learning} & \rotatebox{90}{Selective prediction, log} & \rotatebox{90}{Selective prediction, Brier} & \rotatebox{90}{Selective prediction, 0/1} & \rotatebox{90}{Regret, log} & \rotatebox{90}{Regret, Brier} & \rotatebox{90}{Regret, 0/1} & \rotatebox{90}{Regret, log, raw area} & \rotatebox{90}{Test accuracy} \\
\midrule
OOD detection, far & \textit{0.80} &  &  &  &  &  &  &  &  &  \\
Active learning & $0.64$ & \textit{0.93} &  &  &  &  &  &  &  &  \\
Selective prediction, log & $0.25$ & $0.16$ & \textit{0.60} &  &  &  &  &  &  &  \\
Selective prediction, Brier & $0.64$ & $0.69$ & $0.62$ & \textit{0.87} &  &  &  &  &  &  \\
Selective prediction, 0/1 & $0.31$ & $-0.13$ & $0.52$ & $0.37$ & \textit{0.96} &  &  &  &  &  \\
Regret, log & $0.10$ & $-0.14$ & $0.64$ & $0.12$ & $0.47$ & \textit{0.68} &  &  &  &  \\
Regret, Brier & $0.82$ & $0.62$ & $0.41$ & $0.81$ & $0.43$ & $0.31$ & \textit{0.77} &  &  &  \\
Regret, 0/1 & $0.35$ & $-0.10$ & $0.55$ & $0.41$ & $1.00$ & $0.48$ & $0.47$ & \textit{0.94} &  &  \\
Regret, log, raw area & $0.56$ & $0.71$ & $0.59$ & $0.73$ & $0.05$ & $0.41$ & $0.69$ & $0.09$ & \textit{0.86} &  \\
Test accuracy & $0.72$ & $0.84$ & $0.08$ & $0.58$ & $-0.19$ & $-0.17$ & $0.64$ & $-0.15$ & $0.61$ & \textit{0.89} \\
\bottomrule
\end{tabular}\end{table}

\begin{table}[htbp]
\centering
\scriptsize
\setlength{\tabcolsep}{3pt}
\caption{\textbf{Values of every task on Treeversity\#6 (DCIC).} Fifteen methods, networks trained on the annotator distribution of each image, means over three seeds, the epistemic estimate as the score; rank in parentheses, best value per column in bold; rows by $n$-$\AuReC$ under the log loss. OOD detection: $\AuROC$ against the test images of the other DCIC datasets. The other Treeversity dataset, which shares its images, is not an OOD set. AL: $\NAUC$, acquiring by the epistemic estimate (\cref{app:protocols}). Selective prediction and regret: $n$-$\AuRC$ and $n$-$\AuReC$ under the loss named, against the annotator distribution. References: the raw area $\AuReC_{\log}$ and the test accuracy against the majority label.}
\label{tab:results-treeversity-6}
\resizebox{\textwidth}{!}{\begin{tabular}{lcccccccccc}
\toprule
 & OOD & AL & \multicolumn{3}{c}{$n$-$\AuRC$ $\downarrow$} & \multicolumn{3}{c}{$n$-$\AuReC$ $\downarrow$} & \multicolumn{2}{c}{references} \\
\cmidrule(lr){4-6} \cmidrule(lr){7-9} \cmidrule(lr){10-11}
Method & far & $\NAUC$ $\uparrow$ & log & Brier & 0/1 & log & Brier & 0/1 & $\AuReC_{\log}$ $\downarrow$ & Acc.\ $\uparrow$ \\
\midrule
BatchEnsemble & $0.906$\,{\tiny(7)} & $0.734$\,{\tiny(14)} & $0.547$\,{\tiny(7)} & $0.455$\,{\tiny(10)} & $0.652$\,{\tiny(4)} & $\mathbf{0.386}$\,{\tiny(1)} & $\mathbf{0.319}$\,{\tiny(1)} & $\mathbf{0.352}$\,{\tiny(1)} & $0.1520$\,{\tiny(5)} & $0.782$\,{\tiny(13)} \\
MFVI-BNN & $0.812$\,{\tiny(10)} & $0.753$\,{\tiny(12)} & $0.534$\,{\tiny(6)} & $0.444$\,{\tiny(9)} & $0.730$\,{\tiny(6)} & $0.468$\,{\tiny(2)} & $0.413$\,{\tiny(3)} & $0.531$\,{\tiny(5)} & $0.2073$\,{\tiny(13)} & $0.773$\,{\tiny(14)} \\
Deep Ensemble & $0.963$\,{\tiny(3)} & $\mathbf{0.817}$\,{\tiny(1)} & $0.483$\,{\tiny(4)} & $0.371$\,{\tiny(5)} & $0.720$\,{\tiny(5)} & $0.479$\,{\tiny(3)} & $0.432$\,{\tiny(5)} & $0.462$\,{\tiny(2)} & $\mathbf{0.1319}$\,{\tiny(1)} & $0.812$\,{\tiny(5)} \\
DropConnect ($p{=}0.5$) & $0.920$\,{\tiny(5)} & $0.792$\,{\tiny(6)} & $0.610$\,{\tiny(10)} & $0.457$\,{\tiny(11)} & $0.830$\,{\tiny(10)} & $0.518$\,{\tiny(4)} & $0.442$\,{\tiny(6)} & $0.680$\,{\tiny(11)} & $0.1483$\,{\tiny(3)} & $\mathbf{0.832}$\,{\tiny(1)} \\
MC-Dropout ($p{=}0.5$) & $0.901$\,{\tiny(8)} & $0.798$\,{\tiny(4)} & $0.634$\,{\tiny(11)} & $0.470$\,{\tiny(12)} & $0.847$\,{\tiny(11)} & $0.529$\,{\tiny(5)} & $0.424$\,{\tiny(4)} & $0.705$\,{\tiny(14)} & $0.1512$\,{\tiny(4)} & $0.821$\,{\tiny(4)} \\
SubEnsemble & $0.918$\,{\tiny(6)} & $0.799$\,{\tiny(3)} & $0.477$\,{\tiny(3)} & $0.358$\,{\tiny(3)} & $0.851$\,{\tiny(12)} & $0.530$\,{\tiny(6)} & $0.481$\,{\tiny(7)} & $0.701$\,{\tiny(13)} & $0.1424$\,{\tiny(2)} & $0.827$\,{\tiny(3)} \\
Evidential & $\mathbf{0.978}$\,{\tiny(1)} & $0.787$\,{\tiny(8)} & $1.047$\,{\tiny(13)} & $0.342$\,{\tiny(2)} & $1.130$\,{\tiny(13)} & $0.554$\,{\tiny(7)} & $0.391$\,{\tiny(2)} & $0.573$\,{\tiny(7)} & $0.1618$\,{\tiny(6)} & $0.812$\,{\tiny(6)} \\
SWAG & $0.944$\,{\tiny(4)} & $0.801$\,{\tiny(2)} & $0.570$\,{\tiny(9)} & $0.409$\,{\tiny(7)} & $0.820$\,{\tiny(9)} & $0.642$\,{\tiny(8)} & $0.569$\,{\tiny(8)} & $0.680$\,{\tiny(12)} & $0.1639$\,{\tiny(7)} & $0.830$\,{\tiny(2)} \\
SNGP & $0.590$\,{\tiny(13)} & $0.679$\,{\tiny(15)} & $0.409$\,{\tiny(2)} & $0.359$\,{\tiny(4)} & $0.639$\,{\tiny(3)} & $0.668$\,{\tiny(9)} & $0.658$\,{\tiny(11)} & $0.550$\,{\tiny(6)} & $0.2376$\,{\tiny(14)} & $0.762$\,{\tiny(15)} \\
DDU & $0.965$\,{\tiny(2)} & $0.745$\,{\tiny(13)} & $1.195$\,{\tiny(14)} & $1.237$\,{\tiny(15)} & $1.235$\,{\tiny(15)} & $0.692$\,{\tiny(10)} & $0.664$\,{\tiny(13)} & $0.671$\,{\tiny(10)} & $0.1851$\,{\tiny(9)} & $0.791$\,{\tiny(12)} \\
PostNet & $0.514$\,{\tiny(14)} & $0.756$\,{\tiny(11)} & $\mathbf{0.332}$\,{\tiny(1)} & $\mathbf{0.250}$\,{\tiny(1)} & $\mathbf{0.299}$\,{\tiny(1)} & $0.717$\,{\tiny(11)} & $0.617$\,{\tiny(9)} & $0.467$\,{\tiny(3)} & $0.1925$\,{\tiny(11)} & $0.806$\,{\tiny(10)} \\
MC-Dropout ($p{=}0.1$) & $0.845$\,{\tiny(9)} & $0.787$\,{\tiny(7)} & $0.561$\,{\tiny(8)} & $0.422$\,{\tiny(8)} & $0.803$\,{\tiny(8)} & $0.720$\,{\tiny(12)} & $0.651$\,{\tiny(10)} & $0.667$\,{\tiny(9)} & $0.1849$\,{\tiny(8)} & $0.807$\,{\tiny(9)} \\
Laplace (last layer) & $0.802$\,{\tiny(11)} & $0.795$\,{\tiny(5)} & $0.490$\,{\tiny(5)} & $0.399$\,{\tiny(6)} & $0.601$\,{\tiny(2)} & $0.724$\,{\tiny(13)} & $0.658$\,{\tiny(12)} & $0.507$\,{\tiny(4)} & $0.1858$\,{\tiny(10)} & $0.805$\,{\tiny(11)} \\
DropConnect ($p{=}0.1$) & $0.760$\,{\tiny(12)} & $0.785$\,{\tiny(9)} & $0.665$\,{\tiny(12)} & $0.490$\,{\tiny(13)} & $0.756$\,{\tiny(7)} & $0.819$\,{\tiny(14)} & $0.748$\,{\tiny(14)} & $0.647$\,{\tiny(8)} & $0.1971$\,{\tiny(12)} & $0.809$\,{\tiny(7)} \\
NatPN & $0.387$\,{\tiny(15)} & $0.762$\,{\tiny(10)} & $1.304$\,{\tiny(15)} & $1.054$\,{\tiny(14)} & $1.225$\,{\tiny(14)} & $1.328$\,{\tiny(15)} & $1.193$\,{\tiny(15)} & $1.194$\,{\tiny(15)} & $0.2885$\,{\tiny(15)} & $0.807$\,{\tiny(8)} \\
\bottomrule
\end{tabular}}\end{table}

\begin{table}[htbp]
\centering
\scriptsize
\setlength{\tabcolsep}{4pt}
\caption{\textbf{Rank agreement between the tasks on Treeversity\#6 (DCIC).} Spearman correlation between the method rankings of two tasks over the fifteen methods of \cref{tab:results-treeversity-6} ($+1$: identical rankings). Diagonal (italic): the agreement of two seeds of the same task, mean over the three pairs of seeds. Last two rows: the raw log-loss regret area and the test accuracy, as references.}
\begin{tabular}{lcccccccccc}
\toprule
 & \rotatebox{90}{OOD detection, far} & \rotatebox{90}{Active learning} & \rotatebox{90}{Selective prediction, log} & \rotatebox{90}{Selective prediction, Brier} & \rotatebox{90}{Selective prediction, 0/1} & \rotatebox{90}{Regret, log} & \rotatebox{90}{Regret, Brier} & \rotatebox{90}{Regret, 0/1} & \rotatebox{90}{Regret, log, raw area} & \rotatebox{90}{Test accuracy} \\
\midrule
OOD detection, far & \textit{0.90} &  &  &  &  &  &  &  &  &  \\
Active learning & $0.39$ & \textit{0.93} &  &  &  &  &  &  &  &  \\
Selective prediction, log & $-0.25$ & $0.04$ & \textit{0.95} &  &  &  &  &  &  &  \\
Selective prediction, Brier & $0.05$ & $0.20$ & $0.72$ & \textit{0.96} &  &  &  &  &  &  \\
Selective prediction, 0/1 & $-0.45$ & $-0.16$ & $0.77$ & $0.45$ & \textit{0.95} &  &  &  &  &  \\
Regret, log & $0.52$ & $0.09$ & $0.25$ & $0.14$ & $0.07$ & \textit{0.82} &  &  &  &  \\
Regret, Brier & $0.50$ & $0.14$ & $0.19$ & $0.32$ & $0.08$ & $0.90$ & \textit{0.85} &  &  &  \\
Regret, 0/1 & $-0.04$ & $-0.30$ & $0.55$ & $0.42$ & $0.79$ & $0.28$ & $0.31$ & \textit{0.79} &  &  \\
Regret, log, raw area & $0.74$ & $0.65$ & $0.09$ & $0.19$ & $-0.19$ & $0.61$ & $0.60$ & $-0.02$ & \textit{0.84} &  \\
Test accuracy & $0.43$ & $0.82$ & $-0.24$ & $0.02$ & $-0.44$ & $0.11$ & $0.15$ & $-0.55$ & $0.65$ & \textit{0.78} \\
\bottomrule
\end{tabular}\end{table}

\begin{table}[htbp]
\centering
\scriptsize
\setlength{\tabcolsep}{3pt}
\caption{\textbf{Values of every task on Turkey (DCIC).} Fifteen methods, networks trained on the annotator distribution of each image, means over three seeds, the epistemic estimate as the score; rank in parentheses, best value per column in bold; rows by $n$-$\AuReC$ under the log loss. OOD detection: $\AuROC$ against the test images of the other DCIC datasets. AL: $\NAUC$, acquiring by the epistemic estimate (\cref{app:protocols}). Selective prediction and regret: $n$-$\AuRC$ and $n$-$\AuReC$ under the loss named, against the annotator distribution. References: the raw area $\AuReC_{\log}$ and the test accuracy against the majority label.}
\label{tab:results-turkey}
\resizebox{\textwidth}{!}{\begin{tabular}{lcccccccccc}
\toprule
 & OOD & AL & \multicolumn{3}{c}{$n$-$\AuRC$ $\downarrow$} & \multicolumn{3}{c}{$n$-$\AuReC$ $\downarrow$} & \multicolumn{2}{c}{references} \\
\cmidrule(lr){4-6} \cmidrule(lr){7-9} \cmidrule(lr){10-11}
Method & far & $\NAUC$ $\uparrow$ & log & Brier & 0/1 & log & Brier & 0/1 & $\AuReC_{\log}$ $\downarrow$ & Acc.\ $\uparrow$ \\
\midrule
BatchEnsemble & $0.863$\,{\tiny(5)} & $0.855$\,{\tiny(10)} & $0.306$\,{\tiny(4)} & $0.248$\,{\tiny(9)} & $0.493$\,{\tiny(3)} & $\mathbf{0.275}$\,{\tiny(1)} & $0.194$\,{\tiny(7)} & $0.321$\,{\tiny(3)} & $0.0992$\,{\tiny(9)} & $0.869$\,{\tiny(13)} \\
Laplace (last layer) & $\mathbf{0.926}$\,{\tiny(1)} & $0.867$\,{\tiny(6)} & $\mathbf{0.253}$\,{\tiny(1)} & $\mathbf{0.187}$\,{\tiny(1)} & $0.633$\,{\tiny(5)} & $0.296$\,{\tiny(2)} & $0.163$\,{\tiny(2)} & $0.515$\,{\tiny(5)} & $0.0906$\,{\tiny(2)} & $0.872$\,{\tiny(9)} \\
Deep Ensemble & $0.810$\,{\tiny(11)} & $\mathbf{0.883}$\,{\tiny(1)} & $0.308$\,{\tiny(5)} & $0.232$\,{\tiny(7)} & $0.616$\,{\tiny(4)} & $0.300$\,{\tiny(3)} & $0.185$\,{\tiny(6)} & $0.437$\,{\tiny(4)} & $\mathbf{0.0893}$\,{\tiny(1)} & $0.885$\,{\tiny(5)} \\
MC-Dropout ($p{=}0.1$) & $0.849$\,{\tiny(7)} & $0.873$\,{\tiny(2)} & $0.276$\,{\tiny(2)} & $0.194$\,{\tiny(2)} & $0.854$\,{\tiny(11)} & $0.309$\,{\tiny(4)} & $\mathbf{0.162}$\,{\tiny(1)} & $0.794$\,{\tiny(12)} & $0.0921$\,{\tiny(3)} & $0.872$\,{\tiny(10)} \\
DropConnect ($p{=}0.1$) & $0.829$\,{\tiny(8)} & $0.868$\,{\tiny(5)} & $0.298$\,{\tiny(3)} & $0.204$\,{\tiny(3)} & $0.831$\,{\tiny(10)} & $0.317$\,{\tiny(5)} & $0.170$\,{\tiny(3)} & $0.745$\,{\tiny(10)} & $0.0944$\,{\tiny(4)} & $0.870$\,{\tiny(12)} \\
Evidential & $0.823$\,{\tiny(9)} & $0.861$\,{\tiny(8)} & $0.349$\,{\tiny(11)} & $0.225$\,{\tiny(6)} & $0.359$\,{\tiny(2)} & $0.318$\,{\tiny(6)} & $0.179$\,{\tiny(4)} & $0.281$\,{\tiny(2)} & $0.1123$\,{\tiny(10)} & $0.873$\,{\tiny(8)} \\
SNGP & $0.865$\,{\tiny(4)} & $0.825$\,{\tiny(15)} & $0.320$\,{\tiny(7)} & $0.253$\,{\tiny(11)} & $0.707$\,{\tiny(7)} & $0.337$\,{\tiny(7)} & $0.228$\,{\tiny(10)} & $0.617$\,{\tiny(7)} & $0.1242$\,{\tiny(11)} & $0.847$\,{\tiny(15)} \\
DropConnect ($p{=}0.5$) & $0.795$\,{\tiny(12)} & $0.873$\,{\tiny(3)} & $0.345$\,{\tiny(9)} & $0.240$\,{\tiny(8)} & $0.815$\,{\tiny(9)} & $0.371$\,{\tiny(8)} & $0.206$\,{\tiny(8)} & $0.676$\,{\tiny(9)} & $0.0956$\,{\tiny(5)} & $0.896$\,{\tiny(2)} \\
SWAG & $0.855$\,{\tiny(6)} & $0.867$\,{\tiny(7)} & $0.328$\,{\tiny(8)} & $0.214$\,{\tiny(4)} & $0.775$\,{\tiny(8)} & $0.372$\,{\tiny(9)} & $0.180$\,{\tiny(5)} & $0.671$\,{\tiny(8)} & $0.0958$\,{\tiny(6)} & $0.892$\,{\tiny(3)} \\
MFVI-BNN & $0.674$\,{\tiny(13)} & $0.830$\,{\tiny(14)} & $0.345$\,{\tiny(10)} & $0.288$\,{\tiny(13)} & $0.637$\,{\tiny(6)} & $0.373$\,{\tiny(10)} & $0.277$\,{\tiny(13)} & $0.549$\,{\tiny(6)} & $0.1417$\,{\tiny(13)} & $0.853$\,{\tiny(14)} \\
SubEnsemble & $0.878$\,{\tiny(3)} & $0.860$\,{\tiny(9)} & $0.315$\,{\tiny(6)} & $0.223$\,{\tiny(5)} & $0.881$\,{\tiny(13)} & $0.378$\,{\tiny(11)} & $0.218$\,{\tiny(9)} & $0.816$\,{\tiny(13)} & $0.0968$\,{\tiny(7)} & $0.888$\,{\tiny(4)} \\
MC-Dropout ($p{=}0.5$) & $0.817$\,{\tiny(10)} & $0.870$\,{\tiny(4)} & $0.387$\,{\tiny(13)} & $0.259$\,{\tiny(12)} & $0.868$\,{\tiny(12)} & $0.417$\,{\tiny(12)} & $0.232$\,{\tiny(11)} & $0.788$\,{\tiny(11)} & $0.0989$\,{\tiny(8)} & $\mathbf{0.903}$\,{\tiny(1)} \\
PostNet & $0.625$\,{\tiny(14)} & $0.848$\,{\tiny(12)} & $0.373$\,{\tiny(12)} & $0.251$\,{\tiny(10)} & $\mathbf{0.294}$\,{\tiny(1)} & $0.429$\,{\tiny(13)} & $0.245$\,{\tiny(12)} & $\mathbf{0.267}$\,{\tiny(1)} & $0.1247$\,{\tiny(12)} & $0.870$\,{\tiny(11)} \\
NatPN & $0.597$\,{\tiny(15)} & $0.837$\,{\tiny(13)} & $1.224$\,{\tiny(14)} & $0.888$\,{\tiny(14)} & $1.133$\,{\tiny(14)} & $1.199$\,{\tiny(14)} & $0.835$\,{\tiny(14)} & $1.041$\,{\tiny(14)} & $0.2450$\,{\tiny(14)} & $0.877$\,{\tiny(7)} \\
DDU & $0.892$\,{\tiny(2)} & $0.848$\,{\tiny(11)} & $1.377$\,{\tiny(15)} & $1.435$\,{\tiny(15)} & $1.476$\,{\tiny(15)} & $1.242$\,{\tiny(15)} & $1.323$\,{\tiny(15)} & $1.393$\,{\tiny(15)} & $0.2708$\,{\tiny(15)} & $0.883$\,{\tiny(6)} \\
\bottomrule
\end{tabular}}\end{table}

\begin{table}[htbp]
\centering
\scriptsize
\setlength{\tabcolsep}{4pt}
\caption{\textbf{Rank agreement between the tasks on Turkey (DCIC).} Spearman correlation between the method rankings of two tasks over the fifteen methods of \cref{tab:results-turkey} ($+1$: identical rankings). Diagonal (italic): the agreement of two seeds of the same task, mean over the three pairs of seeds. Last two rows: the raw log-loss regret area and the test accuracy, as references.}
\begin{tabular}{lcccccccccc}
\toprule
 & \rotatebox{90}{OOD detection, far} & \rotatebox{90}{Active learning} & \rotatebox{90}{Selective prediction, log} & \rotatebox{90}{Selective prediction, Brier} & \rotatebox{90}{Selective prediction, 0/1} & \rotatebox{90}{Regret, log} & \rotatebox{90}{Regret, Brier} & \rotatebox{90}{Regret, 0/1} & \rotatebox{90}{Regret, log, raw area} & \rotatebox{90}{Test accuracy} \\
\midrule
OOD detection, far & \textit{0.57} &  &  &  &  &  &  &  &  &  \\
Active learning & $0.02$ & \textit{0.87} &  &  &  &  &  &  &  &  \\
Selective prediction, log & $0.44$ & $0.44$ & \textit{0.43} &  &  &  &  &  &  &  \\
Selective prediction, Brier & $0.37$ & $0.60$ & $0.82$ & \textit{0.60} &  &  &  &  &  &  \\
Selective prediction, 0/1 & $-0.16$ & $-0.04$ & $0.26$ & $0.22$ & \textit{0.93} &  &  &  &  &  \\
Regret, log & $0.30$ & $0.46$ & $0.86$ & $0.67$ & $0.52$ & \textit{0.50} &  &  &  &  \\
Regret, Brier & $0.32$ & $0.66$ & $0.82$ & $0.94$ & $0.31$ & $0.82$ & \textit{0.70} &  &  &  \\
Regret, 0/1 & $-0.17$ & $-0.05$ & $0.22$ & $0.19$ & $1.00$ & $0.49$ & $0.27$ & \textit{0.84} &  &  \\
Regret, log, raw area & $0.21$ & $0.85$ & $0.79$ & $0.84$ & $0.15$ & $0.71$ & $0.83$ & $0.13$ & \textit{0.69} &  \\
Test accuracy & $-0.06$ & $0.54$ & $-0.31$ & $0.04$ & $-0.41$ & $-0.32$ & $-0.05$ & $-0.38$ & $0.25$ & \textit{0.41} \\
\bottomrule
\end{tabular}\end{table}

\begin{table}[htbp]
\centering
\scriptsize
\setlength{\tabcolsep}{3pt}
\caption{\textbf{Values of every task on APPA-REAL.} Fourteen methods, heads with one hidden layer on the frozen features of a ResNet-101 trained on CSFD-1.6M \citep{csfd}, trained on the annotator distribution of each image (\cref{app:implementation}), means over five seeds, the epistemic estimate of the absolute-loss decomposition as the score; rank in parentheses, best value per column in bold; rows by $n$-$\AuReC$ under the absolute loss. OOD detection: $\AuROC$ against faces detected in LAION-Face images that do not show a photographed human face, so that no apparent age can be estimated (near, \cref{fig:appa_ood}), and against ImageNet and Places365 images (far, the mean of the two). AL: the test risk under the absolute loss in years, mean over the ten rounds, acquiring by the epistemic estimate (\cref{app:protocols}). Selective prediction and regret: $n$-$\AuRC$ and $n$-$\AuReC$ under the loss named, against the annotator distribution. References: the raw area $\AuReC$ under the absolute loss and the test risk under the absolute loss in years.}
\label{tab:regret_appareal}
\resizebox{\textwidth}{!}{\begin{tabular}{lccccccccc}
\toprule
 & \multicolumn{2}{c}{OOD $\AuROC$ $\uparrow$} & AL & \multicolumn{2}{c}{$n$-$\AuRC$ $\downarrow$} & \multicolumn{2}{c}{$n$-$\AuReC$ $\downarrow$} & \multicolumn{2}{c}{references} \\
\cmidrule(lr){2-3} \cmidrule(lr){5-6} \cmidrule(lr){7-8} \cmidrule(lr){9-10}
Method & near & far & risk $\downarrow$ & absolute & squared & absolute & squared & $\AuReC_{\text{abs}}$ $\downarrow$ & risk $\downarrow$ \\
\midrule
PostNet & $0.552$\,{\tiny(12)} & $0.670$\,{\tiny(10)} & $8.24$\,{\tiny(14)} & $0.469$\,{\tiny(5)} & $\mathbf{0.118}$\,{\tiny(1)} & $\mathbf{0.375}$\,{\tiny(1)} & $\mathbf{0.103}$\,{\tiny(1)} & $2.0756$\,{\tiny(12)} & $8.27$\,{\tiny(14)} \\
MC-Dropout ($p{=}0.5$) & $\mathbf{0.675}$\,{\tiny(1)} & $0.796$\,{\tiny(2)} & $6.14$\,{\tiny(4)} & $\mathbf{0.328}$\,{\tiny(1)} & $0.301$\,{\tiny(3)} & $0.486$\,{\tiny(2)} & $0.322$\,{\tiny(2)} & $\mathbf{1.2311}$\,{\tiny(1)} & $5.92$\,{\tiny(9)} \\
MFVI-BNN & $0.671$\,{\tiny(2)} & $0.678$\,{\tiny(9)} & $6.48$\,{\tiny(12)} & $0.395$\,{\tiny(2)} & $0.365$\,{\tiny(5)} & $0.526$\,{\tiny(3)} & $0.358$\,{\tiny(5)} & $1.4613$\,{\tiny(5)} & $6.18$\,{\tiny(12)} \\
SWAG & $0.647$\,{\tiny(4)} & $0.707$\,{\tiny(8)} & $6.17$\,{\tiny(7)} & $0.448$\,{\tiny(3)} & $0.291$\,{\tiny(2)} & $0.566$\,{\tiny(4)} & $0.329$\,{\tiny(3)} & $1.3272$\,{\tiny(2)} & $5.87$\,{\tiny(4)} \\
MC-Dropout ($p{=}0.1$) & $0.644$\,{\tiny(6)} & $0.726$\,{\tiny(5)} & $6.13$\,{\tiny(3)} & $0.456$\,{\tiny(4)} & $0.320$\,{\tiny(4)} & $0.567$\,{\tiny(5)} & $0.354$\,{\tiny(4)} & $1.3332$\,{\tiny(3)} & $5.88$\,{\tiny(5)} \\
BatchEnsemble & $0.666$\,{\tiny(3)} & $0.757$\,{\tiny(3)} & $6.17$\,{\tiny(8)} & $0.502$\,{\tiny(6)} & $0.441$\,{\tiny(10)} & $0.594$\,{\tiny(6)} & $0.418$\,{\tiny(10)} & $1.3769$\,{\tiny(4)} & $5.88$\,{\tiny(6)} \\
DropConnect ($p{=}0.5$) & $0.643$\,{\tiny(7)} & $0.751$\,{\tiny(4)} & $6.16$\,{\tiny(6)} & $0.678$\,{\tiny(7)} & $0.415$\,{\tiny(7)} & $0.622$\,{\tiny(7)} & $0.380$\,{\tiny(7)} & $1.4920$\,{\tiny(6)} & $5.98$\,{\tiny(10)} \\
SNGP & $0.641$\,{\tiny(8)} & $0.712$\,{\tiny(7)} & $6.42$\,{\tiny(11)} & $0.712$\,{\tiny(10)} & $0.419$\,{\tiny(8)} & $0.649$\,{\tiny(8)} & $0.364$\,{\tiny(6)} & $1.6817$\,{\tiny(11)} & $6.18$\,{\tiny(11)} \\
SubEnsemble & $0.605$\,{\tiny(10)} & $0.712$\,{\tiny(6)} & $\mathbf{6.05}$\,{\tiny(1)} & $0.698$\,{\tiny(9)} & $0.410$\,{\tiny(6)} & $0.711$\,{\tiny(9)} & $0.402$\,{\tiny(9)} & $1.5817$\,{\tiny(8)} & $5.90$\,{\tiny(8)} \\
Evidential & $0.644$\,{\tiny(5)} & $\mathbf{0.886}$\,{\tiny(1)} & $6.22$\,{\tiny(9)} & $0.757$\,{\tiny(12)} & $0.425$\,{\tiny(9)} & $0.719$\,{\tiny(10)} & $0.397$\,{\tiny(8)} & $1.5680$\,{\tiny(7)} & $5.87$\,{\tiny(2)} \\
Deep Ensemble & $0.620$\,{\tiny(9)} & $0.628$\,{\tiny(11)} & $6.09$\,{\tiny(2)} & $0.690$\,{\tiny(8)} & $0.466$\,{\tiny(11)} & $0.731$\,{\tiny(11)} & $0.481$\,{\tiny(11)} & $1.5908$\,{\tiny(9)} & $5.87$\,{\tiny(3)} \\
Laplace (last layer) & $0.563$\,{\tiny(11)} & $0.530$\,{\tiny(12)} & $6.14$\,{\tiny(5)} & $0.751$\,{\tiny(11)} & $0.558$\,{\tiny(12)} & $0.791$\,{\tiny(12)} & $0.551$\,{\tiny(12)} & $1.6598$\,{\tiny(10)} & $\mathbf{5.84}$\,{\tiny(1)} \\
NatPN & $0.480$\,{\tiny(14)} & $0.460$\,{\tiny(13)} & $7.82$\,{\tiny(13)} & $0.894$\,{\tiny(13)} & $1.031$\,{\tiny(13)} & $0.874$\,{\tiny(13)} & $1.033$\,{\tiny(13)} & $2.6285$\,{\tiny(14)} & $6.80$\,{\tiny(13)} \\
DDU & $0.536$\,{\tiny(13)} & $0.223$\,{\tiny(14)} & $6.41$\,{\tiny(10)} & $1.104$\,{\tiny(14)} & $1.094$\,{\tiny(14)} & $1.132$\,{\tiny(14)} & $1.150$\,{\tiny(14)} & $2.2471$\,{\tiny(13)} & $5.89$\,{\tiny(7)} \\
\bottomrule
\end{tabular}}\end{table}

\begin{table}[htbp]
\centering
\scriptsize
\setlength{\tabcolsep}{4pt}
\caption{\textbf{Rank agreement between the tasks on APPA-REAL.} Spearman correlation between the method rankings of two tasks over the fourteen methods of \cref{tab:regret_appareal} ($+1$: identical rankings). Diagonal (italic): the agreement of two seeds of the same task, mean over the ten pairs of seeds. Last two rows: the raw absolute-loss regret area and the test risk, as references.}
\label{tab:agreement_appa_real}
\begin{tabular}{lccccccccc}
\toprule
 & \rotatebox{90}{OOD detection, near} & \rotatebox{90}{OOD detection, far} & \rotatebox{90}{Active learning} & \rotatebox{90}{Selective prediction, absolute} & \rotatebox{90}{Selective prediction, squared} & \rotatebox{90}{Regret, absolute} & \rotatebox{90}{Regret, squared} & \rotatebox{90}{Regret, absolute, raw area} & \rotatebox{90}{Test risk, absolute} \\
\midrule
OOD detection, near & \textit{0.86} &  &  &  &  &  &  &  &  \\
OOD detection, far & $0.74$ & \textit{0.86} &  &  &  &  &  &  &  \\
Active learning & $0.21$ & $0.31$ & \textit{0.83} &  &  &  &  &  &  \\
Selective prediction, absolute & $0.74$ & $0.41$ & $0.18$ & \textit{0.91} &  &  &  &  &  \\
Selective prediction, squared & $0.49$ & $0.43$ & $0.05$ & $0.83$ & \textit{0.88} &  &  &  &  \\
Regret, absolute & $0.64$ & $0.47$ & $-0.08$ & $0.92$ & $0.92$ & \textit{0.92} &  &  &  \\
Regret, squared & $0.55$ & $0.47$ & $-0.05$ & $0.82$ & $0.96$ & $0.93$ & \textit{0.90} &  &  \\
Regret, absolute, raw area & $0.91$ & $0.71$ & $0.47$ & $0.79$ & $0.58$ & $0.63$ & $0.56$ & \textit{0.94} &  \\
Test risk, absolute & $0.16$ & $0.14$ & $0.60$ & $-0.17$ & $-0.27$ & $-0.35$ & $-0.30$ & $0.29$ & \textit{0.82} \\
\bottomrule
\end{tabular}\end{table}

}

\clearpage

\end{document}